\documentclass[twoside]{article}

\usepackage[accepted]{aistats2021}
\usepackage[round]{natbib}
\renewcommand{\bibname}{References}
\renewcommand{\bibsection}{\subsubsection*{\bibname}}

\usepackage[utf8]{inputenc} 
\usepackage[T1]{fontenc}    
\usepackage{url}            
\usepackage{booktabs}       
\usepackage{amsfonts}       
\usepackage{nicefrac}       
\usepackage{microtype}      

\usepackage{amsmath,amsfonts,bm}

\def\eqref#1{equation~\ref{#1}}

\def\1{\bm{1}}

\DeclareMathAlphabet{\mathsfit}{\encodingdefault}{\sfdefault}{m}{sl}
\SetMathAlphabet{\mathsfit}{bold}{\encodingdefault}{\sfdefault}{bx}{n}

\DeclareMathOperator*{\argmax}{arg\,max}

\usepackage{array}
\usepackage{url}
\usepackage{tikz}
\usepackage{booktabs}       
\usepackage{amsfonts}       
\usepackage{nicefrac}       
\usepackage{microtype}      
\usepackage{float}
\usepackage{mathtools}
\usepackage{amsthm}
\usepackage{amssymb}
\usepackage{amsfonts}
\usepackage{graphicx}
\usepackage{booktabs}
\usepackage{siunitx}
\usepackage{arydshln}
\usepackage{subcaption}
\usepackage{soul}
\usepackage{xcolor}
\usepackage{epstopdf}
\usepackage{enumitem}
\usepackage{algorithm}
\usepackage{algpseudocode}
\usepackage{multicol}

\usepackage[wby]{callouts}
\usetikzlibrary{positioning}

\usepackage{amssymb}
\usepackage{pifont}
\newcommand{\cmark}{\ding{51}}%
\newcommand{\xmark}{\ding{55}}%

\usepackage{wrapfig}
\usepackage{tabularx}
\usepackage{arydshln}
\usepackage{array}
\usepackage{multirow}

\newtheorem{theorem}{Theorem}
\newtheorem{lemma}[theorem]{Lemma}

\newtheorem{definition}[theorem]{Definition}
\newtheorem{remark}[theorem]{Remark}

\newcommand{\Lip}{\mathrm{Lip}}
\newcommand{\biLip}{\mathrm{biLip}}

\usepackage[colorlinks=true]{hyperref}
\definecolor{mydarkblue}{rgb}{0,0.08,0.45}
\hypersetup{urlcolor=mydarkblue,citecolor=mydarkblue,linkcolor=mydarkblue}

\definecolor{lightgreen}{HTML}{33DD33}

\begin{document}

\runningtitle{Understanding and Mitigating Exploding Inverses in Invertible Neural Networks}

\runningauthor{Jens Behrmann$^*$, Paul Vicol$^*$, Kuan-Chieh Wang$^*$, Roger Grosse, J\"orn-Henrik Jacobsen}

\twocolumn[
\aistatstitle{Understanding and Mitigating Exploding Inverses \\ in Invertible Neural Networks}
\aistatsauthor{Jens Behrmann$^{\boldsymbol{*}1}$\ \  Paul Vicol$^{\boldsymbol{*}2, 3}$  Kuan-Chieh Wang$^{\boldsymbol{*}2 ,3}$ Roger Grosse$^{2, 3}$ J\"orn-Henrik Jacobsen$^{2,3}$}
\aistatsaddress{$^1$University of Bremen \And  $^2$University of Toronto \And $^3$Vector Institute \And $^{\boldsymbol{*}}$ Equal contribution}
]

\begin{abstract}
Invertible neural networks (INNs) have been used to design generative models, implement memory-saving gradient computation, and solve inverse problems. In this work, we show that commonly-used INN architectures suffer from exploding inverses and are thus prone to becoming numerically non-invertible. Across a wide range of INN use-cases, we reveal failures including the non-applicability of the change-of-variables formula on in- and out-of-distribution (OOD) data, incorrect gradients for memory-saving backprop, and the inability to sample from normalizing flow models. We further derive bi-Lipschitz properties of atomic building blocks of common architectures. These insights into the stability of INNs then provide ways forward to remedy these failures. For tasks where local invertibility is sufficient, like memory-saving backprop, we propose a flexible and efficient regularizer. For problems where global invertibility is necessary, such as applying normalizing flows on OOD data, we show the importance of designing stable INN building blocks.
\end{abstract}

\section{INTRODUCTION}
\label{sec:intro}

Invertible neural networks (INNs) have become a standard building block in the deep learning toolkit~\citep{papamakarios2019normalizing,donahue2019large}.
Invertibility is useful for training normalizing flow (NF) models with exact likelihoods \citep{dinh2014nice,dinh2016density}, increasing posterior flexibility in VAEs \citep{rezende2015variational}, learning transition operators in MCMC samplers \citep{song2017nice}, computing memory-efficient gradients \citep{gomez2017reversible}, allowing for bi-directional training \citep{grover2018flow}, solving inverse problems \citep{ardizzone2018analyzing} and analyzing adversarial robustness \citep{jacobsen2018excessive}. 

\begin{wrapfigure}{l}{0.18\textwidth}
\vspace{-6.4mm}
  \begin{center}
    \includegraphics[width=\linewidth]{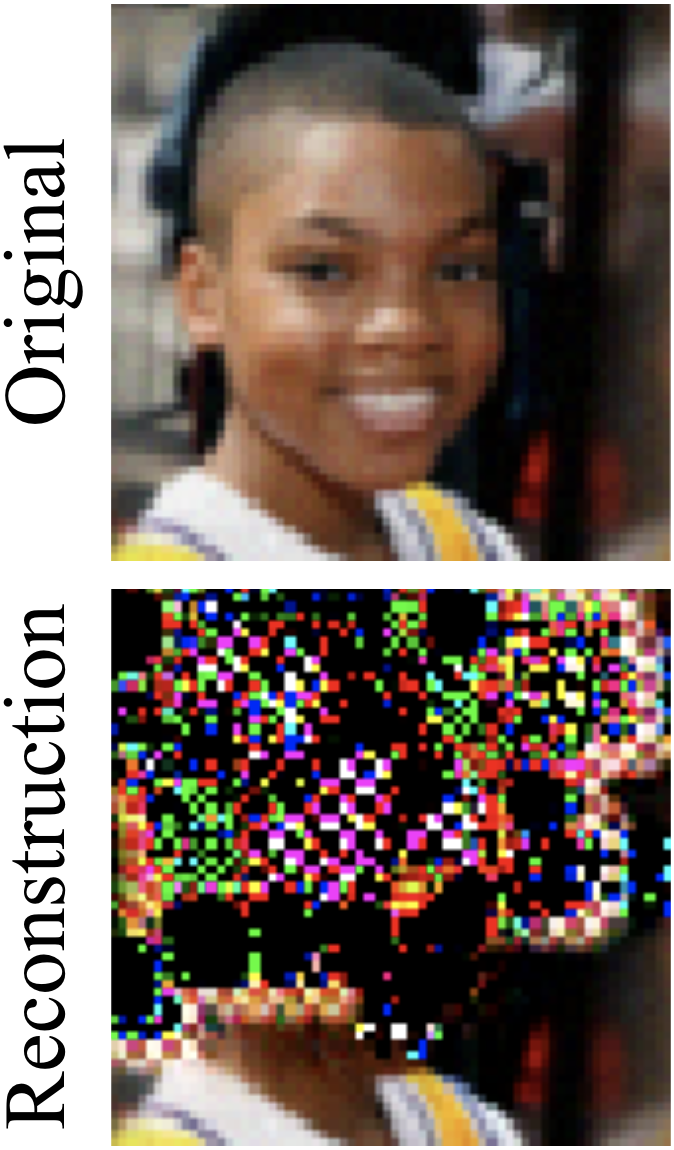}
  \end{center}
  \vspace{-2.4mm}
\end{wrapfigure}

All the aforementioned applications rely on the assumption that the theoretical invertibility of INNs carries through to their numerical instantiation. In this work, we challenge this assumption by probing their inverse stability in generative and discriminative modeling settings.
As a motivating example, on the left we show an image $x$ from within the dequantization distribution of a training example, and the reconstructed image $F^{-1}(F(x))$ from a competitive CelebA normalizing flow model $F$~\citep{kingma2018glow}.
In the same vein as exploding gradients in RNNs, here the inverse mapping explodes, leading to severe reconstruction errors up to \texttt{Inf/NaN} values.
The model exhibits similar failures both on out-of-distribution data and on samples from the model distribution (discussed in Section \ref{sec:nonInvFlows}).
Interestingly, none of these failures are immediately apparent during training.
Hence, NFs can silently become non-invertible, violating the assumption underlying their main advantages---exact likelihood computation and efficient sampling~\citep{papamakarios2019normalizing}.

\begin{wrapfigure}[11]{l}{0.18\textwidth}
\vspace{-0.8cm}
  \begin{center}
    \includegraphics[width=1.2\linewidth]{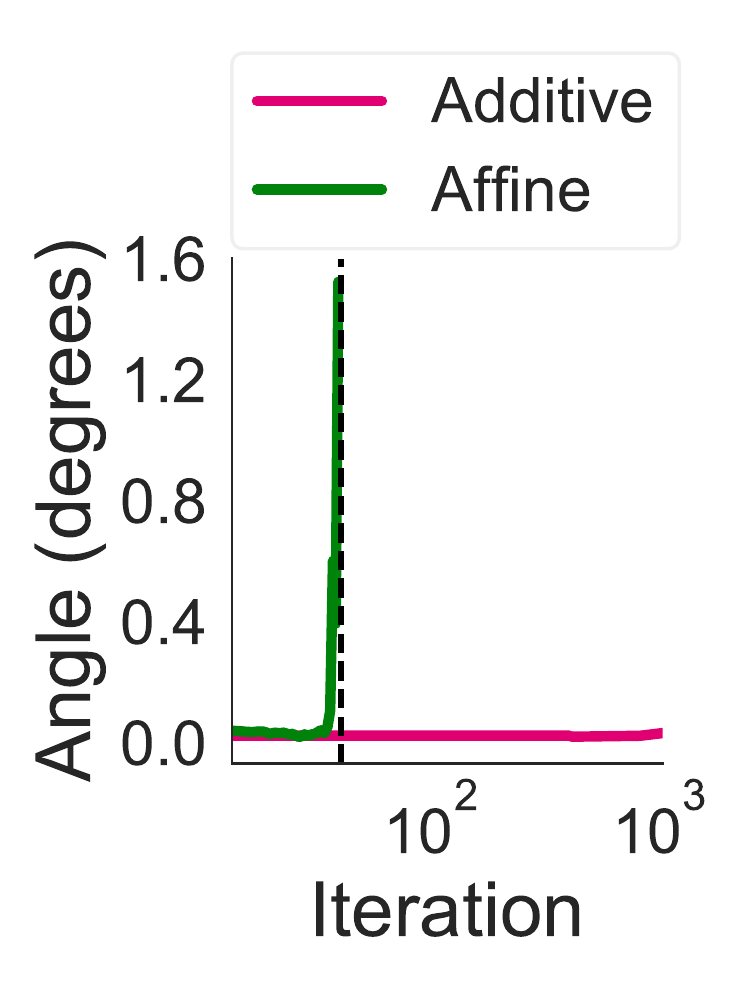}
  \end{center}
  \vspace{4mm}
\end{wrapfigure}
Memory-saving gradient computation~\citep{gomez2017reversible} is another popular application of INNs where exploding inverses can be detrimental.
On the left, we show the angle between (1) gradients obtained from standard backprop and (2) memory-efficient gradients obtained using the inverse mapping to recompute activations, during training of additive- and affine-coupling INN classifiers on CIFAR-10.
The affine model exhibits exploding inverses, leading to a rapidly increasing angle that becomes \texttt{NaN} after the dashed vertical line---making memory-efficient training infeasible---whereas the additive model is stable.
This highlights the importance of understanding the influence of different INN architectures on stability.
Different tasks may have different stability requirements:
NFs require the model to be invertible on training and test data and, for many applications, on out-of-distribution data as well. 
In contrast, memory-saving gradients only require the model to be invertible on the training data, to reliably compute gradients.

As we will show in our numerical experiments, the aforementioned failures are not a rare phenomenon, but are a concern across most application areas of INNs. To provide an understanding of these failures, we study the elementary components that influence the stability of INNs: 1) bi-Lipschitz properties of INN blocks (section \ref{sec:lipschitz-properties}); 2) the effect of the training objective (section \ref{sec:controlLocalStability}); and 3) task-specific requirements for the inverse computations (section \ref{sec:experiments} on global invertibility for change-of-variables vs. local invertibility for accurate memory-efficient gradients).
Finally, putting our theoretical analysis into practice, we empirically verify solutions to overcome exploding inverses. These solutions follow two main paradigms: 1) Enforcing global stability using Lipschitz-constrained INN architectures or 2) regularization to enforce local stability. In summary, we both uncover exploding inverses as an issue across most use-cases of INNs and provide ways to mitigate this instability.
Our code can be found at: \url{http://www.github.com/asteroidhouse/INN-exploding-inverses}.

\vspace{-0.1cm}
\section{INVERTIBLE NETWORKS}
\label{sec:background}
\vspace{-0.2cm}

Invertible neural networks (INNs) are bijective functions with a forward mapping $F: \mathbb{R}^d \rightarrow \mathbb{R}^d$ and an inverse mapping $F^{-1}: \mathbb{R}^d \rightarrow \mathbb{R}^d$. This inverse can be given in closed-form \citep{dinh2016density, kingma2018glow} or approximated numerically \citep{behrmann2019, song2019mintnet}. Central to our analysis of the stability of INNs is bi-Lipschitz continuity:

\begin{definition}[Lipschitz and bi-Lipschitz continuity]
A function $F: \mathbb{R}^{d} \rightarrow \mathbb{R}^{d}$ is called \emph{Lipschitz continuous} if there exists a constant $L \coloneqq \Lip(F)$ such that:
\begin{align}
\label{eq:defFowardLip}
    \|F(x_1) - F(x_2)\| \leq L \|x_1 - x_2\|, \quad \forall\; x_1, x_2 \in \mathbb{R}^{d}.
\end{align}
If an inverse $F^{-1}: \mathbb{R}^{d} \rightarrow \mathbb{R}^{d}$ and a constant $L^* \coloneqq \Lip(F^{-1})$ exists such that for all $y_1, y_2 \in \mathbb{R}^{d}$
\begin{align}
\label{eq:defInverseLip}
    \|F^{-1}(y_1) - F^{-1}(y_2)\| \leq L^* \|y_1 - y_2\|
    \end{align}
holds, then $F$ is called \emph{bi-Lipschitz continuous}. Furthermore, $F$ or $F^{-1}$ is called \emph{locally Lipschitz continuous in $[a,b]^d$}, if the above inequalities hold for $x_1, x_2$ or $y_1, y_2$ in the interval $[a,b]^d$. 
\end{definition}

As deep-learning computations are carried out with limited precision, numerical error is always introduced in both the forward and inverse passes. Instability in either pass will aggravate this imprecision, and can make an \textit{analytically} invertible network \textit{numerically non-invertible}.
If the singular values of the Jacobian of the inverse mapping can become arbitrarily large, we refer to this effect as an \textit{exploding inverse}.

To obtain a better understanding in the context of INNs, we first define coupling blocks and then study numerical error and instability in a toy setting. Let $I_1, I_2$ denote disjoint index sets of $\{1, \ldots, d\}$ to partition vectors $x \in \mathbb{R}^d$. For example, they denote a partition of feature channels in a convolutional architecture. Additive coupling blocks~\citep{dinh2014nice} are defined as:
\begin{align}
    \label{eq:additiveCoupling}
    F(x)_{I_1} &= x_{I_1} \\\nonumber
    F(x)_{I_2} &= x_{I_2} + t(x_{I_1})
\end{align}
and affine coupling blocks~\citep{dinh2016density} as:
\begin{align}
\label{eq:affineCoupling}
    F(x)_{I_1} &= x_{I_1}\\ \nonumber
    F(x)_{I_2} &= x_{I_2} \odot g(s(x_{I_1})) + t(x_{I_1}),
\end{align}
where $t,s:\mathbb{R}^{|I_1|} \rightarrow \mathbb{R}^{|I_2|}$ are Lipschitz continuous and $\odot$ denotes elementwise multiplication.
Hence, coupling blocks differ only in their scaling \footnote{Affine blocks scale the input with an elementwise function $g$ that has to be non-zero everywhere---common choices are sigmoid or $\exp(\cdot)$.}.

\begin{figure*}[t]
\centering
\begin{subfigure}[t]{0.23\linewidth}
	\centering
	\includegraphics[width=\linewidth]{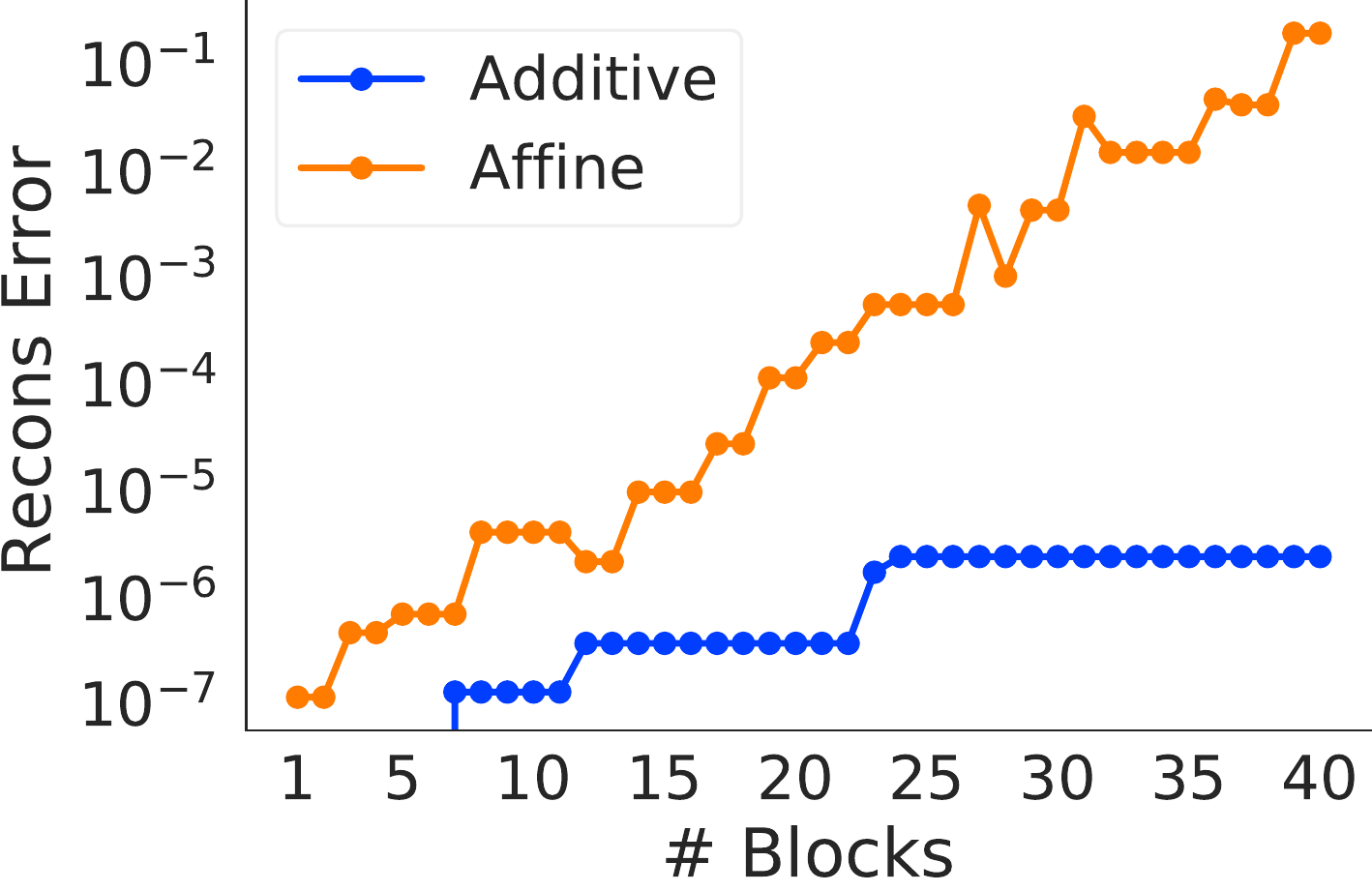}
\end{subfigure}
\begin{subfigure}[t]{0.23\linewidth}
	\centering
	\includegraphics[width=\linewidth]{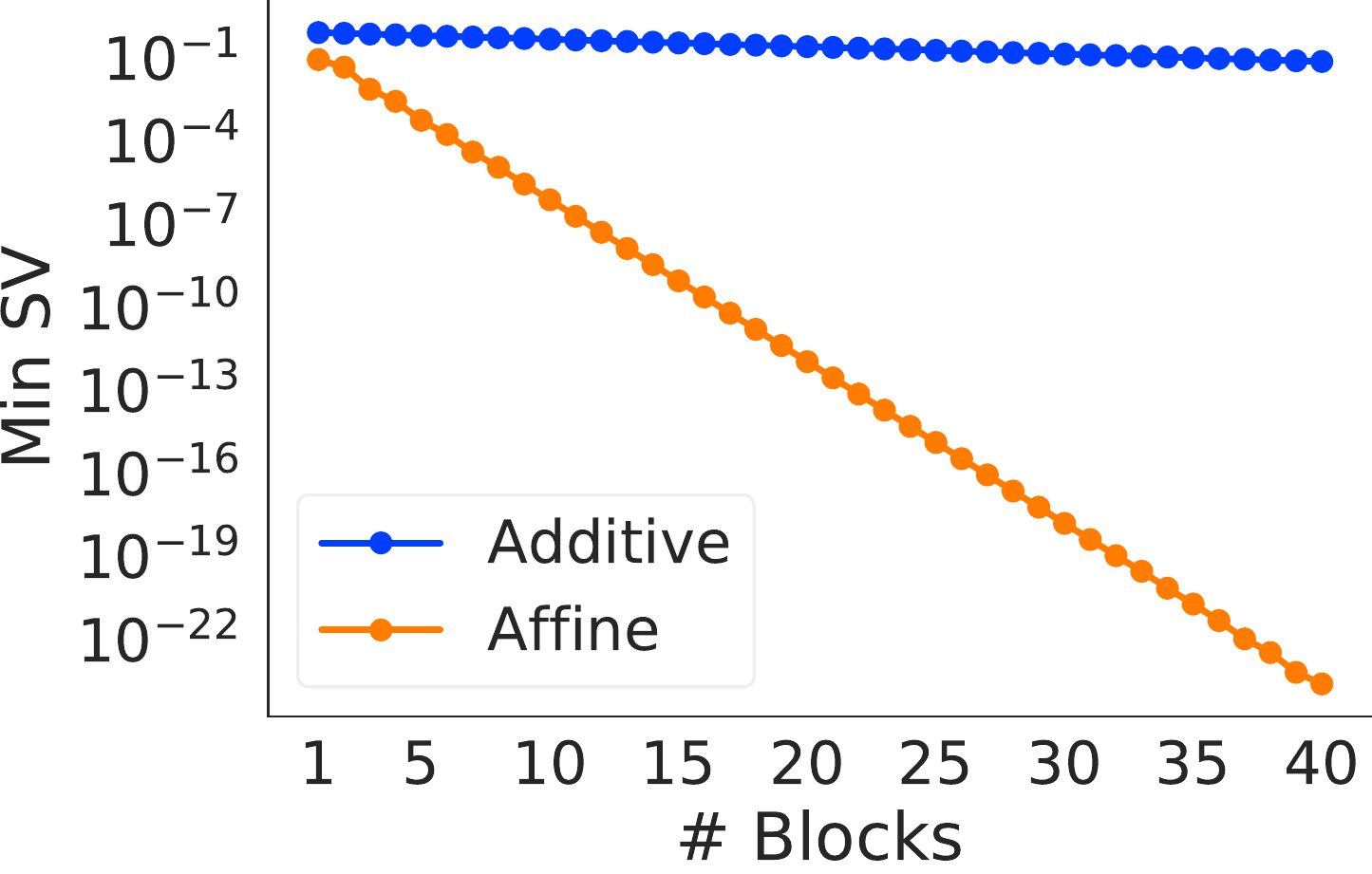}
\end{subfigure}
\begin{subfigure}[t]{0.23\linewidth}
	\centering
	\includegraphics[width=\linewidth]{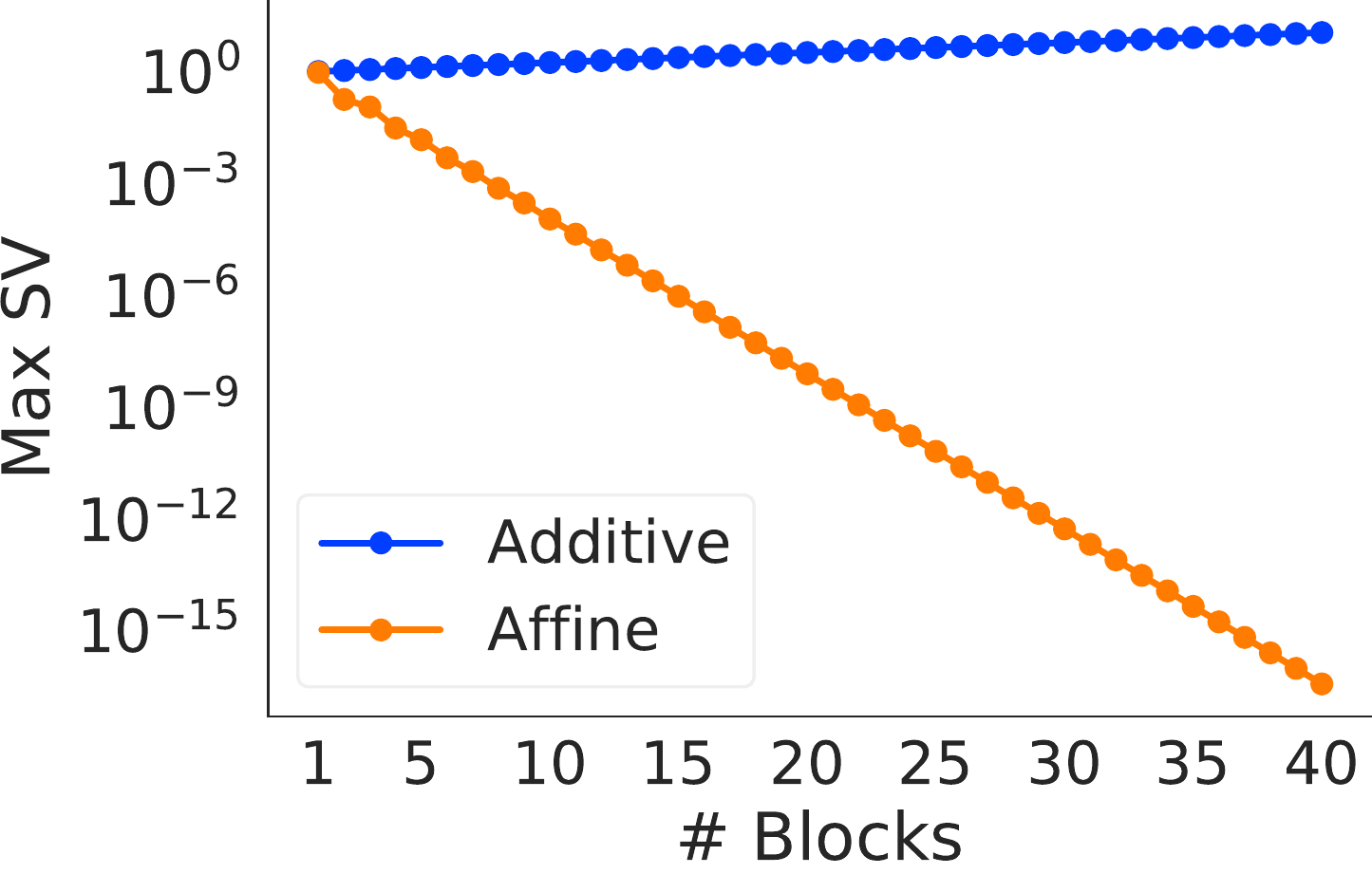}
\end{subfigure}
\begin{subfigure}[t]{0.23\linewidth}
	\centering
	\includegraphics[width=\linewidth]{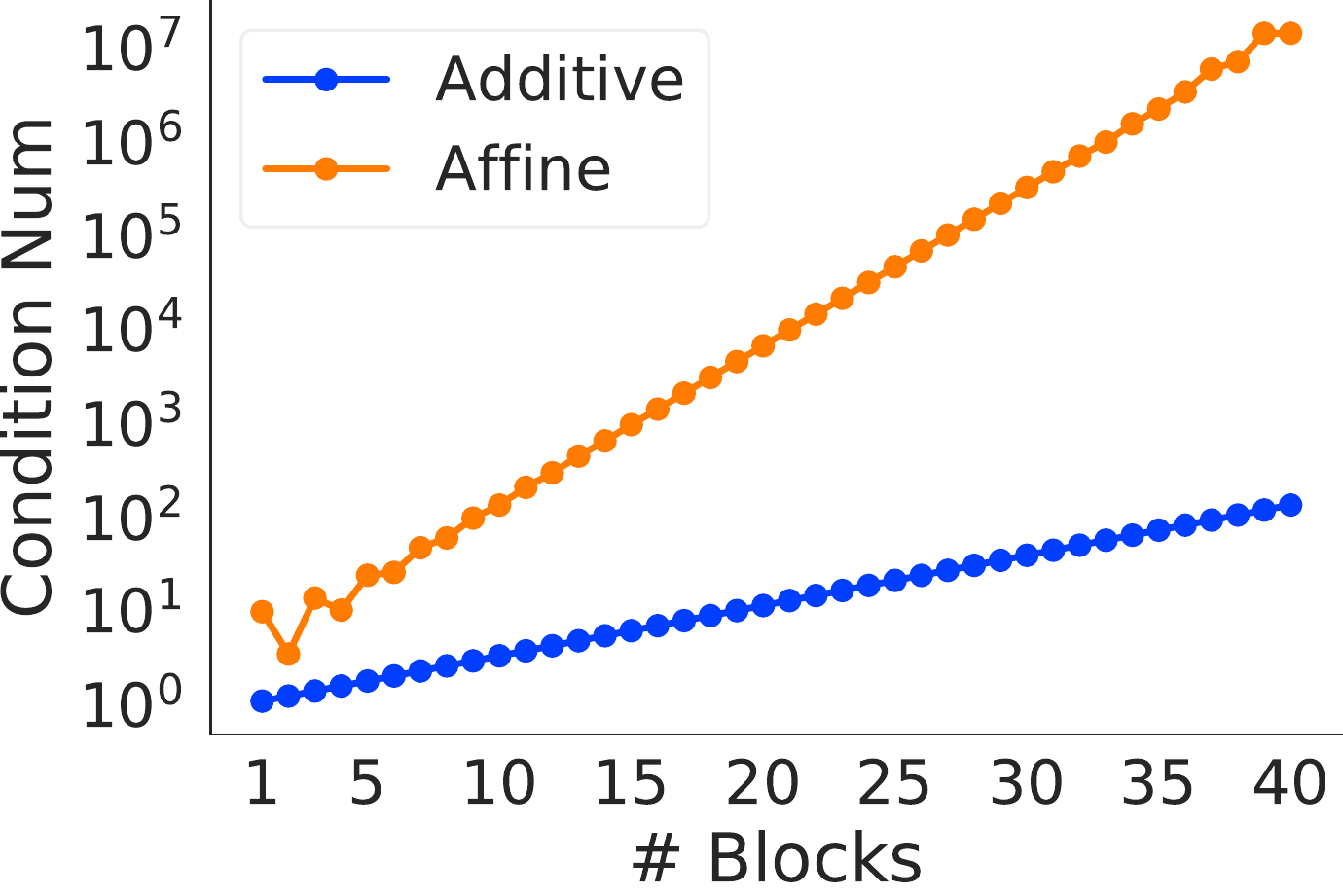}
\end{subfigure}
\caption{\textbf{Numerical error and instability in coupling blocks:} Consider a coupling block with $t(x_1) = 0.123456789$ and $g(s(x_1)) = 0.1$ (see Section \ref{sec:background}). The input is $(x_1, x_2) = (1, 1)$. From left to right, we show the: 1) reconstruction error; 2) minimum singular value $\sigma_{d}$; 3) maximum singular value $\sigma_{1}$; and 4) condition number, $\sigma_{1} / \sigma_{d}$, as a function of the number of INN blocks.}
\label{fig:toy-constructed}
\vspace{-0.2cm}
\end{figure*}

To understand how numerical error occurs in coupling blocks, consider the following example: consider $x \in \mathbb{R}^2$ and the trivial partition $I_1=1$ and $I_2=2$. Further, let $t(x_1) = 0.123456789$ and $g(s(x_1)) = 0.1$.
Floating-point operations introduce rounding errors when numbers from different scales are summed up. 
Figure \ref{fig:toy-constructed} visualizes the reconstruction error introduced at various depth for the given $t(x_1), g(s(x_1))$. In particular, the additive coupling INN has small numerical errors when reconstructing due to rounding errors in the summation. The affine blocks, however, show an exploding inverse effect since the singular values of the forward mapping tend to zero as depth increases. Thus, while being analytically invertible, the numerical errors and instability renders the network numerically non-invertible even in this simple example.

To formalize the connection of numerical errors and Lipschitz constants, consider $F(x) = z$ as the analytical exact forward computation and $F_\delta(x) = z + \delta =: z_\delta$ as the floating-point inexact computation with error $\delta$. In order to bound the error in the reconstruction due to the imprecision in the forward mapping, let $x_{\delta_1} = F^{-1}(z_\delta)$. Now consider:
\begin{align*}
         \|x - x_{\delta_1}\|_2 \leq \Lip(F^{-1}) \|z - z_\delta\|_2 = \Lip(F^{-1}) \|\delta\|_2,
 \end{align*}
where $\Lip(F^{-1})$ is used to bound the influence of the numerical error introduced in the forward mapping. Additionally, similarly to the forward mapping, the inverse mapping adds numerical imprecision. Thus, we introduce
$F_\delta^{-1}(z_\delta) = x_{\delta_1} + \delta_2 \coloneqq x_{\delta_2}$. Hence, we obtain the bound:
 \begin{align*}
         \|x - (x_{\delta_1} + \delta_2) \|_2 &\leq \|x - x_{\delta_1}\|_2 + \|\delta_2\|_2 \\
         &\leq \Lip(F^{-1}) \|z - z_\delta\|_2 + \|\delta_2\|_2 \\
         &= \Lip(F^{-1}) \|\delta\|_2 + \|\delta_2\|_2.
 \end{align*}
While obtaining quantitative values for $\delta$ and $\delta_2$ for a model as complex as deep INNs is hard, the above formalization still provides insights into the role of the inverse stability when reconstructing inputs. 
In sum, INNs are designed to be analytically invertible, but numerical errors are always present. How much these errors are magnified is bounded by their Lipschitz continuity (stability).

\section{STABILITY OF INVERTIBLE NEURAL NETWORKS}
\label{sec:stability}

We first discuss bi-Lipschitz properties of common INN building blocks, and how certain architectures suffer from exploding inverses in Section~\ref{sec:lipschitz-properties}.
Then we explore how to stabilize INNs globally (Section~\ref{sec:controlStability}) and locally (Section~\ref{sec:controlLocalStability}).

\vspace{-0.1cm}
\subsection{Lipschitz Properties of INN Blocks}
\label{sec:lipschitz-properties}
\vspace{-0.1cm}

Research on INNs has produced a large variety of architectural building blocks.
Here, we build on the work of \citet{behrmann2019}, that proved bi-Lipschitz bounds for invertible ResNets.
In particular, we derive Lipschitz bounds of coupling-based INNs and provide an overview of the stability of other common building blocks.
Most importantly for our subsequent discussion are the qualitative insights we can draw from this analysis.
Hence, we summarize these results in Theorem~\ref{thm:qualDiffCoupling} and provide the quantitative Lipschitz bounds as lemmas in Appendix~\ref{app:bounds}.

We start with our main result on coupling blocks, which are the most commonly used INN architectures. While they only differ in the scaling, this difference strongly influences stability:
\begin{theorem}[Stability of additive and affine blocks]
\label{thm:qualDiffCoupling}
Consider additive and affine blocks as in Eq.~\ref{eq:additiveCoupling} and Eq.~\ref{eq:affineCoupling} and assume the same function $t$. Further, let $t,s,g$ Lipschitz continuous, continuously differentiable and non-constant functions. Then the following differences w.r.t. bi-Lipschitz continuity hold:
\begin{enumerate}[label=(\roman*)]
    \item Affine blocks have strictly larger (local) bi-Lipschitz bounds than additive blocks.
    \item There is a global bound on $\Lip(F)$ and $\Lip(F^{-1})$ for additive blocks, but only local bounds for affine blocks, i.e. there exist no $\Lip(F) > 0$ and $\Lip(F^{-1})>0$ such that Eqs. \ref{eq:defFowardLip}, \ref{eq:defInverseLip} hold over $\mathbb{R}^d$.
\end{enumerate}
\end{theorem}
The proof is given in Appendix \ref{app:proofQualStatement}, together with upper bounds in Lemmas~\ref{lem:lipBoundAdditive} and \ref{lem:lipBoundAffine}.
Note that an upper bound on $\Lip(F)$ provides a lower bound on $\Lip(F^{-1})$ and vice versa. 
These differences offer two main insights. First, affine blocks can have arbitrarily large singular values in the inverse Jacobian, i.e.~an exploding inverse. Thus, they are more likely to become numerically non-invertible than additive blocks.
Second, controlling stability in each architecture requires fundamentally different approaches since additive blocks have global bounds, while affine blocks are not globally bi-Lipschitz.

In addition to the Lipschitz bounds of coupling layers, we provide an overview of known Lipschitz bounds of other common INN building blocks in Table~\ref{tab:overviewTable} (Appendix \ref{app:table-lipschitz}).
Besides coupling-based approaches we cover free-form approaches like Neural ODEs \citep{chen2018neural} and i-ResNets \citep{behrmann2019}.
Note that the bounds provide the worst-case stability and are primarily meant to serve as a guideline for the design of invertible blocks.

\vspace{-0.1cm}
\subsection{Controlling Global Stability of INNs}
\label{sec:controlStability}
\vspace{-0.1cm}

As we showed in the previous section, each INN building block has its own stability properties (see Table \ref{tab:overviewTable} for an overview).
The bi-Lipschitz bounds of additive coupling blocks can be controlled using a similar strategy to i-ResNets.
Via spectral normalization~\citep{miyato2018spectral}, it is possible to control the Lipschitz-constant of $t$ in Eq. \ref{eq:additiveCoupling}, which guarantees stability via the bounds from Lemma \ref{lem:lipBoundAdditive}. On the other hand, spectral normalization does not provide guarantees for affine blocks, as they are not globally bi-Lipschitz over $\mathbb{R}^d$ due to dependence on the range of the inputs $x$ (see Theorem \ref{thm:qualDiffCoupling}).

\paragraph{Modified Affine Scaling.} A natural way to increase stability of affine blocks is to consider different elementwise scaling functions $g$, see \citet{ardizzone2019guided} for an example of such a modification. In particular, avoiding scaling by small values strongly influences the inverse Lipschitz bound (see Lemma \ref{lem:lipBoundAffine}, Appendix~\ref{app:bounds}). 
Thus, a modification guided by Lemma \ref{lem:lipBoundAffine} would be to adapt the sigmoid scaling to output values in a restricted range such as $(0.5, 1)$ rather than the standard range $(0,1)$.
As we show in our experiments (Sections~\ref{sec:oodInstabilityExp} \& \ref{sec:flowInDistr}), this indeed improves stability. However, it may still suffer from exploding inverses as there is no global Lipschitz bound (see Theorem \ref{thm:qualDiffCoupling}).

\vspace{-0.1cm}
\subsection{Controlling Local Stability of INNs}
\label{sec:controlLocalStability}
\vspace{-0.1cm}
While the previous section aimed at controlling global stability, in this section we discuss how penalty functions and the training objective itself stabilize INNs locally, i.e. around inputs $x \in \mathbb{R}^d$.

\vspace{-0.1cm}
\subsubsection{Bi-Directional Finite Differences Regularization}
\label{sec:FDregularizer}
\vspace{-0.1cm}
Penalty terms on the Jacobian can be used to enforce local stability~\citep{sokolic17, hoffman2019robust}.
Their connection to Lipschitz bounds can be understood using the identity from \citet[Thm. 3.1.6]{federer1969geometric}: if $F:\mathbb{R}^d \rightarrow \mathbb{R}^d$ is Lipschitz continuous and differentiable, then we have:
\begin{align}
\label{eq:jacobianLipschitz}
    \Lip(F) &= \sup_{x \in \mathbb{R}^d} \|J_F(x)\|_2 \nonumber\\
    &= \sup_{x \in \mathbb{R}^d} \sup_{\|v\|_2=1} \| J_F(x) v\|_2,
\end{align}
where $J_F(x)$ is the Jacobian matrix of $F$ at $x$ and $\|J_F(x)\|_2$ denotes its spectral norm.
To approximate the RHS of Eq.~\ref{eq:jacobianLipschitz} from below, we obtain $v \in \mathbb{R}^d$ as $v / \|v\|_2$ with $v\sim \mathcal{N}(0,I)$, which uniformly samples from the unit sphere \citep{Muller59}.
For given $x$ and $v$, the term $\|J_F(x) v\|_2$ can be added to the loss function as a regularizer.
However, training with such a penalty term requires double-backpropagtion \citep{Drucker92, etmann19}, which makes recomputing the pre-activations during backprop via the inverse \citep{gomez2017reversible} (memory-saving gradients) difficult. 
Thus, in addition to randomization, we introduce a second approximation using finite differences as:
\begin{align}
\label{eq:finiteDiff}
    &\sup_{x \in \mathbb{R}^d} \sup_{\|v\|_2=1} \|J_F(x) v\|_2 \approx \nonumber\\ &\sup_{x \in \mathbb{R}^d} \sup_{\|v\|_2=1} \frac{1}{\varepsilon} \| F(x) - F(x+ \varepsilon v)\|_2,
\end{align}
with a step-size $\varepsilon > 0$.
The approximation error can be estimated via Taylor expansion and has to be traded off with \textit{catastrophic cancellation} due to subtracting near-identical float values \citep{An2011finiteDiff}. Since we aim at having both stable forward and inverse mappings, we employ this penalty on both directions $F$ and $F^{-1}$, and call it \textit{bi-directional finite differences regularization} (abbreviated FD).
Details about this architecture agnostic regularizer and its computational overhead are provided in Appendix \ref{app:classificationExpDetails}.

\begin{figure*}[t]
\centering
\begin{subfigure}[t]{0.3\linewidth}
	\centering
	\hspace{-0.5cm} Glow \\
    \vspace{-0.1cm}
    \begin{annotate}{\includegraphics[width=\linewidth]{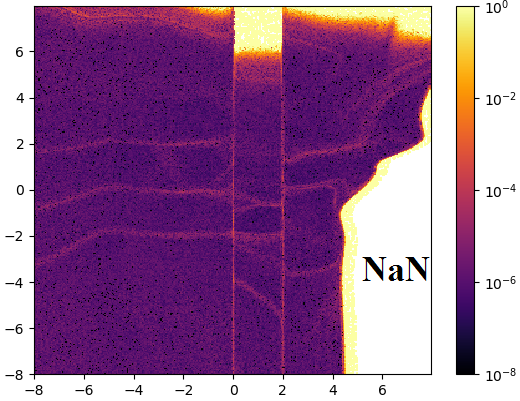}}{1}
    \draw[thick,dashed,green] (-1.18,-0.82) rectangle (0.66,0.98);
    \end{annotate}
\end{subfigure}
\begin{subfigure}[t]{0.3\linewidth}
	\centering
	Glow w/ Modified Scaling \\
    \vspace{-0.15cm}
    \begin{annotate}{\includegraphics[width=\linewidth]{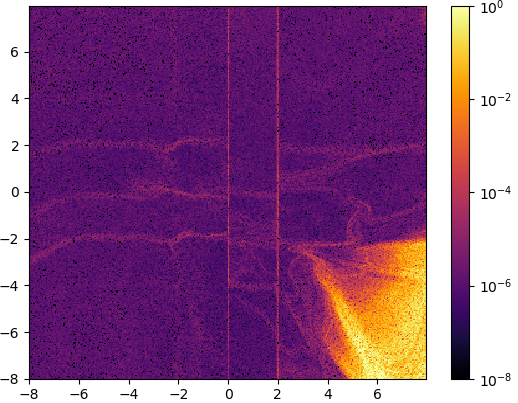}}{1}
    \draw[thick,dashed,green] (-1.18,-0.82) rectangle (0.66,0.98);
    \end{annotate}
\end{subfigure}
\begin{subfigure}[t]{0.3\linewidth}
	\centering
	\hspace{-0.4cm} Residual Flow \\
    \vspace{-0.15cm}
	\begin{annotate}{\includegraphics[width=\linewidth]{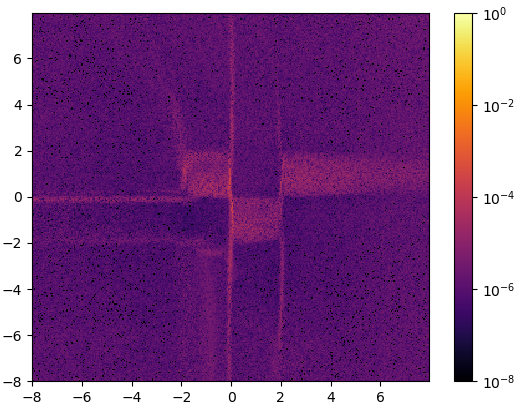}}{1}
    \draw[thick,dashed,green] (-1.18,-0.82) rectangle (0.66,0.98);
    \end{annotate}
\end{subfigure}
\vspace{-0.6cm}
\caption{\textbf{Reconstruction error on 2D checkerboard data.} \textbf{Left:} an affine model with standard sigmoid scaling in $(0, 1)$; \textbf{Middle:} a more stable affine model with scaling in (0.5, 1); \textbf{Right:} a Residual Flow model~\citep{chen2019residual}.
The {\color{lightgreen} green boxes} highlight the training data distribution $[-4, 4]$; we see that both affine models become unstable outside this distribution, while the Residual Flow remains stable.
}
\label{fig:checkerboardReconError}
\vspace{-0.5cm}
\end{figure*}

\begin{figure}[h]
\centering
	\begin{annotate}{\includegraphics[width=0.4956\linewidth, height=0.167 \textheight]{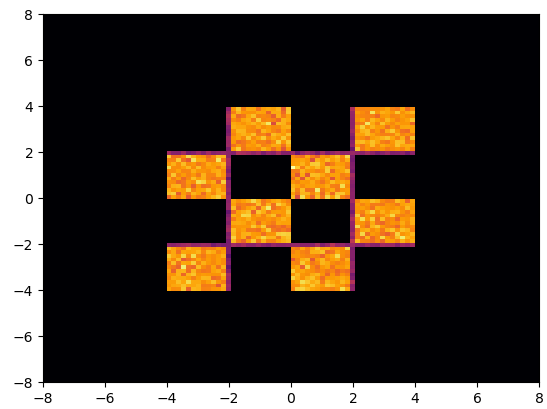}}{1}
    \draw[thick,dashed,green] (-0.80,-0.86) rectangle (1.04,0.94);
    \end{annotate}
    \vspace{-0.2cm}
\caption{\textbf{Samples from 2D checkerboard.}
}
\label{fig:samplesCheckerboard_main}
\vspace{-0.4cm}
\end{figure}

\subsubsection{Influence of the Normalizing Flow Loss on Stability}
\label{sec:CoVstability}

In addition to the INN architecture and local regularization such as bi-directional FD introduced in Section~\ref{sec:FDregularizer}, the training objective itself can impact local stability.
Here, we examine the stabilization effect of the commonly used normalizing flow (NF) objective~\citep{papamakarios2019normalizing}.
Consider a parametrized diffeomorphism $F_\theta: \mathbb{R}^d \rightarrow \mathbb{R}^d$ and a base distribution $p_Z$.
By a change-of-variables, we have for all $x \in \mathbb{R}^d$
\begin{align}
\label{eq:CoV}
    \log p_\theta(x) = \log p_Z(F_{\theta}(x)) + \log \left| \det J_{F_\theta}(x) \right|,
\end{align}
where $J_{F_\theta}(x)$ denotes the Jacobian of $F_\theta$ at $x$.
The log-determinant in Eq.~\ref{eq:CoV} can be expressed as:
\begin{align}
\label{eq:logdetSum}
    \log \left| \det J_{F_\theta}(x) \right| = \sum_{i=1}^d \log \sigma_i(x),
\end{align}
where $\sigma_i(x)$ denotes the $i$-th singular value of  $J_{F_\theta}(x)$. Thus, minimizing the negative log-likelihood as $\min_\theta - \log p_\theta(x)$ involves maximizing the sum of the log singular values (Eq.~\ref{eq:logdetSum}).
Due to the slope of the logarithmic function $\log(x)$, very small singular values are avoided more strongly than large singular values are favored.
Thus, the inverse of $F_\theta$ is encouraged to be more stable than the forward mapping.
Furthermore when using $Z \sim \mathcal{N}(0,I)$ as the base distribution, we minimize:
\begin{align*}
   - \log p_Z(F_\theta(x)) \propto \|F_\theta(x)\|_2^2,
\end{align*}
which bounds the $\ell_2$-norm of the outputs of $F_\theta$. 
Due to this effect, large singular values are avoided and the mapping $F_\theta$ is further locally stabilized.

Thus, the two terms of the normalizing flow objective have complementary effects on stability: the log-determinant increases all singular values, but has a stronger effect on small singular values than on large ones, improving inverse stability, while the base term encourages the output of the function to have small magnitude, improving forward stability. If additional stability is required, bounding the $\ell_2$-norm of intermediate activations would further avoid fluctuations, where a subflow exhibits high magnitude that is cancelled out by the subsequent subflow. The effect of the NF objective, however, acts only on the training data $x \in \mathbb{R}^d$ and is thus not able to globally stabilize INNs, as we show in our experiments (Section~\ref{sec:nonInvFlows}).

\vspace{-0.2cm}
\section{EXPERIMENTS}
\vspace{-0.2cm}

\label{sec:experiments}
First, we show that exploding inverses are a concern across most application areas of INNs. Second, we aim to provide ways to mitigate instability. For this we conduct experiments on two tasks: generative modeling with normalizing flows (Section \ref{sec:nonInvFlows}) and memory-efficient gradient computation for supervised learning (Section \ref{sec:classification}). Due to the growing body of INN architectures we had to restrict our experimental study to a subset of INNs: additive/ affine coupling blocks \citep{dinh2014nice, dinh2016density} and Residual Flows \citep{chen2019residual}. This particular choice was guided by the simplicity of coupling blocks and by the close link of Residual Flows to stability.

\vspace{-0.2cm}
\subsection{Non-Invertibility in Normalizing Flows}
\label{sec:nonInvFlows}
\vspace{-0.2cm}
Here we show that INNs can become numerically non-invertible even when trained with the normalizing flow (NF) loss (despite encouraging local stability, see Section \ref{sec:CoVstability}). We study this behavior on out-of-distribution data and identify the exploding inverse effect in the data and model distribution.

\vspace{-0.1cm}
\subsubsection{Instability on Out-of-Distribution Data}
\label{sec:oodInstabilityExp}
\vspace{-0.1cm}

\begin{figure*}[t]
\setlength{\tabcolsep}{3pt}
\centering
\begin{tabular}{cc}
    \small
    \begin{tabular}{cccc}
    & \textbf{Gaussian} &
    \textbf{Texture} & \textbf{tinyImageNet}  \\
    \rotatebox{90}{\ \ \ \ \ \ Original} &
    \includegraphics[trim={1.17cm 1.2cm 0 0},clip,width=0.13\textwidth]{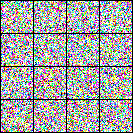} &
    \includegraphics[trim={1.17cm 1.2cm 0 0},clip,width=0.13\textwidth]{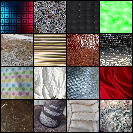} &
    \includegraphics[trim={1.17cm 1.2cm 0 0},clip,width=0.13\textwidth]{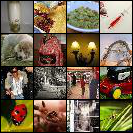}
    \\
    \rotatebox{90}{\ Reconstructed} & 
    
    \includegraphics[trim={3.8cm 3.8cm 0 0},clip,width=0.13\textwidth]{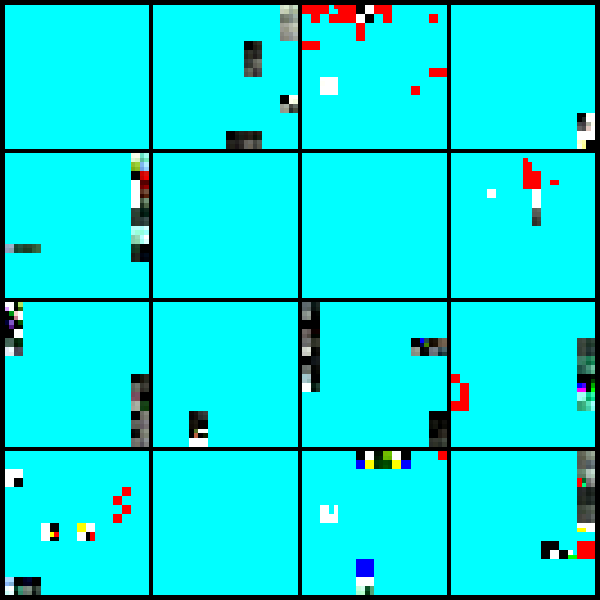} 
    &
    \includegraphics[trim={3.8cm 3.8cm 0 0},clip,width=0.13\textwidth]{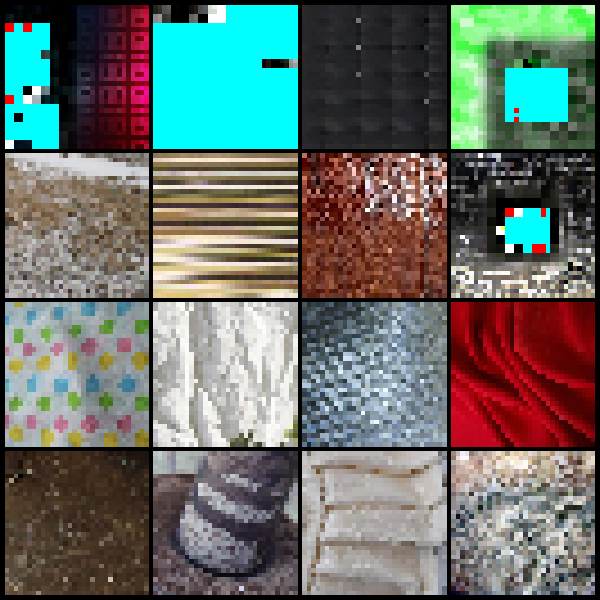} &
    \includegraphics[trim={3.8cm 3.8cm 0 0},clip,width=0.13\textwidth]{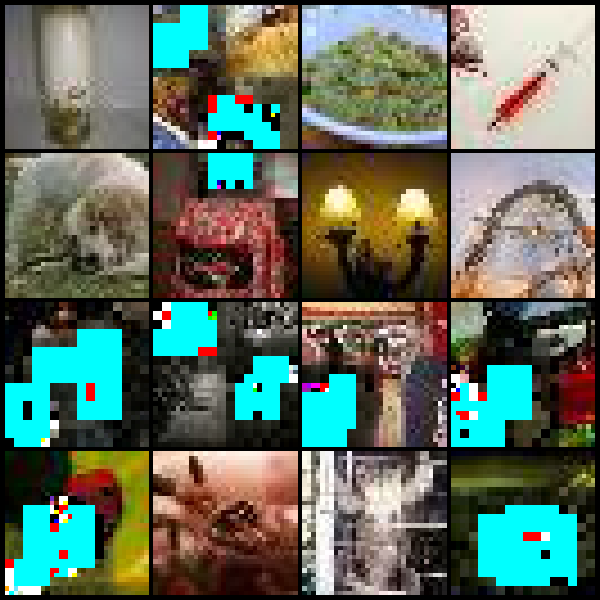}
    \\
    \end{tabular}
    &
    \hspace{1cm}
    \begin{tabular}{ccccc}
    \scriptsize
     &  \multicolumn{2}{c}{\textbf{Glow}} & \multicolumn{2}{c}{\textbf{ResFlow}} \\
    \midrule
    \textbf{Dataset}    & \textbf{\% Inf} & \textbf{Err} & \textbf{\% Inf} & \textbf{Err}  \\
    \midrule
    CIFAR-10                           & 0     & 6.3e-5   & 0 & 2.9e-2  \\
    \midrule
    Uniform                            & 100   & -        & 0 & 1.7e-2  \\
    Gaussian                           & 100   & -        & 0 & 7.2e-3  \\
    Rademacher                         & 100   & -        & 0 & 1.9e-3  \\
    SVHN                               & 0     & 5.5e-5   & 0 & 7.3e-2  \\
    Texture                            & 37.0  & 7.8e-2   & 0 & 2.0e-2  \\
    Places                             & 24.9  & 9.9e-2   & 0 & 2.9e-2  \\
    tinyImageNet                       & 38.9  & 1.6e-1   & 0 & 3.5e-2  \\
    \bottomrule
    \end{tabular}

\end{tabular}
\vspace{-0.1cm}
\caption{\textbf{Left:} Reconstructions of OOD data, using a CIFAR-10 pre-trained Glow model. Broken regions (\texttt{NaN} or \texttt{Inf}) in the reconstructions are {\color{cyan} plotted in cyan}. \textbf{Right:} Mean $\ell_2$ reconstruction errors on in-distribution (CIFAR-10) and out-of-distribution data, for a pre-trained Glow and Residual Flow.
We used three synthetic noise datasets \{Uniform, Gaussian, Rademacher\} as well as SVHN~\citep{netzer2011reading}, Texture~\citep{cimpoi14describing}, Places~\citep{zhou2017places}, and tinyImageNet.
We used 10,000 samples from each OOD dataset, and we report 1) the percentage of images that yielded \texttt{Inf} reconstruction error; and 2) the mean reconstruction error for the non-\texttt{Inf} samples, see \texttt{Err} column.
}
\label{fig:ood-reconstructions}
\vspace{-0.4cm}
\end{figure*}

\vspace{-0.15cm}
Because NFs allow for efficient likelihood computation, they have been used in several likelihood-based approaches for out-of-distribution (OOD) detection~\citep{nalisnick2018deep, fetaya2020understanding}. Here we show, however, that certain classes of flows can be numerically non-invertible on OOD data, implying that the likelihoods computed on such data are not meaningful because the change-of-variables formula no longer applies.
\vspace{-0.4cm}

\paragraph{2D Checkerboard.}
First we consider a 2D checkerboard distribution (see samples in Figure \ref{fig:samplesCheckerboard_main} and further details in Appendix \ref{app:2DcheckerboardExpDetails}).
Despite being ill-posed due to the discontinuous density (jumps at the edges of the checkerboard), it is often used as a benchmark for NFs~\citep{chen2019residual,grathwohl2019ffjord}.
The discontinuity of the dataset is manifested in two main ways in Figure \ref{fig:checkerboardReconError}: 1) at these jumps, the model becomes unstable, which leads to larger reconstruction errors (see slightly expressed grid-like pattern), 2) affine models can become non-invertible outside the data domain.
Figure \ref{fig:checkerboardReconError} further shows reconstruction errors from a modified affine model---which is more stable, but still suffers from exploding inverses in OOD areas---and a Residual Flow \citep{chen2019residual}, which has low reconstruction error globally, consistent with our stability analysis.
\vspace{-0.2cm}

\begin{table}[t]
    \begin{subfigure}[t]{0.65\linewidth}
\small
\setlength{\tabcolsep}{2pt}
\centering
\begin{tabular}{rccccc}
\toprule 
                    & Additive                 & Affine                   & Mod. Affine \\
\hline
Data BPD (Train)          & \multicolumn{1}{c}{3.29} & \multicolumn{1}{c}{3.27} & \multicolumn{1}{c}{3.25}  \\
Data BPD (Test)           & \multicolumn{1}{c}{3.55} & \multicolumn{1}{c}{3.51} & \multicolumn{1}{c}{3.5}   \\
Reliable Sample BPD          &  \cmark                       & \xmark                     &  \cmark           \\
No Visible Sample Recon Err. &  \cmark                       & \xmark                     &  \cmark          \\
\bottomrule
\end{tabular}
\end{subfigure}
    \vspace{-0.1cm}
    \caption{\textbf{Flow results on CIFAR10.}  Bits-per-dimension (BPDs) reported at 100k updates. The bottom 2 rows show whether problems occur during training. The affine model was unstable w.r.t. model samples, and the additive and modified affine models are stable.}
    \label{table:model-distr}
    \vspace{-0.4cm}
\end{table}

\paragraph{CIFAR-10 OOD.}
Next, we evaluated CIFAR-10 pre-trained Glow and Residual Flow models on a set of OOD datasets from~\citet{hendrycks2016baseline,liang2017enhancing,nalisnick2018deep}.
\footnote{For Glow and Residual Flows, we used the pre-trained models from \url{https://github.com/y0ast/Glow-PyTorch} and \url{https://github.com/rtqichen/residual-flows}, respectively.}
Figure~\ref{fig:ood-reconstructions} shows qualitative reconstruction errors for Glow on three OOD datasets, as well as the numerical reconstruction errors for Glow and ResFlow models on all OOD datasets.
While the Residual Flow was always stably invertible, Glow was non-invertible on all datasets except SVHN.
This indicates that OOD detection methods based on Glow likelihoods may be unreliable.
In Appendix~\ref{app:OODsampleExp}, we provide additional details, as well as a comparison of the stability of additive and affine coupling models on OOD data during training.

\subsubsection{Instability in the Data Distribution and further Failures}
\label{sec:flowInDistr}

\begin{figure}[t]
\centering
\begin{subfigure}[t]{0.5\linewidth}
	\centering
	\includegraphics[width=1.05\linewidth]{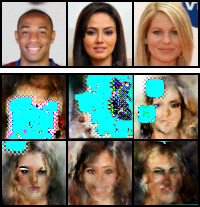}
\end{subfigure}
\caption{\textbf{Instability in model-distribution.} The top row shows data from CelebA64 and the two bottom rows show samples from an affine model during training at epoch 20. \texttt{NaN} pixels are visualized in {\color{cyan} cyan}.
}
\label{fig:samplingIssues}
\vspace{-0.3cm}
\end{figure}

\begin{figure*}[h]
    \vspace{-0.3cm}
    \centering
    \begin{tabular}{l}
    \hspace{4.6cm} Unregularized \hspace{2.0cm} Regularized \\
    \includegraphics[trim={0 0 0 5.5cm},clip,width=0.5\linewidth]{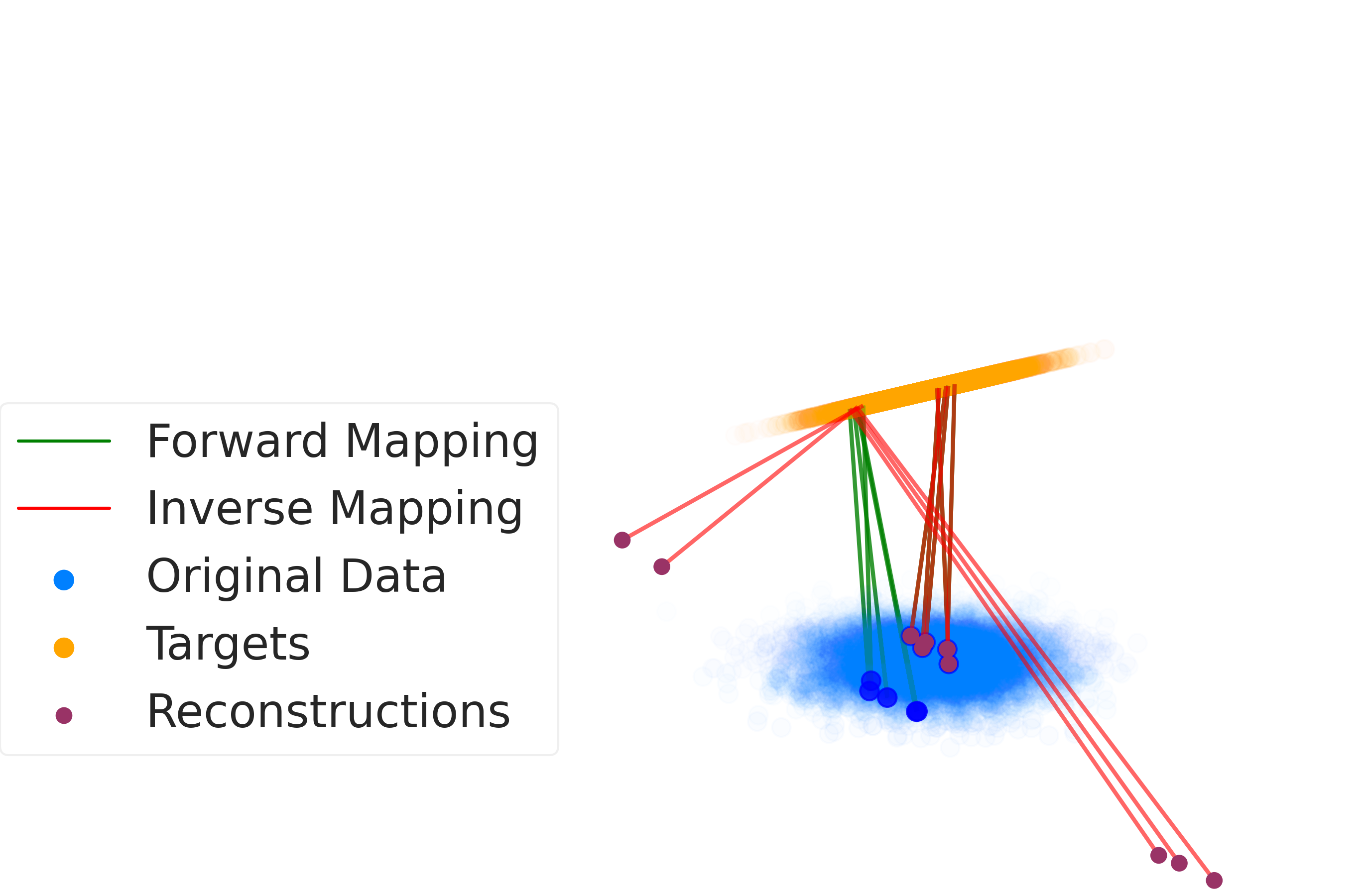}
    \includegraphics[trim={12cm 0 0 5.5cm},clip,width=0.25\linewidth]{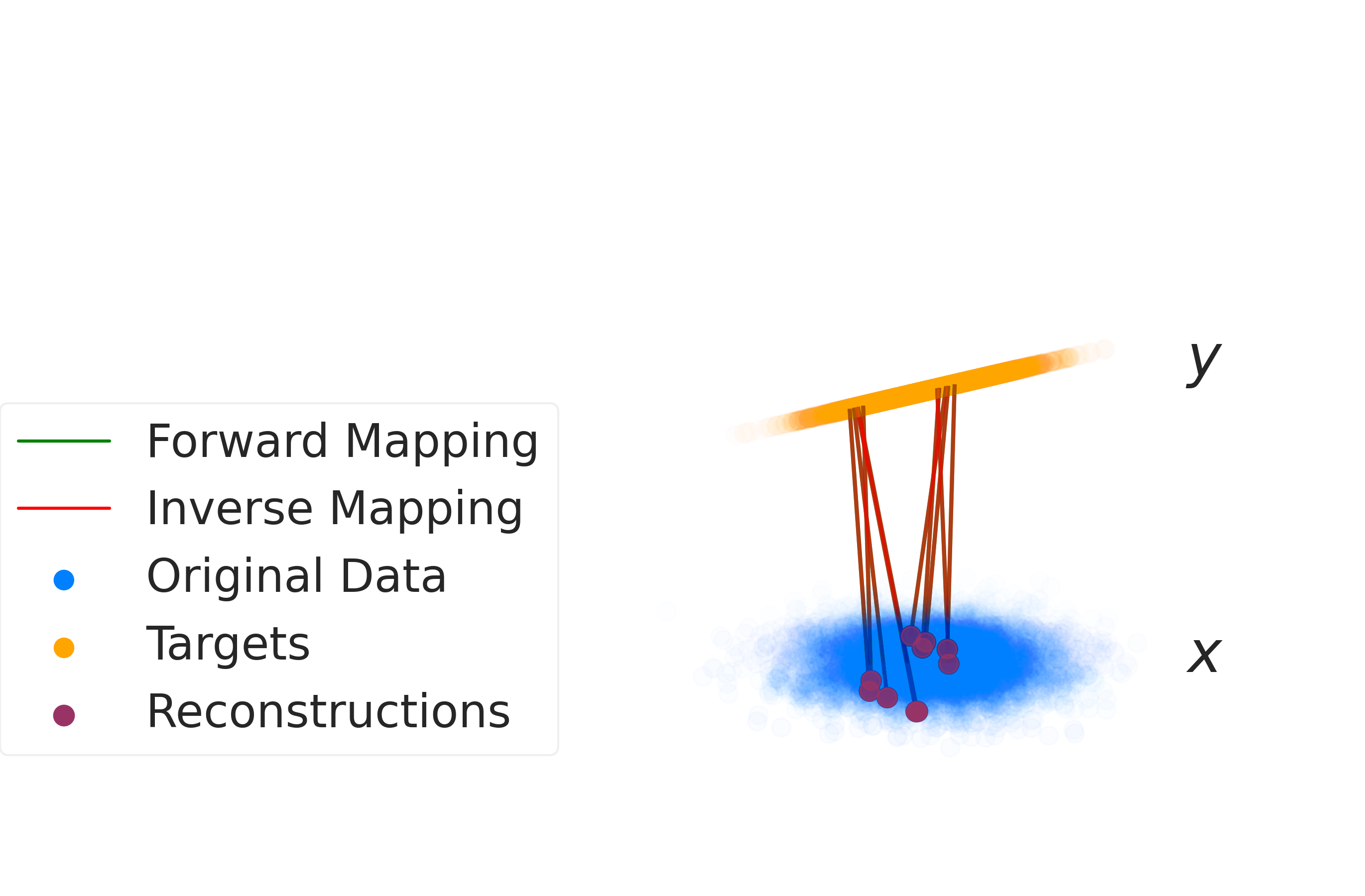}
    \end{tabular}
    \vspace{-0.25cm}
    \caption{\textbf{Exploding inverses on a 2D regression task.}
    A Glow model is trained to map between two 2D Gaussian distributions $(x_1, x_2) \to (y_1, y_2)$, where $y_2$ has low variance, so that we are essentially mapping from 2D space onto a 1D subspace.
\textbf{Left:} An unregularized model exhibits exploding inverses, illustrated by the points that are mapped far outside the original data distribution by the inverse mapping. \textbf{Right:} Regularizing the model by adding the normalizing flow objective with a small coefficient (1e-8) stabilizes the mapping.} 
    \label{fig:toyRegression}
    \vspace{-0.5cm}
\end{figure*}

In this section, we further identify failures within the model and data distributions.
In particular, we study the stability of the inverse on the model distribution by sampling $z$ from the base distribution (rather than obtaining $z$ via a forward pass on some data). 
For this, we trained both additive/affine models on CIFAR-10, and an affine model on CelebA64 (see details in Appendix \ref{app:OODsampleExp}).
We report quantitative results during training on CIFAR-10 in Table~\ref{table:model-distr}. The affine model was unstable w.r.t. model samples.
Furthermore in Figure~\ref{fig:samplingIssues}, we show samples from the affine CelebA64 model, which has \texttt{NaN} values in multiple samples.

\vspace{-0.1cm}
\paragraph{Non-Invertible Inputs within the Dequantization Region.} We can further expose non-invertibility even in the data distribution. By optimizing within the dequantization distribution of a datapoint we are able to find regions that are poorly reconstructed by the model. Note that the inputs found this way are valid samples from the training data distribution. See Appendix~\ref{app:invertibilityAttack} for details.

\vspace{-0.1cm}
\paragraph{Instability beyond the NF loss.}
In Appendix \ref{app:flowgan} we provide results on bi-directional training with INNs in the Flow-GAN setting \citep{grover2018flow}. While additive models did not show instabilities thus far, they exhibit exploding inverses when trained solely adversarially. Lastly, this exploration demonstrates how our analysis can be leveraged to choose appropriate tools for stabilizing INNs, and in this case improve model samples while retaining competitive bits-per-dimension (BPDs).

\subsection{Supervised Learning with Memory Efficient Gradients}
\label{sec:classification}
\vspace{-0.1cm}

For supervised learning, INNs enable memory-efficient training by re-computing intermediate activations in the backward pass, rather than storing them in memory during the forward pass~\citep{gomez2017reversible}.
This enables efficient large-scale generative modeling \citep{donahue2019large} and high-resolution medical image analysis \citep{etmann2020iunets}.
Re-computing the activations during the backward pass, however, relies on the inverse being numerically precise locally around the training data.
This is a weaker requirement than the global stability we desire for NFs such that they can be applied to OOD data.
However, in this section we show that even this weaker requirement can be violated, and how to mitigate these failures by local regularization as discussed in Section \ref{sec:controlLocalStability}. 

\vspace{-0.2cm}
\paragraph{Toy 2D Regression.}
In contrast to NFs, where the likelihood objective encourages local stability, there is no default mechanism to avoid unstable inverses in supervised learning (e.g., classification or regression).
As an example, consider a simple 2D regression problem visualized in Figure \ref{fig:toyRegression}.
Here, the targets $y$ lie almost on a 1D subspace, which requires the learned mapping to contract from 2D to 1D.
Even in such a simple task, a Glow regression model becomes non-invertible, as shown by the misplaced reconstructions for the unregularized model.
Additional details are provided in Appendix~\ref{app:2d-toy-regression-details}.
This illustrates the importance of adding regularization terms to supervised objectives to maintain invertibility, as we do next for memory-efficient training on CIFAR-10.

\vspace{-0.4cm}
\paragraph{CIFAR-10 Classification.}
Here we show that INN classifiers on CIFAR-10 can become non-invertible---making it impossible to compute accurate memory-saving gradients---and that local regularization by adding either the finite differences (FD) penalty or the normalizing flow (NF) objective with a small coefficient stabilizes these models, enabling memory-efficient training.
We focused on additive and affine models with architectures similar to Glow~\citep{kingma2018glow}, with either $1 \times 1$ convolutions or shuffle permutations and ActNorm between building blocks.
We used only coupling approaches, because i-ResNets \citep{behrmann2019} are not suited for memory-efficient gradients due to their use of an expensive iterative inverse.
Experimental details and extended results are provided in Appendix~\ref{app:classificationExpDetails}.

\begin{table*}
\centering
\small
\begin{tabular}{cc|cccccc}
\toprule
  \textbf{Model} & \textbf{Regularizer} & \textbf{Inv?} & \textbf{Test Acc} & \textbf{Recons. Err.} & \textbf{Cond. Num.} & \textbf{Min SV} & \textbf{Max SV}  \\
  \midrule
  \multirow{3}{*}{Additive Conv} & None  & \cmark & 89.73 & 4.3e-2 & 7.2e+4 & 6.1e-2 & 4.4e+3 \\
                                 & FD    & \cmark & 89.71 & 1.1e-3 & 3.0e+2 & 8.7e-2 & 2.6e+1 \\
                                 & NF    & \cmark & 89.52 & 9.9e-4 & 1.7e+3 & 3.9e-2 & 6.6e+1 \\
 \hline

\multirow{3}{*}{Affine Conv}
 & None & {\color{red}\xmark} & 89.07 & {\color{red}\texttt{Inf}}    & {\color{red}8.6e14} & {\color{red}1.9e-12} & 1.7e+3 \\
                              & FD & \cmark & 89.47 & 9.6e-4 & 1.6e+2 & 9.6e-2 & 1.5e+1  \\
                              & NF & \cmark & 89.71 & 1.3e-3 & 2.2e+3 & 3.5e-2 & 7.7e+1 \\
\bottomrule

\end{tabular}
\caption{\textbf{Effect of the finite differences (FD) and normalizing flow (NF) regularizers} when training additive and affine INN Glow architectures using $1 \times 1$ convolutions for CIFAR-10 classification.
For each setting, we report the test accuracy, the numerical reconstruction error, and the condition number and min/max singular values (SVs) of the Jacobian of the forward mapping.
While the additive model was always stable, the unregularized affine model became highly unstable, with \texttt{Inf} reconstruction error; we observe that instability arises from the inverse mapping, as the min SV is 1.9e-12.
}
\label{table:classification}
\vspace{-0.4cm}
\end{table*}

In Table \ref{table:classification}, we compare the performance and stability properties of unregularized additive and affine models, as well as regularized versions using either FD or NF penalties (Appendix~\ref{app:classificationExpDetails} shows the effects of different regularization strengths).
Note that we do not aim to achieve SOTA accuracy, and these accuracies match those reported for a similar Glow classifier by~\citet{behrmann2019}.
In particular, we observe how affine models suffer from exploding inverses and are thus not suited for computing memory-efficient gradients.
Both the NF and FD regularizers mitigate the instability, yielding similar test accuracies to the unregularized model, while maintaining small reconstruction errors and condition numbers.\footnote{Note that the training with FD was performed memory-efficiently, while the None and NF settings were trained using standard backprop.}
In Appendix~\ref{app:classificationExpDetails}, we show that regularization keeps the angle between the true gradient and memory-saving gradient small throughout training.
The computational overhead of the FD regularizer is only 1.26$\times$ that of unregularized training (Table~\ref{tab:timeFD} in App. \ref{app:classificationExpDetails}).
We also experimented with applying the FD regularizer only to the inverse mapping, and while this can also stabilize the inverse and achieve similar accuracies, we found bi-directional FD to be more reliable across architectures.
Applying regularization once every 10 iterations, bi-FD is only 1.06$\times$ slower than inverse-FD.
As the FD regularizer is architecture agnostic and conceptually simple, we envision a wide-spread use for memory-efficient training.

\vspace{-0.2cm}
\section{RELATED WORK}
\vspace{-0.4cm}
\paragraph{Invertibility and Stability of Deep Networks.} The inversion from activations in standard neural networks to inputs has been studied via optimization in input space \citep{Mahendran2014UnderstandingDI} or by linking invertibility and inverse stability for ReLU-networks \citep{jensRelu}.
However, few works have studied the stability of INNs:  \citet{gomez2017reversible} examined the numerical errors in the gradient computation when using memory-efficient backprop. Similarly to our empirical analysis, \citet{jacobsen2018irevnet} computed the SVD of the Jacobian of an i-RevNet and found it to be ill-conditioned.
Furthermore, i-ResNets~\citep{behrmann2019} yield bi-Lipschitz bounds by design. Lastly, \citet{chang2018reversible} studied the stability of reversible INN dynamics in a continuous framework.
In contrast, the stability of neural networks has been of major interest due to the problem of exploding/vanishing gradients as well as for training Wasserstein GANs \citep{pmlr-v70-arjovsky17a}. Furthermore, adversarial examples \citep{szegedy2013intriguing} are strongly tied to stability and inspired our invertibility attack (Section \ref{sec:flowInDistr}). 
\vspace{-0.15cm}

\vspace{-0.2cm}
\paragraph{Improving Stability of Invertible Networks.} Instability in INNs has been noticed in other works, yet without a detailed treatment. For example, \citet{putzkyRim} proposed to employ orthogonal $1\times 1$ convolutions to obtain accurate memory-efficient gradients and \citet{etmann2020iunets} used weight normalization for stabilization. Our finite differences regularizer, on the other hand, is architecture agnostic and could be used in the settings above. Furthermore, \citet{ardizzone2019guided} considered modified scaling in affine models to improve their stability. Neural ODEs \citep{chen2018neural} are another way to design INNs, and research their stability. \citet{finlay2020train, YAN2020On, massaroli2020stable} provide further insights into designing principled and stable INNs.

\vspace{-0.2cm}
\paragraph{Fixed-Point Arithmetic in INNs.}
\citet{maclaurin2015gradient,mackay2018reversible} implement invertible computations using fixed-point numbers, with custom schemes to store information that is lost when bits are shifted, enabling exact invertibility at the cost of memory usage. As \citet{gomez2017reversible} point out, this allows exact numerical inversion when using additive coupling independent of stability. However, our stability analysis aims for a broadly applicable methodology beyond the special case of additive coupling.

\vspace{-0.2cm}
\paragraph{Invertible Building Blocks.} Besides the invertible building blocks listed in Table \ref{tab:overviewTable} (Appendix \ref{app:table-lipschitz}), several other approaches like in \citet{Karami2019} have been proposed.
Most prominently, autogressive models like IAF \citep{kingma2016improved}, MAF \citep{papamakarios2017masked} or NAF \citep{huang18d} provide invertible models that are not studied in our analysis. 
Furthermore, several newer coupling layers that require numerical inversion have been introduced by \citet{ho2019flowpp, pmlr-v97-jaini19a, neuralSplineFlows}.
In addition to the coupling-based approaches, multiple approaches \citep{chen2018neural, behrmann2019, chen2019residual, grathwohl2019ffjord, song2019mintnet} use numerical inversion schemes, where the interplay of numerical error due to stability and error due to the approximation of the inverse adds another dimension to the study of invertibility.

\vspace{-0.2cm}
\paragraph{Stability-Expressivity Tradeoff.} While Lipschitz constrained INNs like i-ResNets \citep{behrmann2019} allow to overcome many failures we observed with affine coupling blocks \citep{dinh2016density}, this constrains the flexibility of INNs. Hence, there is a tradeoff between stability and expressivity as studied in \citet{jaini2019tails, cornish2020}. While numerical invertibility is necessary for a safe usage of INNs, mixtures of Lipschitz constrained INNs could be used for improved flexibility as suggested in \citet{cornish2020}.

\section{CONCLUSION}
\label{sec:conclusion}
\vspace{-0.3cm}
Invertible Neural Networks (INNs) are an increasingly popular component of the modern deep learning toolkit. However, if analytical invertibility does not carry through to the numerical computation, their underlying assumptions break. When applying INNs, it is important to consider how the inverse is used. If local stability is sufficient, like for memory-efficient gradients, our finite difference penalty is sufficient as an architecture agnostic stabilizer. For global stability requirements e.g. when using INNs as normalizing flows, the focus should be on architectures that enable stable mappings like Residual Flows~\citep{chen2019residual}. Altogether we have shown that studying stability properties of both forward and inverse is a key step towards a complete understanding of INNs.

\section*{Acknowledgements}
Resources used in preparing this research were provided, in part, by the Province of Ontario, the Government of Canada through CIFAR, and companies sponsoring the Vector Institute \url{www.vectorinstitute.ai/#partners}.
Jens Behrmann acknowledges the support by the Deutsche Forschungsgemeinschaft (DFG) within the framework of GRK 2224/1 ``$\pi^3$: Parameter Identification -- Analysis, Algorithms, Applications''.
Paul Vicol was supported by a JP Morgan AI Fellowship.
We also would like to thank Joost van Amersfoort for insightful discussions, and members on the Vector's Ops team (George Mihaiescu, Relu Patrascu, and Xin Li) for their computation/infrastructure support.

\bibliographystyle{plainnat}
\bibliography{aistats2021}


\onecolumn

\appendix

\aistatstitle{Understanding and Mitigating Exploding Inverses \\ in Invertible Neural Networks \\ SUPPLEMENTARY MATERIAL}

\begin{itemize}
    \item In Section~\ref{app:table-lipschitz} we provide a table of Lipschitz bounds for various INN building blocks. 
    \item In Section~\ref{app:bounds} we provide statements on Lipschitz bounds and their corresponding proofs.
    \item In Section~\ref{app:2DcheckerboardExpDetails} we provide details for the 2D checkerboard experiments.
    \item In Section~\ref{app:OODsampleExp} we provide details for the OOD and sample evaluation experiments.
    \item In Section~\ref{app:invertibilityAttack} we expose non-invertibility in the data distribution by optimizing within the dequantization distribution of a datapoint to find regions that are poorly reconstructed by the model.
    \item In Section~\ref{app:2d-toy-regression-details} we provide details for the 2D toy regression experiment.
    \item In section~\ref{app:classificationExpDetails} we provide details for the CIFAR-10 classification experiments.
    \item In Section~\ref{app:flowgan} we discuss an outlook on bi-directional training with FlowGANs.
\end{itemize}

\section{TABLE OF LIPSCHITZ BOUNDS}
\label{app:table-lipschitz}
\begin{table}[h]
 \begin{center}
\small{
 \begin{tabular}{S S S S S}
       {\textbf{Building Block}}  & {\textbf{Forward Operation}} & {\textbf{Lipschitz Forward}} & {\textbf{Lipschitz Inverse}} \\ \midrule
       {\textbf{Additive}} & {$F(x)_{I_1} = x_{I_1}$}  & {$\leq 1 + \Lip(t)$} & {$\leq 1 + \Lip(t)$}\\
       {\textbf{Coupling Block}} & {$F(x)_{I_2} = x_{I_2} + t(x_{I_1})$} & & \\
       {\citep{dinh2014nice}} & {} & & \\\midrule
       {\textbf{Affine}} & {$F(x)_{I_1} = x_{I_1}$}  & {$\leq \max (1, c_g) + M$} & {$\leq \max (1, c_{\frac{1}{g}}) + M^*$} \\
       {\textbf{Coupling Block}} & {$F(x)_{I_2} = x_{I_2} \odot g(s(x_{I_1})) + t(x_{I_1}) $}  & {local for $x \in [a,b]^d$} & {local for $y \in [a^*,b^*]^d$}\\
       {\citep{dinh2016density}} & {$g(\cdot) \neq 0 $}  & {$g(x) \leq c_g$} & {$\frac{1}{g}(y) \leq c_{\frac{1}{g}}$}\\\midrule
       {\textbf{Invertible}} & {$F(x) = x + g(x)$}   & {$ \leq 1+ \Lip(g)$} & {$ \leq \frac{1}{1 - \Lip(g)}$} \\ 
       {\textbf{Residual Layer}} & {$ \Lip(g) < 1$}  & {} & {} \\
       {  \citep{behrmann2019}} & {}  & {} & {} \\
       \midrule
       {\textbf{Neural ODE}} & {$\frac{d x(t)}{dt} = F(x(t), t)$}  & {$\leq e^{\Lip(F) \cdot t}$} & {$\leq e^{\Lip(F) \cdot t}$} \\
       { \citep{chen2018neural}} & {$t \in [0, T]$} & {} & {}\\\midrule
       {\textbf{Diagonal Scaling} } & {$F(x) = D x$}  & {$=\max_i |D_{ii}|$} & {$=\frac{1}{\min_i |D_{ii}|}$ }\\
       { \citep{dinh2014nice}} & {$D$ diagonal} &  & \\
       {\textbf{ActNorm}} & {$D_{ii} \neq 0$} &  & \\
       { \citep{kingma2018glow}} & {} &  & \\\midrule
       {\textbf{Invertible $1\times1$} } & {$F(x) = P L (U + \mathrm{diag}(s)) =: W$}  & & \\
       {\textbf{Convolution}} & {$P$ permutation, $L$ lower-triangular}  & {$\leq \|W\|_2$} & {$\leq \|W^{-1}\|_2$}\\
       { \citep{kingma2018glow}} & {$U$ upper-triangular, $s \in \mathbb{R}^d$} &  & \\ \midrule
       \bottomrule
 \end{tabular}
}
 \caption{
 \textbf{Lipschitz bounds on building blocks of invertible neural networks.} The second column shows the operations of the forward mapping and the last two columns show bounds on the Lipschitz constant of the forward and inverse mapping. $M$ in the row for the forward mapping of an affine block is defined as $M=\max(|a|, |b|) \cdot c_{g'} \cdot \Lip(s) + \Lip(t)$. Furthermore, $M^*$ for the inverse of an affine block is $M^* = \max(|a^*|, |b^*|) \cdot c_{\left(\frac{1}{g}\right)'} \cdot \Lip(s) + c_{\left(\frac{1}{g}\right)'} \cdot \Lip(s) \cdot c_t + c_{\frac{1}{g}} \cdot \Lip(t)$. Note that the bounds of the affine blocks hold only locally. }
 \label{tab:overviewTable}
 \end{center}
 \end{table}

\section{STATEMENTS ON LIPSCHITZ BOUNDS AND PROOFS}
\label{app:bounds}
In this section, we provide our analysis of bi-Lipschitz bounds of common INN architectures. The obtained bounds are summarized in Table \ref{tab:overviewTable}. In general, Lipschitz bounds for deep neural networks often tend to be loose in practice, because a derived bound for a single layer needs to be multiplied by the number of layers of then entire network. Thus these bounds are rarely used quantitatively, however such an analysis can reveal crucial qualitative differences between architecture designs. This is why, we provide the technical analysis of the bi-Lipschitz bounds in the appendix and discuss their qualitative implications in Section \ref{sec:stability} of the main body. 

The bounds for i-ResNets are taken from \citep{behrmann2019}. For Neural ODEs \citep{chen2018neural}, one needs to consider a Lipschitz constant $\Lip(F)$ that holds for all $t \in [0,T]$, i.e.
\begin{align*}
 \|F(t, x_1) - F(t, x_2)\|_2 \leq \Lip(F) \|x_1 - x_2\|_2, \quad \text{for all} \quad t \in [0,T].
\end{align*}
Then, the claimed bound is a standard result, see e.g. \citep[Theorem 2.3]{ascher2008numerical}. Note that the inverse is given by $\frac{dy(t)}{dt} = - F(y(t), t)$, hence the same bound holds. 

In the following, we proceed as follows: First, we state bi-Lipschitz bounds of additive coupling blocks as a lemma and proove them (Lemma \ref{lem:lipBoundAdditive}). Their derivation is generally straightforward but somewhat technical at stages due to the handling of the partition. Second, we perform the same analysis for affine coupling blocks (Lemma \ref{lem:lipBoundAffine}). Third, we use these technical lemmas to proof Theorem \ref{thm:qualDiffCoupling} from the main body of the paper. 

Before deriving the upper bounds, we note that the upper bounds on the bi-Lipschitz constants also provide lower bounds: 
\begin{remark}[Lower bounds via upper bounds]
By reversing, upper bounds on the Lipschitz constant of the inverse mapping yield lower bounds on the Lipschitz constant of the forward and vice versa. This holds due to the following derivation: Let $x, x^* \in \mathbb{R}^d$ and $F^{-1}(z) = x$, $F^{-1}(z^*) = x^*$. By employing the definition of the forward and inverse Lipschitz constants, we have:
\begin{align}
\label{eq:upperLowerBounds}
    \|x - x^*\| = \|F^{-1}(z) - F^{-1}(z^*)\| \leq \Lip(F^{-1}) \|z - z^*\| &= \Lip(F^{-1}) \|F(x) - F(x^*)\| \nonumber\\
     &\leq \Lip(F^{-1}) \Lip(F) \|x - x^*\| \nonumber\\
    \Leftrightarrow \frac{1}{\Lip(F^{-1})} \|x-x^*\| \leq \|F(x) - F(x^*)\| \leq \Lip(F) \|x-x^*\|.
\end{align}
By denoting the upper bounds as $L \geq \Lip(F)$ and $L^* \geq \Lip(F^{-1})$ and using the same reasoning as above, we thus have:
\begin{align*}
    L \geq \Lip(F) \geq \frac{1}{\Lip(F^{-1})} \geq \frac{1}{L^*} \quad \text{ and } \quad L^* \geq \Lip(F^{-1}) \geq \frac{1}{\Lip(F)} \geq \frac{1}{L}.
\end{align*}
Hence our bounds provide a rare case, where not only upper bounds on the Lipschitz constants of neural networks are known, but also lower bounds due to the bi-Lipschitz continuity.
\end{remark}

\begin{remark}[bi-Lipschitz constant]
By considering inequality \ref{eq:upperLowerBounds}, it is further possible to introduce a single constant as $\biLip(F) = \max\{\Lip(F), \Lip(F^{-1})\}$ for which: 
\begin{align*}
    \frac{1}{\biLip(F)} \|x-x^*\| \leq \|F(x) - F(x^*)\| \leq \biLip(F) \|x-x^*\|
\end{align*}
holds. This constant is usually called the bi-Lipschitz constant. We, on the other hand, refer to both constants $\Lip(F), \Lip(F^{-1})$ as bi-Lipschitz constants. We use this slightly more descriptive language, because we are particularly interested in the stability of each mapping direction.
\end{remark}

Now we consider coupling blocks and provide upper Lipschitz bounds:
\begin{lemma}[Lipschitz bounds for additive coupling block]
\label{lem:lipBoundAdditive}
Let $I_1, I_2$ be a disjoint index sets of $\{1, ..., d\}$ and let $I_1, I_2$ be non-empty. Consider an additive coupling block \citep{dinh2014nice} as:
\begin{align*}
    F(x)_{I_1} &= x_{I_1} \\
    F(x)_{I_2} &= x_{I_2} + t(x_{I_1}),
\end{align*}
where $t:\mathbb{R}^{|I_1|} \rightarrow \mathbb{R}^{|I_2|}$ is a Lipschitz continuous and differentiable function. Then, the Lipschitz constant of the forward mapping $F$ and inverse mapping $F^{-1}$ can be upper-bounded by:
\begin{align*}
    \Lip(F) &\leq 1 + \Lip(t) \\
    \Lip(F^{-1}) &\leq 1 + \Lip(t). \\
\end{align*}
\end{lemma}

\begin{proof}
To prove the Lipschitz bounds, we use the identity:
 \begin{align*}
     \Lip(F) = \sup_{x \in \mathbb{R}^d} \|J_F(x)\|_2.
 \end{align*}
 Thus, in order to obtain a bound on the Lipschitz constant, we look into the structure of the Jacobian. Without loss of generality we consider $I_1 = \{1, \ldots, m\}$ and $I_2 = \{m+1, \ldots, d\}$, where $1 < m < d$. The general case, where $I_1, I_2$ are arbitrary disjoint and non-empty index sets, can be recovered by a permutation. A permutation is norm-preserving and thus does not influence the bound on the Lipschitz constant.
 
 The Jacobian of $F$ has a lower-block structure with an identity diagonal, i.e.
 \begin{align*}
     J_F(x) = \begin{pmatrix} I_{m \times m}  & 0_{(m) \times (d-m)} \\
     J_t(x) & I_{(d - m) \times (d - m)}\\
     \end{pmatrix},
 \end{align*}
 where $J_t(x) \in \mathbb{R}^{(d-m) \times m}$ denotes the Jacobian of $t$ at $x$. By using this structure, we can derive the following upper bound:
 \begin{align}
 \label{eq:derivationAdditive}
     \Lip(F)^2 &= \sup_{x \in \mathbb{R}^d} \|J_F(x)\|_2^2 \nonumber \\ 
     &= \sup_{x \in \mathbb{R}^d} \sup_{\|x^*\|_2=1} \|J_F(x) x^*\|_2^2 \nonumber\\
     &= \sup_{x \in \mathbb{R}^d} \sup_{\|x^*\|_2=1} \|(J_F(x) x^*)_{I_1}\|_2^2 + \|(J_F(x) x^*)_{I_2}\|_2^2 \nonumber\\
     &= \sup_{x \in \mathbb{R}^d} \sup_{\|x^*\|_2=1} \|x^*_{I_1}\|_2^2 + \|x^*_{I_2} + J_t(x) x^*_{I_1}\|_2^2 \nonumber\\
     &\leq  \sup_{x \in \mathbb{R}^d} \sup_{\|x^*\|_2=1} \|x^*_{I_1}\|_2^2 + (\|x^*_{I_2}\|_2 + \|J_t(x) x^*_{I_1}\|_2)^2 \\
     & = \sup_{x \in \mathbb{R}^d} \sup_{\|x^*\|_2=1} \|x^*_{I_1}\|_2^2 + \|x^*_{I_2}\|_2^2 + 2 \|x^*_{I_2}\|_2 \|J_t(x) x^*_{I_1}\|_2 + \|J_t(x) x^*_{I_1}\|_2^2 \nonumber\\
     &= \sup_{x \in \mathbb{R}^d} \sup_{\|x^*\|_2=1} \|x^*\|_2^2 + 2 \|x^*_{I_2}\|_2 \|J_t(x) x^*_{I_1}\|_2 + \|J_t(x) x^*_{I_1}\|_2^2 \nonumber\\
     &= \sup_{x \in \mathbb{R}^d} \sup_{\|x^*\|_2=1} 1 + 2 \|x^*_{I_2}\|_2 \|J_t(x) x^*_{I_1}\|_2 + \|J_t(x) x^*_{I_1}\|_2^2 \nonumber\\
     &= \sup_{x \in \mathbb{R}^d} \sup_{\|x^*\|_2=1} 1 + 2 \|J_t(x) x^*_{I_1}\|_2 + \|J_t(x) x^*_{I_1}\|_2^2 \nonumber\\
     &= \sup_{x \in \mathbb{R}^d} \sup_{\|x^*\|_2=1} \left(1 + \|J_t(x) x^*_{I_1}\|_2\right)^2 \nonumber\\
     &= \sup_{x \in \mathbb{R}^d} \left(1 + \|J_t(x)\|_2\right)^2 \nonumber\\
     \Rightarrow \Lip(F) &\leq 1 + \Lip(t).\nonumber
 \end{align}
 Furthermore, the inverse of $F$ can be obtained via the simple algebraic transformation ($y:=F(x)$)
 \begin{align*}
     F^{-1}(y)_{I_1} &= y_{I_1} \\
     F^{-1}(y)_{I_2} &= y_{I_2} - t(y_{I_1}). \\
 \end{align*}
 Since the only difference to the forward mapping is the minus sign, the Lipschitz bound for the inverse is the same as for the forward mapping. 
\end{proof}

\begin{lemma}[Lipschitz bounds for affine coupling block]
\label{lem:lipBoundAffine}
Let $I_1, I_2$ be a disjoint partition of indices $\{1, ..., d\}$ and let $I_1, I_2$ be non-empty. Consider an affine coupling block \citep{dinh2016density} as:
\begin{align*}
    F(x)_{I_1} &= x_{I_1} \\
    F(x)_{I_2} &= x_{I_2} \odot g(s(x_{I_1})) + t(x_{I_1}),
\end{align*}
where $g,s,t:\mathbb{R}^{|I_1|} \rightarrow \mathbb{R}^{|I_2|}$ are Lipschitz continuous, continuously differentiable  and non-constant functions. 
Then, the Lipschitz constant of the forward $F$ can be locally bounded for $x \in [a,b]^d$ as:
\begin{align*}
    \Lip(F) \leq \max (1, c_g) + M,
\end{align*}
where $M=\max(|a|, |b|) \cdot c_{g'} \cdot \Lip(s) + \Lip(t)$. The Lipschitz constant of the inverse $F^{-1}$ can be locally bounded for $y \in [a^*,b^*]^d$ as:
\begin{align*}
    \Lip(F^{-1}) \leq \max (1, c_{\frac{1}{g}}) + M^*,
\end{align*}
where $M^* = \max(|a^*|, |b^*|) \cdot c_{\left(\frac{1}{g}\right)'} \cdot \Lip(s) + c_{\left(\frac{1}{g}\right)'} \cdot \Lip(s) \cdot c_t + c_{\frac{1}{g}} \cdot \Lip(t)$.
\end{lemma}

\begin{proof}
We employ a similar proof strategy as in Lemma \ref{lem:lipBoundAdditive} and consider the Jacobian of the affine coupling layer. As in the proof of Lemma \ref{lem:lipBoundAdditive} we consider $I_1 = \{1, \ldots, m\}$ and $I_2 = \{m+1, \ldots, d\}$, where $1 < m < d$, without loss of generality. Since the structure of the forward and inverse mapping for affine coupling layers has some differences, we split the proof of the Lipschitz bounds into two steps. First, we start with the forward mapping and then reuse several steps for the bounds on the inverse mapping.

\textbf{Derivation for the forward mapping:} \\
The Jacobian of the forward affine block has the structure
\begin{align*}
     J_F(x) = \begin{pmatrix} I_{m \times m} & 0_{(m) \times (d-m)} \\
     D_I(x_{I_2}) D_{g'}(x_{I_1}) J_s(x_{I_1}) + J_t(x_{I_1}) & D_g(s(x_{I_1}))\\
     \end{pmatrix},
\end{align*}
where $D$ are the following diagonal matrices
\begin{align*}
 D_I(x_{I_2}) &= \mathrm{diag}\left((x_{I_2})_1, \ldots, (x_{I_2})_{|I_2|}\right), \\ 
 D_{g'}(x_{I_1}) &= \mathrm{diag} \left(g'(s(x_{I_2})_1, \ldots, g'(s(x_{I_2})_{|I_2|}) \right), \\
 D_{g}(s(x_{I_1})) &= \mathrm{diag} \left(g(s(x_{I_2})_1, \ldots, g(s(x_{I_2})_{|I_2|}) \right),
\end{align*}
where $D_I(x_{I_2}), D_{g'}(x_{I_1}) \in \mathbb{R}^{(d-m) \times m}$ and $D_{g}(s(x_{I_1})) \in \mathbb{R}^{(d - m) \times (d - m)} $. To simplify notation, we introduce
\begin{align*}
    M(x) = D_I(x_{I_2}) D_{g'}(x_{I_1}) J_s(x_{I_1}) + J_t(x_{I_1}).
\end{align*}
By using an analogous derivation as in the proof of Lemma \eqref{lem:lipBoundAffine} (up to the inequality sign), we get:
\begin{align*}
 \Lip(F)^2 &\leq \sup_{x \in [a,b]^d} \sup_{\|x^*\|_2=1} \|x^*_{I_1}\|_2^2 + \left(\|D_g(s(x_{I_1})) x^*_{I_2}\|_2 + \|M(x) x^*_{I_1}\|_2\right)^2 \\
 &= \sup_{x \in [a,b]^d}  \max_{i \in [|I_1|]} (1, D_g(s(x_{I_1})_i))^2 + 2 \max_{i \in [|I_1|]} ( D_g(s(x_{I_1})_i))\|M(x)\|_2 + \|M(x)\|_2^2 \\
 &\leq \sup_{x \in [a,b]^d}  \max_{i \in [|I_1|]} (1, D_g(s(x_{I_1})_i))^2 + 2 \max_{i \in [|I_1|]} (1, D_g(s(x_{I_1})_i))\|M(x)\|_2 + \|M(x)\|_2^2 \\
 &= \sup_{x \in [a,b]^d}  \left( \max_{i \in [|I_1|]} (1, D_g(s(x_{I_1})_i)) + \|M(x)\|_2\right)^2 \\
\iff \Lip(F) &\leq \max_{i \in [|I_1|]} (1, D_g(s(x_{I_1})_i)) + \sup_{x \in [a,b]^d}\|M(x)\|_2.
\end{align*}
Next, we will look into the structure of $M(x)$ to derive a more precise bound. Since inputs $x$ are assumed to be bounded as $x \in [a, b]^d$, it holds: 
\begin{align*}
 \|D_I(x_{I_2})\|_2 \leq \max(|a|, |b|).
\end{align*}
Since we assumed that $g$ is continuously differentiable, both $g$ and $g'$ will be bounded over closed intervals like $[a,b]^d$. We will denote these bounds as $c_g$ and $c_{g'}$. Thus, for $x\in [a,b]^d$ it holds
\begin{align*}
 \|D_{g'}(x_{I_1})\|_2 \leq c_{g'}.
\end{align*}
In a similar manner as in Lemma \ref{lem:lipBoundAdditive}, the spectral norm of the Jacobian of the scale-function $s$ and translation-function $t$ can be bounded by their Lipschitz constant, i.e.
\begin{align*}
 \|J_s(x_{I_1})\|_2 &\leq \Lip(s) \\
 \|J_t(x_{I_1})\|_2 &\leq \Lip(t) .
\end{align*}
By using the above bounds, we obtain
\begin{align*}
 \sup_{x\in [a,b]^d}\|M(x)\|_2^2 \leq \max(|a|, |b|) \cdot c_{g'} \cdot \Lip(s) + \Lip(t),
\end{align*}
which results in the local Lipschitz bounds for $x\in[a,b]^d$
\begin{align*}
 \Lip(F) &\leq \max (1, c_g) + \max(|a|, |b|) \cdot c_{g'} \cdot \Lip(s) + \Lip(t).
\end{align*}

\textbf{Derivation for the inverse mapping:} \\
For the affine coupling block, the inverse is defined as: 
 \begin{align*}
     F^{-1}(y)_{I_1} &= y_{I_1} \\
     F^{-1}(y)_{I_2} &= (y_{I_2} - t(x_{I_1})) \oslash g(s(y_{I_1})), 
 \end{align*}
 where $g(\cdot) \neq 0$ for all $X_{I_2}$, $I_1, I_2$ as before and $\oslash$ denotes elementwise division. The Jacobian for this operation has the structure:
 \begin{align*}
     J_{F^{-1}}(x) = \begin{pmatrix} I & 0 \\
     M^*(y) & D_{\frac{1}{g}}(s(x_{I_1}))\\
     \end{pmatrix},
 \end{align*}
 where $D_{\frac{1}{g}}(s(x_{I_1}))$ denotes a diagonal matrix, as before. Furthermore, $M^*$ is defined as: 
 \begin{align*}
     M^*(y) = D_I(y_{I_2}) D_{\left(\frac{1}{g} \right)'}(s(y_{I_1})) J_s(y_{I_1}) - D_{\left(\frac{1}{g} \right)'}(s(y_{I_1})) J_s(y_{I_1}) D_I(t(y_{I_1})) - D_{\frac{1}{g}}(s(y_{I_1})) J_t (y_{I_1}),
 \end{align*}
 where $D_{\left(\frac{1}{g} \right)'}(s(x_{I_1}))$ also denotes a diagonal matrix. Using analogous arguments as for the forward mapping, we obtain the bound:
 \begin{align*}
     \Lip(F^{-1}) \leq \max_{i \in [|I_1|]} (1, D_{\frac{1}{g}}(s(x_{I_1})_i)) + \sup_{y\in [a^*, b^*]^d}\|M^*(y)\|_2. 
 \end{align*}
 Hence, we need to further bound the spectral norm of $M^*$. Since we assumed that $g$ is continuously differentiable, both $\frac{1}{g}$ and $\left(\frac{1}{g}\right)'$ will be bounded over closed intervals like $[a^*,b^*]^d$. Furthermore, translation $t$ is assumed to be globally continuous and thus also bounded over closed intervals. We will denote these upper bounds by $c_{\frac{1}{g}}$, $c_{\left(\frac{1}{g}\right)'}$ and $c_t$.
 Then we obtain the bound:
 \begin{align*}
     \sup_{y \in [a^*, b^*]^d}\|M^*(y)\|_2^2 \leq \max(|a^*|, |b^*|) \cdot c_{\left(\frac{1}{g}\right)'} \cdot \Lip(s) + c_{\left(\frac{1}{g}\right)'} \cdot \Lip(s) \cdot c_t + c_{\frac{1}{g}} \cdot \Lip(t).
 \end{align*}
 Hence, we can bound the Lipschitz constant  of the inverse of an affine block over the interval $[a^*, b^*]^d$ as:
 \begin{align*}
     \Lip(F^{-1}) \leq  \max_{i \in [|I_1|]} (1, c_{\frac{1}{g}}) +  \max(|a^*|, |b^*|) \cdot c_{\left(\frac{1}{g}\right)'} \cdot \Lip(s) + c_{\left(\frac{1}{g}\right)'} \cdot \Lip(s) \cdot c_t + c_{\frac{1}{g}} \cdot \Lip(t).
 \end{align*}
\end{proof}

\subsection{Proof of Theorem \ref{thm:qualDiffCoupling}}
\label{app:proofQualStatement}
\begin{proof} 
Proof of statement (i): \\
The larger bi-Lipschitz bounds of the affine models compared to the additive models follows directly from Lemmas \ref{lem:lipBoundAdditive} and \ref{lem:lipBoundAffine}, as the affine bounds have only additional non-zero components in their bounds.

Proof of statement (ii): \\
The global Lipschitz bound for additive models is given in Lemma \ref{lem:lipBoundAdditive}. In Lemma \ref{lem:lipBoundAffine} we also provide local bounds for $x \in [a,b]^d$ for affine models. What remains to be show, is that there are no bi-Lipschitz bounds that hold globally for $x \in \mathbb{R}^d$.

For this, consider a simplified affine model as $F(x_1, x_2) = x_1  f(x_2)$ with $\frac{dF}{dx_2} = x_1 \frac{df}{dx_2}$. This derivative is unbounded if $x_1$ is allowed to be arbitrarily large, hence there is no global bound. 

The same argument caries over to the full affine model, since both forward and inverse Jacobian involve the terms $D_I(x_{I_2})$ and $D_I(y_{I_2})$, respectively (see proofs of Lemma \ref{lem:lipBoundAdditive} and \ref{lem:lipBoundAffine}). When $x$ for the forward and $y$ for the inverse are not assumed to be bounded, the Jacobian can have unbounded Frobenius norm. Not that this only holds if $g$ and $s$ are not constant functions, which is why we need to assume this property to hold (otherwise the affine model would collapse to an additive model). The unbounded Jacobian in turn induces a unbounded spectral norm due to the equivalence of norms in finite dimensions and thus no Lipschitz bound can be obtained.  
\end{proof}

\section{DETAILS FOR 2D DENSITY MODELING EXPERIMENTS}
\label{app:2DcheckerboardExpDetails}
Here we provide experimental details for the 2D checkerboard experiments from Section~\ref{sec:oodInstabilityExp}.
The samples are shown in Figure \ref{fig:samplesCheckerboard}, which shows that the data lies within $x_1 \in [-4,4]$ and $x_2 \in [-4, 4]$ and exhibit jumps at the border of the checkerboard.

\begin{figure}[h]
\centering
	\includegraphics[width=0.32\linewidth]{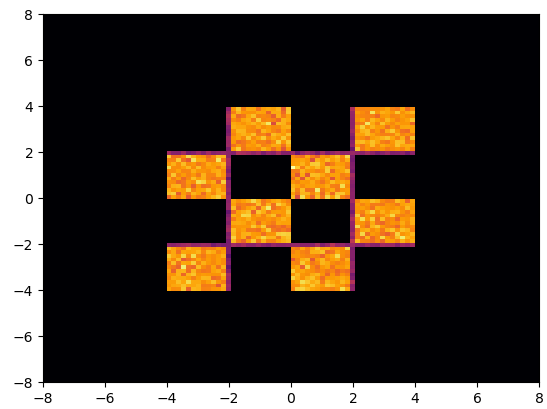}
\caption{\textbf{Samples from 2D checkerboard.}
}
\label{fig:samplesCheckerboard}
\end{figure}

We train the following three large models using the residual flows repository and the corresponding hyperparameter settings\footnote{from \url{https://github.com/rtqichen/residual-flows/blob/master/train_toy.py}}. For completeness, we provide the hyperparameters in Table \ref{table:train-checkerboard} below.

 \begin{table}[h]
 \centering
 \begin{tabular}{cc}
 \hline
  \textbf{Hyperparameter}       &  \textbf{Value}\\
  \hline
 Batch Size      & 500   \\
 Learning Rate   & 1e-3  \\
 Weight Decay    & 1e-5     \\
 Optimizer       & Adam \\
Hidden Dim       & 128-128-128-128 \\
 Num Blocks      & 100 \\
 Activation      & swish \\
 ActNorm         & False \\
 \hline
 \end{tabular}
 \caption{\textbf{Hyperparameters for training 2D models on checkerboard data.}}
 \label{table:train-checkerboard}
 \end{table}
To consider the effect of different architecture settings on stability we train three INN variants:
\begin{enumerate}
    \item Affine coupling model with standard sigmoid scaling for the elementwise function $g$ from \eqref{eq:affineCoupling}, which results in a scaling in $(0,1)$.
    \item Modified affine coupling model with a scaling in $(0.5, 1)$ by a squashed sigmoid.
    \item Residual flow \citep{chen2019residual} with a coefficient of $0.8$ for spectral normalization to satisfy the contraction requirement from i-ResNets \citep{behrmann2019}.
\end{enumerate}
In addition to the reconstruction error in Figure \ref{fig:checkerboardReconError} (main body of the paper), we visualize the learned density function for the models above in Figure \ref{fig:2dtoyDensities}. Most importantly, the instability a the affine model is clearly visible in the NaN density values outside the data domain. The more stable affine model with the modified scaling does not exhibit this failure, but still appears to learn a density with large slopes. The residual flow on the other hand learns a more stable distribution. Lastly, we note that the trained models on this 2D data were deliberately large to emphasize failures even in a low dimensional setting.

\begin{figure}
\centering
\begin{subfigure}[t]{0.32\linewidth}
	\centering
	\includegraphics[width=\linewidth]{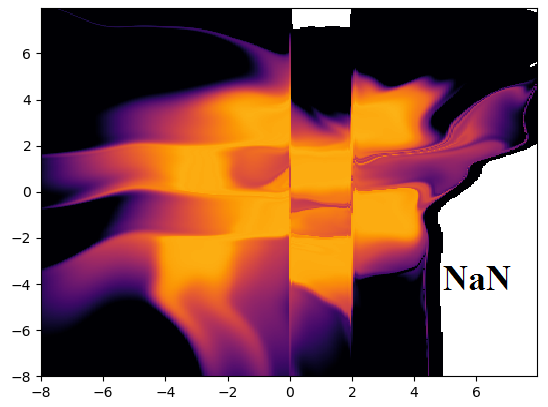}
\end{subfigure}
\begin{subfigure}[t]{0.32\linewidth}
	\centering
	\includegraphics[width=\linewidth]{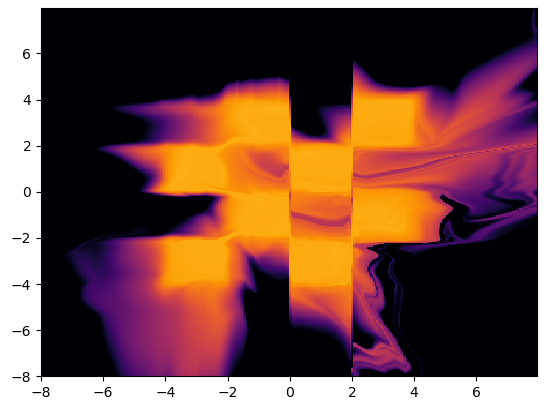}
\end{subfigure}
\begin{subfigure}[t]{0.32\linewidth}
	\centering
	\includegraphics[width=\linewidth]{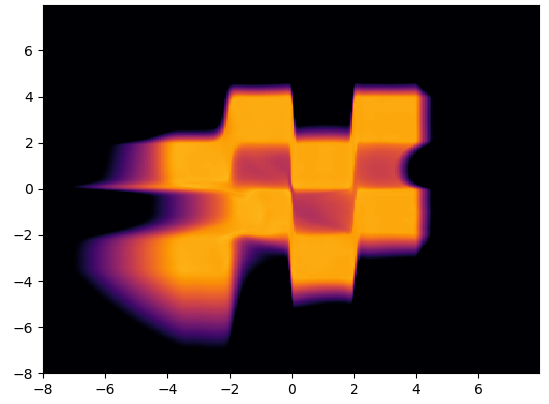}
\end{subfigure}
\caption{\textbf{Learned density on 2D checkerboard data:} \textbf{Left:} standard affine model with sigmoid scaling; \textbf{Middle}: more stable affine model with scaling in [0.5, 1]; \textbf{Right}: Residual Flow \citep{chen2019residual}. 
}
\label{fig:2dtoyDensities}
\end{figure}

\section{DETAILS FOR OOD AND SAMPLE EVALUATION EXPERIMENTS}
\label{app:OODsampleExp}

The examples from each OOD dataset were normalized such that pixel values fell into the same range as each model was originally trained on: $[-0.5, 0.5]$ for Glow, and $[0, 1]$ for Residual Flows.
The OOD datasets were adopted from previous studies including those by \cite{hendrycks2016baseline,liang2017enhancing}. Extended results are shown in Table \ref{table:ood-reconstructionsExtended}.
The pre-trained models we used for OOD evaluation were from the official Github repositories corresponding to the respective papers.
The pre-trained Glow model used $1 \times 1$ convolutions.

\begin{table}[H]
\setlength{\tabcolsep}{3pt}
\centering
\begin{tabular}{ccccc}
\toprule
\textbf{CIFAR-10} & \textbf{SVHN} & \textbf{Uniform} & \textbf{Rademacher} & \textbf{Places}  \\
\includegraphics[width=0.18\textwidth]{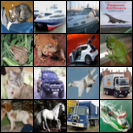} &
\includegraphics[width=0.18\textwidth]{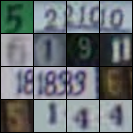} &
\includegraphics[width=0.18\textwidth]{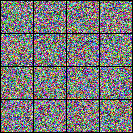} & 
\includegraphics[width=0.18\textwidth]{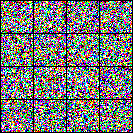} &
\includegraphics[width=0.18\textwidth]{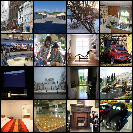} \\
\midrule
\includegraphics[width=0.18\textwidth]{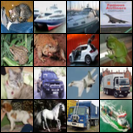} &
\includegraphics[width=0.18\textwidth]{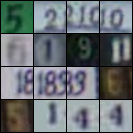} &
\includegraphics[width=0.18\textwidth]{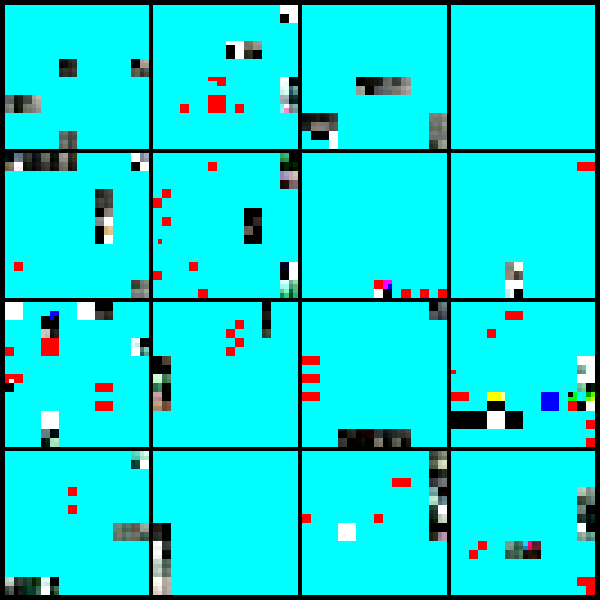} &
\includegraphics[width=0.18\textwidth]{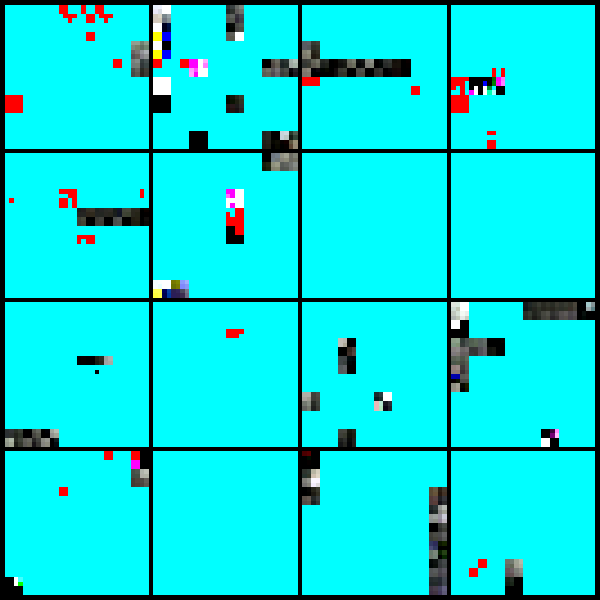} &
\includegraphics[width=0.18\textwidth]{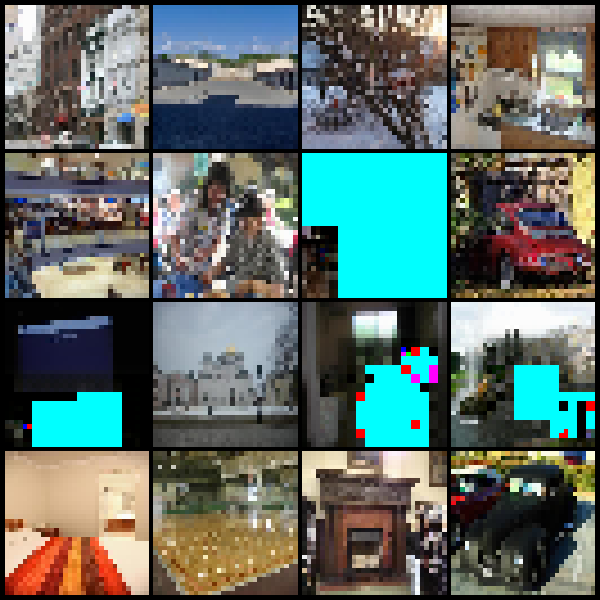} \\
\bottomrule
\end{tabular}
\caption{\textbf{Reconstructions of OOD data, using a CIFAR-10 pre-trained Glow model.} Broken reconstructions that contain \texttt{inf} values are {\color{cyan} highlighted in cyan}.}
\label{table:ood-reconstructionsExtended}
\end{table}

\paragraph{Details for Training Glow}
This section details the Glow models trained with the Flow/MLE objective reported in Section~\ref{sec:flowInDistr}.
These models have slightly different hyperparameters than the pre-trained Glow models.
The controlled hyperparameters are listed in Table~\ref{table:train-glow}.

 \begin{table}[h]
 \centering
 \begin{tabular}{cc}
 \hline
  \textbf{Hyperparameter}       &  \textbf{Value}\\
  \hline
 Batch Size      & 64   \\
 Learning Rate   & 5e-4  \\
 Weight Decay    & 5e-5     \\
 Optimizer       & Adamax \\
 Flow Permutation & Reverse \\
 \# of Layers & 3 \\
 \# of Blocks per Layer & 32 \\
 \# of Conv Layers per Block & 3 \\
 \# of Channels per Conv & 512 \\
 
 \hline
 \end{tabular}
 \caption{\textbf{Hyperparameters for trained Glow.}}
 \label{table:train-glow}
 \end{table}
The LR is warmed up linearly for 5 epochs. 
The training data is augmented with 10\% translation, and random horizontal flips. 
The additive and affine models used the coupling blocks described in Eq.~\ref{eq:additiveCoupling} and Eq.~\ref{eq:affineCoupling}.

\paragraph{Details for Evaluating Glow}
Stability statistics reported included reconstruction errors and condition numbers (Figure~\ref{fig:flow-stats}).  Reconstruction error reported is the pixel-wise $\ell_2^2$ distance measured in the  $[-0.5, 0.5]$ range. Input and reconstruction pairs are visualized in Figure~\ref{fig:flow-recons}. Condition numbers are computed numerically as follows: 1. gradient w.r.t. each dimension of the network output is computed sequentially to form the Jacobian, 2. SVD of the Jacobian is computed using ``numpy.linalg.svd'', 3. the reported condition number is the ratio of the largest to smallest singular value.

\begin{figure}[t]
\centering
\includegraphics[width=\linewidth]{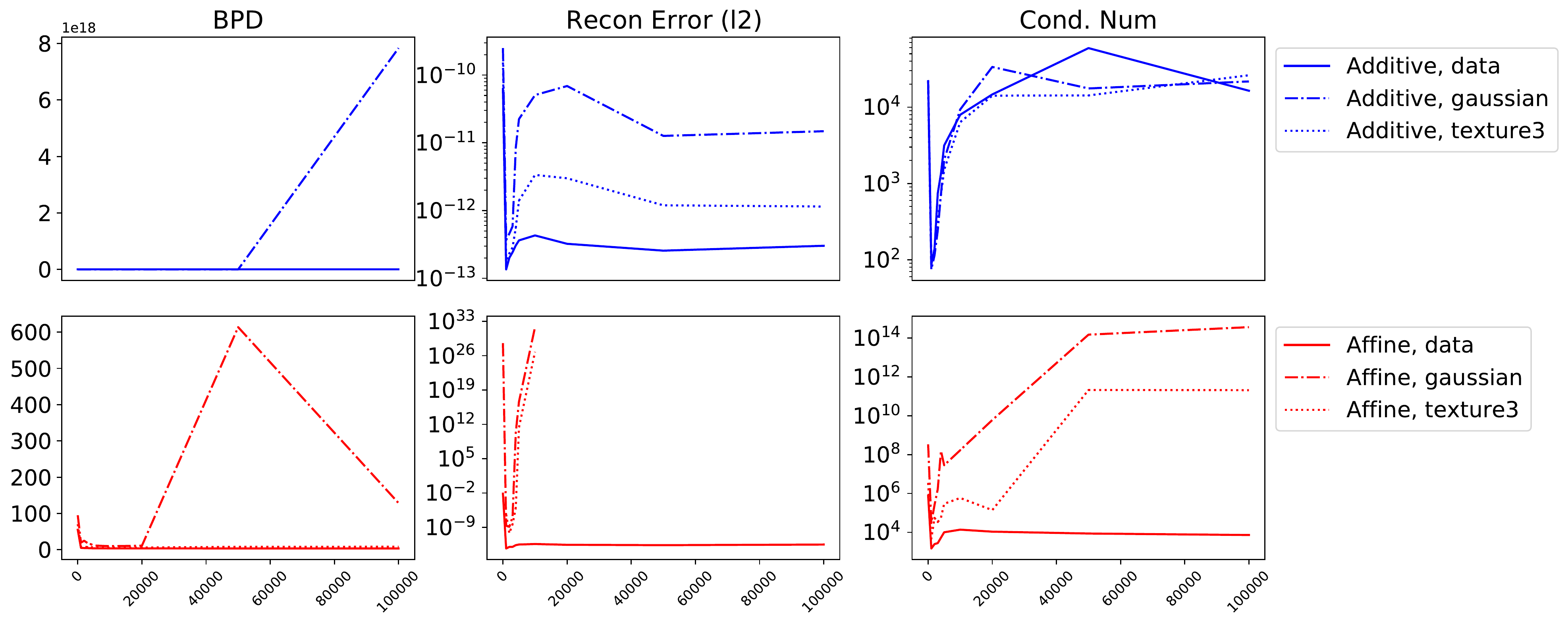}
\caption{\textbf{Stability statistics of Flow model} over gradient steps on the x-axis.  
}
\label{fig:flow-stats}
\end{figure}

\begin{figure}[t]
\centering
\includegraphics[width=\linewidth]{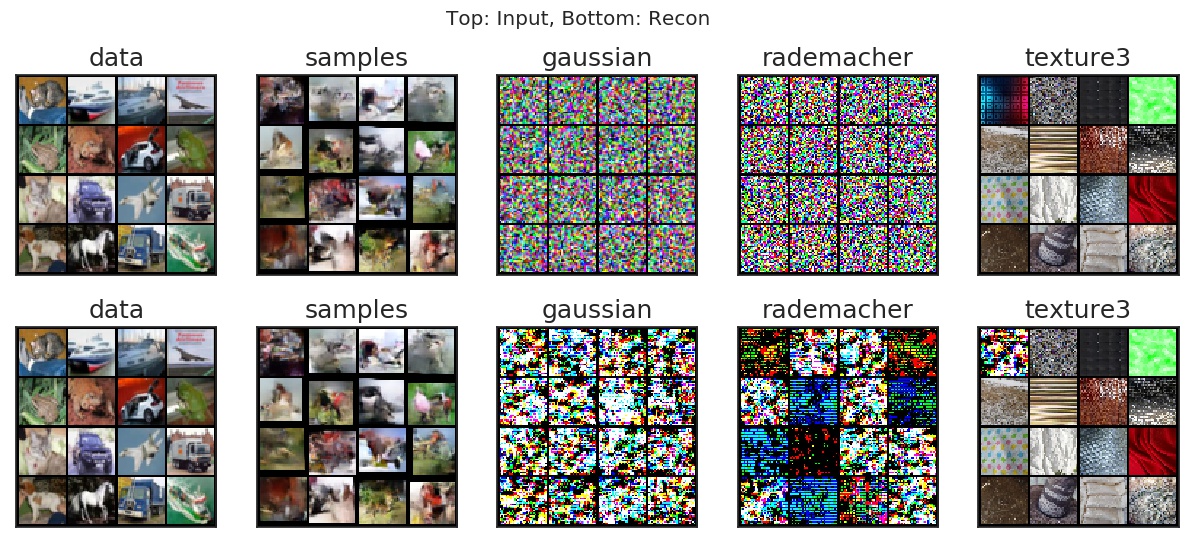}
\caption{\textbf{Reconstructions using Affine model.} Top row is the input, and bottom row is the reconstruction. In the case of perfect reconstruction, the two rows should look identical.  Notice that the model is not only non-invertible for the white noise, the top-left image of ``texture 3'' also fails to be reconstructed. Some of the images in ``samples'' also look slightly different. 
}
\label{fig:flow-recons}
\end{figure}

\section{Non-Invertibility in Data Distribution}
\label{app:invertibilityAttack}

In this section we expose non-invertibility in the data distribution by optimizing within the dequantization distribution of a datapoint to find regions that are poorly reconstructed by the model.
Note that the inputs found this way are valid samples from the training distribution. We performed this analysis with three NFs trained on CelebA64 with commonly-used 5-bit uniform dequantization: 1) affine Glow with standard sigmoid scaling; 2) affine Glow model with modified scaling in $(0.5, 1)$; and 3) a Residual Flow.
Starting from an initial training datapoint $x$, we optimized in input-space to find a perturbed example $x'$ that induces large reconstruction error using Projected Gradient Descent \citep{madry2017towards}:
\begin{equation} \label{eq:attack}
    \argmax_{||x' - x||_\infty < \epsilon} ||x' - F^{-1}(F(x')) ||_2,
\end{equation}
where $\epsilon$ is determined by the amount of uniform dequantization, see Appendix~\ref{app:invertibilityAttackdetails} for details.
As shown in Figure~\ref{fig:attackAffineModel}, this attack reveals the instability of affine models and underlines the stability of Residual Flows \citep{chen2019residual}.
To conclude, we recommend Residual Flows for a principled application of NFs.

\begin{figure}[t]
\centering
\begin{subfigure}[c]{0.4\linewidth}
	\centering
	\begin{tikzpicture}
      \node (img1)  {\includegraphics[trim={0cm 0.4cm 1.6cm 1.45cm},clip,width=0.85\linewidth]{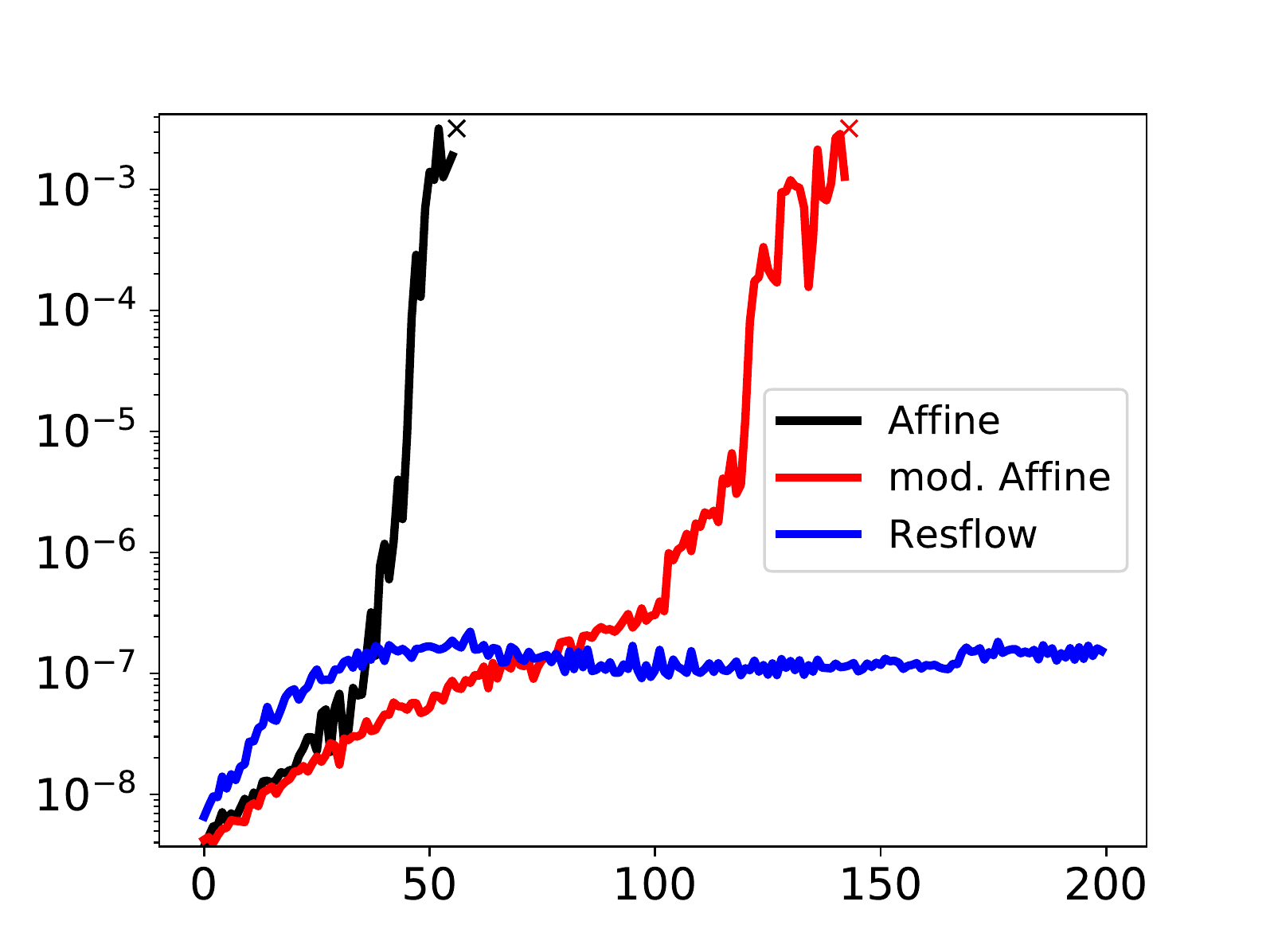}};
      \node[below=of img1, node distance=0cm,xshift=0.3cm, yshift=1.2cm,font={\footnotesize \color{black}}] {Attack Iteration};
      \node[left=of img1, node distance=0cm, rotate=90, anchor=center,xshift=0.1cm,yshift=-1cm,font={\footnotesize \color{black}}] {$\ell_2$ Recons. Error};
    \end{tikzpicture}
\end{subfigure}
\begin{subfigure}[c]{0.25\linewidth}
	\centering
	\includegraphics[width=\linewidth]{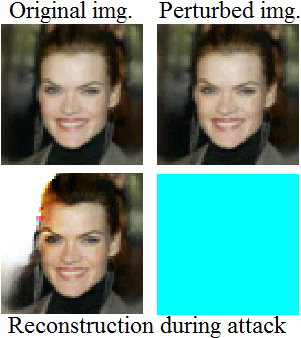}
\end{subfigure}
    \centering
    \caption{\textbf{Non-invertibility within the data-distribution of CelebA64 revealed by an invertibility attack.} \textbf{Left:} $\ell_2$-reconstruction error obtained by the attack. The reconstructions of both affine models explode to \texttt{NaN} (indicated by $\times$), while the Residual Flow remains stable. \textbf{Right:} Despite no visual differences between the original and perturbed image, reconstruction fails for the affine model.}
    \label{fig:attackAffineModel}
\end{figure}

\begin{figure}
    \centering
    \includegraphics[width=0.6\textwidth]{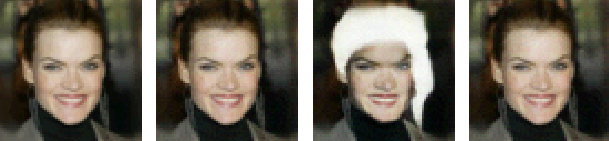}
    \caption{\textbf{Checking invertibility of Residual Flow on CelebA64.} From left to right, the images show: original image, perturbed image within dequantization range, reconstruction with 5 fixed point iterations, reconstruction with converged fixed point iterations. While the attack is able to find regions in the data distribution, which exhibit reconstruction error, the error is only due to a limited number of inversion steps.}
    \label{fig:attacksResflow}
\end{figure}

\subsection{Experimental Details of Invertibility Attack}
\label{app:invertibilityAttackdetails}
In order to probe the invertibility of trained flow models over the entire data distribution, we trained three INN models on CelebA64 (with 5-bit dequantization) using the residual flows repository and the corresponding hyperparameter settings\footnote{from \url{https://github.com/rtqichen/residual-flows/blob/master/train_img.py}}. For completeness, we provide the hyperparameters in the Table \ref{table:train-celebA}.
 \begin{table}[h]
 \centering
 \begin{tabular}{cc}
 \hline
  \textbf{Hyperparameter}       &  \textbf{Value}\\
  \hline
 Batch Size      & 64   \\
 Learning Rate   & 1e-3  \\
 Weight Decay    & 0    \\
 Optimizer       & Adam (ResFlow), Adamax (Affine)\\
 Warmup iter     & 1000 \\
Inner Dim       & 512 \\
 Num Blocks      & 16-16-16-16 \\
 ActNorm         & True \\
 Activation      & Swish (ResFlow), ELU (Affine) \\
 Squeeze First   & True \\
 Factor Out      & False \\
 FC end          & True \\
 Num Exact Terms & 8 \\
 Num Trace Samples & 1 \\
 Padding Distribution & uniform \\
 \hline
 \end{tabular}
 \caption{\textbf{Hyperparameters for training affine and ResFlow models on CelebA64.}}
 \label{table:train-celebA}
 \end{table}

For evaluation we used the checkpoint obtained via tracking the lowest test bits-per-dimension. As a summary, we obtained the following three models:
\begin{enumerate}
    \item An affine coupling model with standard sigmoid scaling for the elementwise function $g$ from Eq.~\ref{eq:affineCoupling}, which results in a scaling in $(0,1)$.
    \item A modified affine coupling model with a scaling in $(0.5, 1)$ by a squashed sigmoid.
    \item A Residual Flow \citep{chen2019residual} with a coefficient of $0.98$ for spectral normalization to satisfy the contraction requirement from i-ResNets \citep{behrmann2019}.
\end{enumerate}

To explore the data distribution, we first need to consider the pre-processing step that is used to turn the quantized digital images into a continuous probability density: dequantization, see e.g. \citep{ho2019flowpp} for a recent discussion on dequantization in normalizing flows. Here we used the common uniform dequantization with 5-bit CelebA64 data, which creates noisy samples:
\begin{align*}
    \hat{x} = \frac{x + \delta}{2^5}, \text{ where} \quad x \in \{0, \ldots, 2^5\}^d \quad \text{and} \quad  \delta \sim \text{uniform}\{0,1\},
\end{align*}
where $d=64\cdot 64 \cdot 3$. Thus, around each training sample $x$ there is a uniform distribution which has equal likelihood under the data model. If we explore this distribution with the invertibility attack below, we can guarantee by design that we stay within the data distribution. This is fundamentally different to the classical setting of adversarial attacks \citep{szegedy2013intriguing}, where it is unknown if the crafted inputs stay within the data distribution.

These models were then probed for invertibility with the data distribution using the following attack objective: 
\begin{equation} 
    \argmax_{||x' - \hat{x}||_\infty < \epsilon} L(x') := ||x' - F^{-1}(F(x')) ||_2,
\end{equation}
where $\epsilon =\frac{0.5}{2^5}$, $\hat{x} = \frac{x + 0.5}{2^5}$ and $x$ was selected from the training set of CelebA64. By using this $\ell_\infty$ constraint, we thus make sure that we explore only within the data distribution. The optimization was performed using projections to fulfill the constraint and by signed gradient updates as:
\begin{align*}
    x'_{k+1} = x'_k + \alpha \; \text{sign}(\nabla_{x'} L(x'_{k})),
\end{align*}
with $\alpha=5e^{-4}$ and $k=\{1, \ldots, 200\}$. 

In addition to the visual results in Figure \ref{fig:attackAffineModel}, we present visual results from the Residual Flow model in Figure \ref{fig:attacksResflow}. For this model, we have to remark one important aspect: due to memory constraint on the used hardware (NVIDIA GeForce GTX 1080, 12 GB memory), we could only use 5 fixed-point iterations for the inverse of the i-ResNet \citep{behrmann2019}, since computing backprop through this iteration is memory intensive. This is why, we observe visible reconstruction errors in Figure \ref{fig:attacksResflow}, which are due to a small number of iterations as the reconstruction on the right with more iterations shows. However, we note that the comparison is not entirely fair since the gradient-based attack has only access to the 5-iteration inverse, while for affine models it has access to the full analytical inverse. Thus, future work should address this issue e.g. by considering non-gradient based attacks. 

Furthermore, we note that the bi-Lipschitz bounds for i-ResNets from \citep{behrmann2019} (see also overview Table \ref{tab:overviewTable}) guarantee stability, which is why using more sophisticated attacks for the Residual Flow model should not result in invertibility failures as in affine models.

\section{EXPERIMENTAL DETAILS FOR 2D TOY REGRESSION}
\label{app:2d-toy-regression-details}

In this section we provide experimental details for the 2D toy regression experiment in Section~\ref{sec:classification}.

\paragraph{Data.}
To generate the toy data, we sampled $10,000$ input-target pairs $((x_1, y_1), (x_2, y_2))$ distributed according to the following multivariate normal distributions:
\begin{align}
(x_1, y_1) &\sim \mathcal{N}\left( \mathbf{0}, \begin{bmatrix} 1 & 0 \\ 0 & \epsilon
\end{bmatrix} \right) \\
(x_2, y_2) &\sim \mathcal{N}\left( \mathbf{0}, \begin{bmatrix} 1 & 1 \\ 1 & 1 \end{bmatrix} \right).
\end{align}
We used $\epsilon=$1e-24 so that we are essentially mapping $\mathbf{x}$ onto a 1D subspace.

\paragraph{Model.}
We used an affine Glow model with reverse permutations and ActNorm \citep{kingma2018glow}, where each block was a multi-layered perceptron (MLP) with two hidden layers of 128 units each and ReLU activations.
We trained the model in full-batch mode using Adam with fixed learning rate 1e-4 for 40,000 iterations.

\paragraph{Additional Results.}
Figure~\ref{fig:toy-training-plots} compares the mean squared error (MSE) loss, numerical reconstruction error, and condition number of the unregularized and regularized Glow models trained on this toy task.
The regularized model adds the normalizing flow objective with a small coefficient of 1e-8.
Here we see that the regularized model still achieves low MSE, while remaining stable with reconstruction error more than four orders of magnitude smaller than the unregularized model, and a condition number six orders of magnitude smaller.

\begin{figure*}[htbp]
\centering
\begin{subfigure}[t]{0.31\linewidth}
	\centering
	\includegraphics[width=\linewidth]{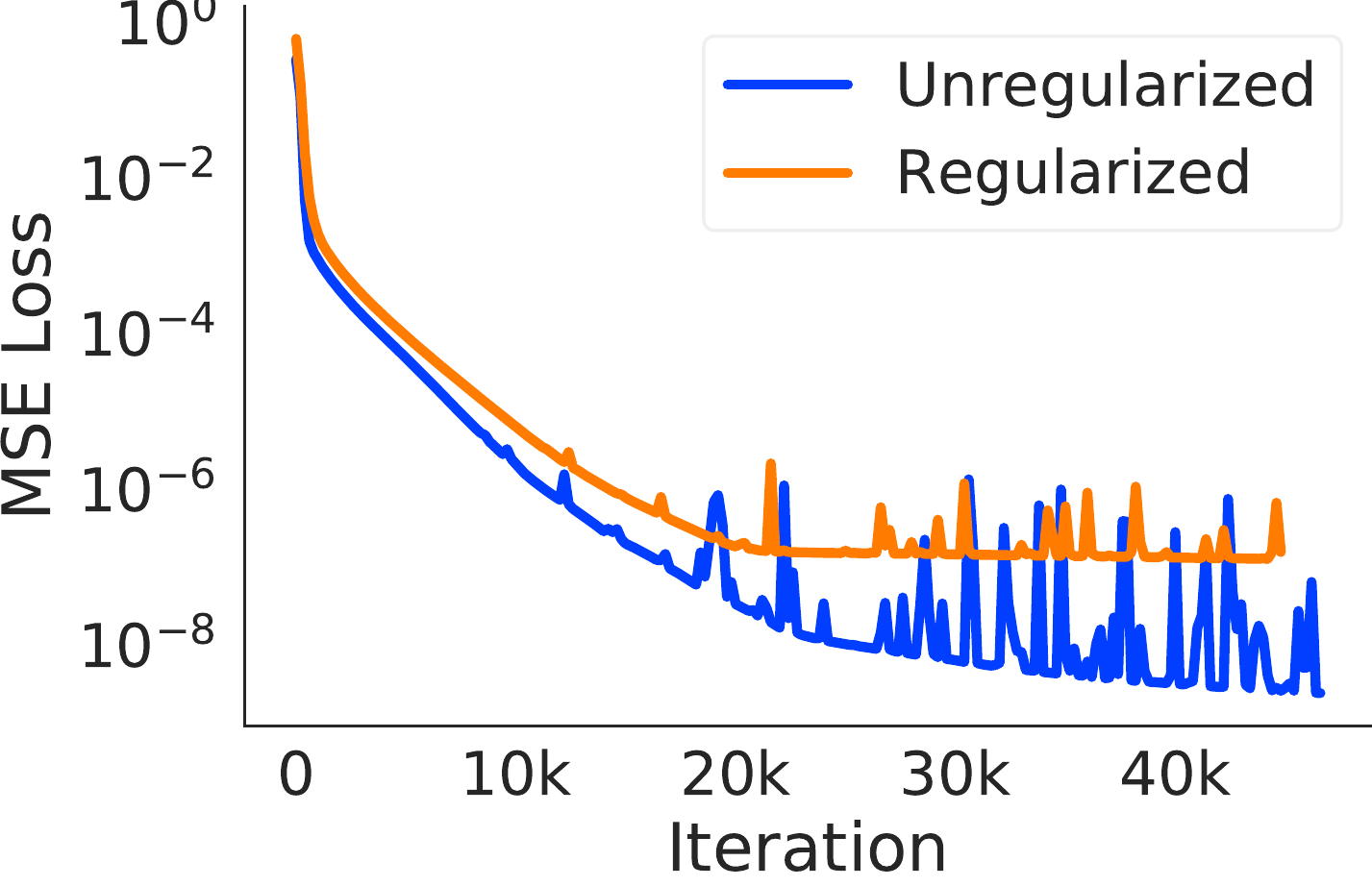}
\end{subfigure}
\hfill
\begin{subfigure}[t]{0.31\linewidth}
	\centering
	\includegraphics[width=\linewidth]{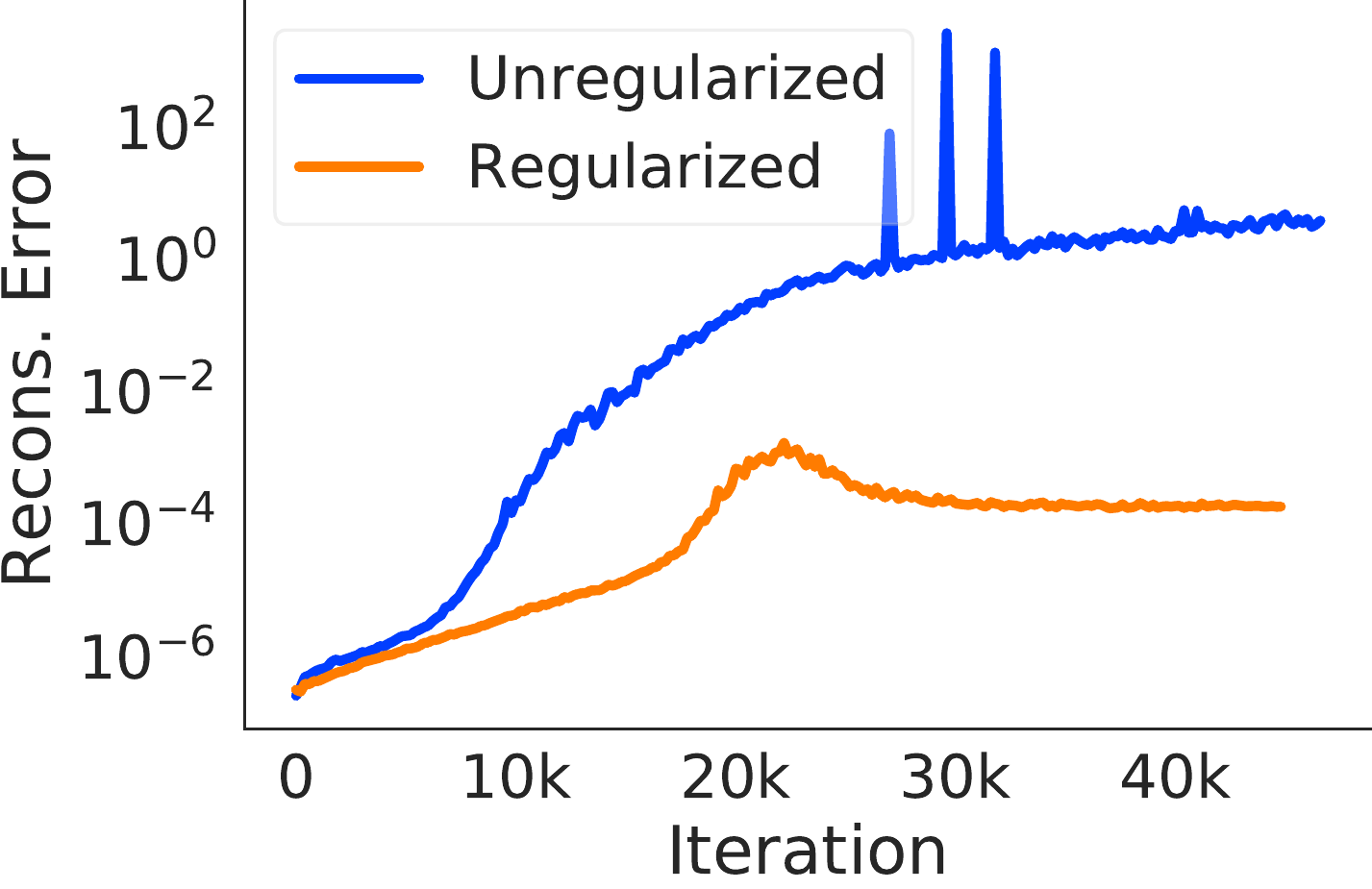}
\end{subfigure}
\hfill
\begin{subfigure}[t]{0.31\linewidth}
	\centering
	\includegraphics[width=\linewidth]{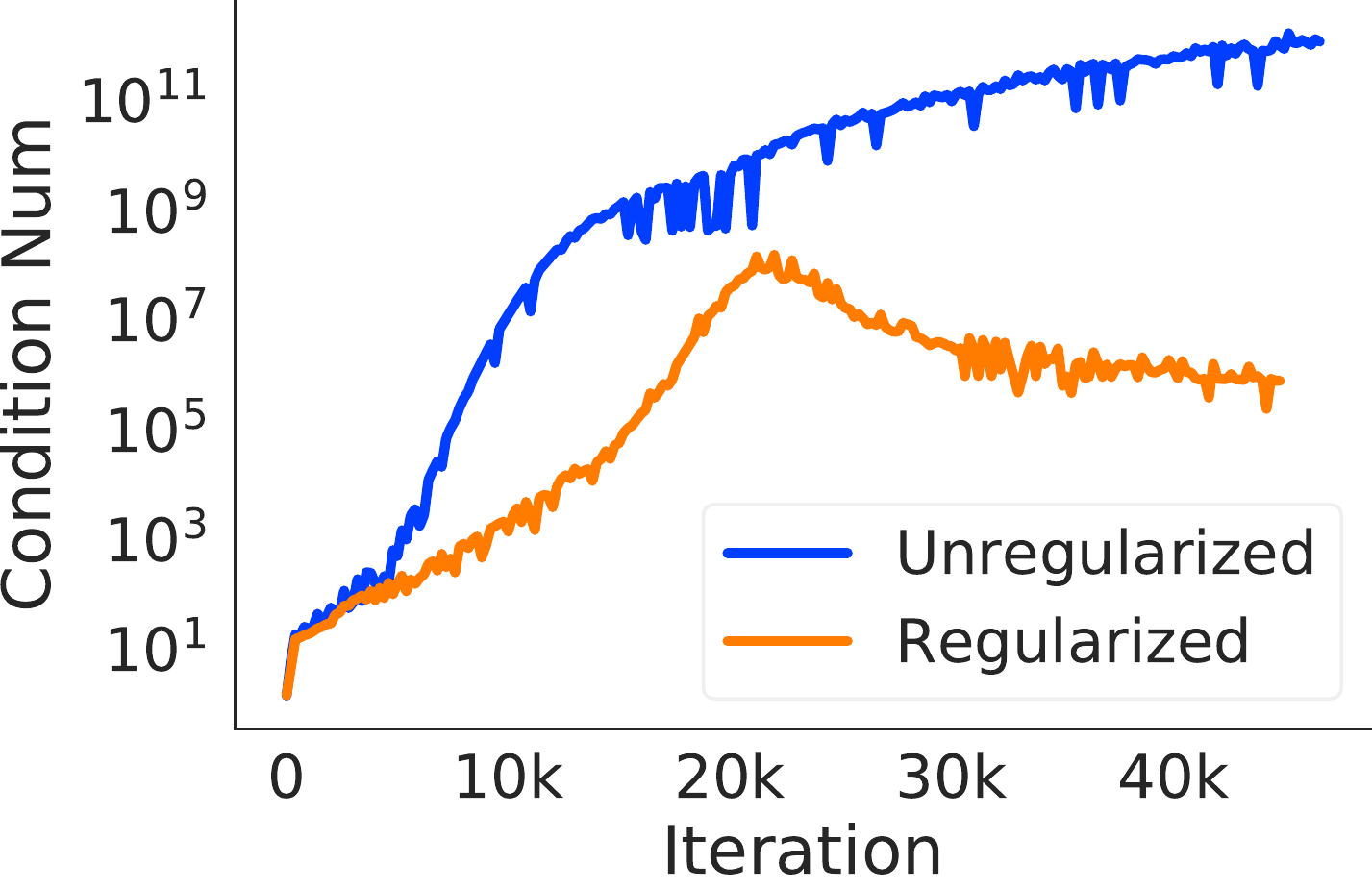}
\end{subfigure}
\caption{\textbf{A comparison of the mean-squared error (MSE) loss, reconstruction error, and condition number of an unregularized and a regularized Glow model trained on the toy 2D regression task.} The regularized model uses the normalizing flow objective with coefficient 1e-8.}
\label{fig:toy-training-plots}
\end{figure*}

\section{EXPERIMENTAL DETAILS FOR CLASSIFICATION EXPERIMENTS}
\label{app:classificationExpDetails}

In this section we provide details on the experimental setup for the classification experiments in Section~\ref{sec:classification}, as well as additional results.
We used the PyTorch framework~\citep{paszke2019pytorch} for our INN implementations.
All experiments were run on NVIDIA Titan Xp GPUs.

\paragraph{Experimental Setup.}

All the additive and affine coupling-based models we used have the same architecture, that consists of 3 levels, 16 blocks per level, and 128 hidden channels.
Each level consists of a sequence of residual blocks that operate on the same dimensionality.
Between levels, the input is spatially downsampled by $2\times$ in both width and height, while the number of channels is increased by $4\times$.
Each block consists of a chain of $3 \times 3 \to \text{ReLU} \to 1 \times 1 \to \text{ReLU} \to 3 \times 3$ convolutions.
Because the dimension of the output of an INN is equal to that of the input (e.g., we have a 3072-dimensional feature space for $3\times 32 \times 32$ CIFAR-10 images), we use a projection layer ($\text{1D BatchNorm} \to \text{ReLU} \to \text{Linear}$) to map the feature representation to 10 dimensions representing class logits.
We trained the models on CIFAR-10 for 200 epochs, using Adam with initial learning rate 1e-4, decayed by a factor of 10 at epochs 60, 120, and 160 (following~\citep{zhang2018three}).
We used standard data normalization (transforming the data to have zero mean and unit variance), and data augmentation (random cropping and horizontal flipping).
The hyperparameters we used are summarized in Table~\ref{tab:hyperp-classifier}.

 \begin{table}[h]
 \centering
 \begin{tabular}{cc}
 \hline
  \textbf{Hyperparameter}       &  \textbf{Value}\\
  \hline
 Batch Size      & 128   \\
 Learning Rate   & 1e-4 (decayed by 10x at epochs $ \{60, 120, 160 \}$) \\
 Weight Decay    & 0     \\
 Optimizer       & Adam($\beta_1=0.9$, $\beta_2=0.999$) \\
 \# of Layers & 3 \\
 \# of Blocks per Layer & 16 \\
 \# of Conv Layers per Block & 3 \\
 \# of Channels per Conv & 128 \\
 
 \hline
 \end{tabular}
 \caption{\textbf{Hyperparameters for INN classifiers.}}
 \label{tab:hyperp-classifier}
 \end{table}

\begin{table*}[h]
\centering
\begin{tabular}{cc|cccccc}
\toprule
  \textbf{Method} & \textbf{Coeff.} & \textbf{Inv?} & \textbf{Test Acc} & \textbf{Recons. Err.} & \textbf{Cond. Num.} & \textbf{Min SV} & \textbf{Max SV}  \\
  \hline
   Add. C, None  &  0   & \cmark & 89.73 & 4.3e-2 & 7.2e+4 & 6.1e-2 & 4.4e+3 \\
  \hline
   Add. C, FD    & 1e-3 & \cmark & 88.34 & 5.9e-4 & 4.3e+1 & 2.1e-1 & 9.1e+0 \\
   Add. C, FD    & 1e-4 & \cmark & 89.10 & 9.3e-4 & 2.0e+2 & 9.8e-2 & 2.0e+1 \\
   Add. C, FD    & 5e-5 & \cmark & 89.71 & 1.1e-3 & 3.0e+2 & 8.7e-2 & 2.6e+1 \\
  \hline
   Add. C, NF    & 1e-3 & \cmark & 84.11 & 5.5e-4 & 6.5e+3 & 3.2e-2 & 2.1e+2  \\
   Add. C, NF    & 1e-4 & \cmark & 88.65 & 6.5e-4 & 4.7e+3 & 3.4e-2 & 1.6e+2 \\
   Add. C, NF    & 1e-5 & \cmark & 89.52 & 9.9e-4 & 1.7e+3 & 3.9e-2 & 6.6e+1 \\
   
   \bottomrule
   Aff. S, None  &   0  & \xmark & 87.64 & \texttt{NaN} & 9.9e+13 & 1.7e-12 & 1.6e+2       \\
  \hline
   Aff. S, FD    & 1e-3 & \cmark & 86.85 & 2.0e-5 & 4.0e+1  & 1.8e-1 & 7.1e+0 \\
   Aff. S, FD    & 1e-4 & \cmark & 88.14 & 3.1e-5 & 1.2e+2  & 1.1e-1 & 1.3e+1 \\
   Aff. S, FD    & 5e-5 & \cmark & 88.13 & 3.9e-5 & 1.8e+2  & 9.6e-2 & 1.7e+1 \\
  \hline
   Aff. S, NF    & 1e-3 & \cmark & 80.75 & 3.1e-5 & 2.2e+4  & 2.0e-2 & 4.4e+2 \\
   Aff. S, NF    & 1e-4 & \cmark & 87.87 & 2.3e-5 & 5.8e+3  & 3.1e-2 & 1.8e+2 \\
   Aff. S, NF    & 1e-5 & \cmark & 88.31 & 2.8e-5 & 8.5e+2  & 3.7e-2 & 3.2e+1 \\
   
   \bottomrule
   Aff. C, None  & 0    & \xmark & 89.07 & \texttt{Inf}    & 8.6e+14 & 1.9e-12 & 1.7e+3 \\
  \hline
   Aff. C, FD    & 1e-3 & \cmark & 88.24 & 6.0e-4 & 4.2e+1  & 2.0e-1  & 8.4e+0 \\
   Aff. C, FD    & 1e-4 & \cmark & 89.47 & 9.6e-4 & 1.6e+2  & 9.6e-2  & 1.5e+1 \\
  \hline
   Aff. C, NF    & 1e-3 & \cmark & 83.28 & 7.0e-4 & 1.0e+4  & 3.6e-2  & 3.8e+2 \\
   Aff. C, NF    & 1e-4 & \cmark & 88.64 & 8.6e-4 & 9.6e+3  & 2.7e-2  & 2.6e+2 \\
   Aff. C, NF    & 1e-5 & \cmark & 89.71 & 1.3e-3 & 2.2e+3  & 3.5e-2  & 7.7e+1 \\
\bottomrule
\end{tabular}
\caption{\textbf{Effect of the regularization coefficients for both finite differences regularization (denoted FD) and the normalizing flow regularizer (NF), when training several INN classifiers on CIFAR-10.}
All experiments in this table used Glow-like architectures with either additive or affine coupling, and either shuffle permutations or 1$\times$1 convolutions.
In the Method column, ``Add.'' and ``Aff.'' denote additive and affine coupling, respectively; ``C.'' and ``S.'' denote $1 \times 1$ convolutions and shuffle permutations, respectively.
Note that for the affine model with $1\times 1$ convolutions, FD coefficient 5e-5 was too small to ensure stabilization, so it is not included above.
}
\vspace{-0.5cm}
\label{table:regularization-comparison}
\end{table*}

\begin{figure*}[htbp]
\centering
\begin{subfigure}[t]{0.33\linewidth}
	\centering
    \includegraphics[width=\linewidth]{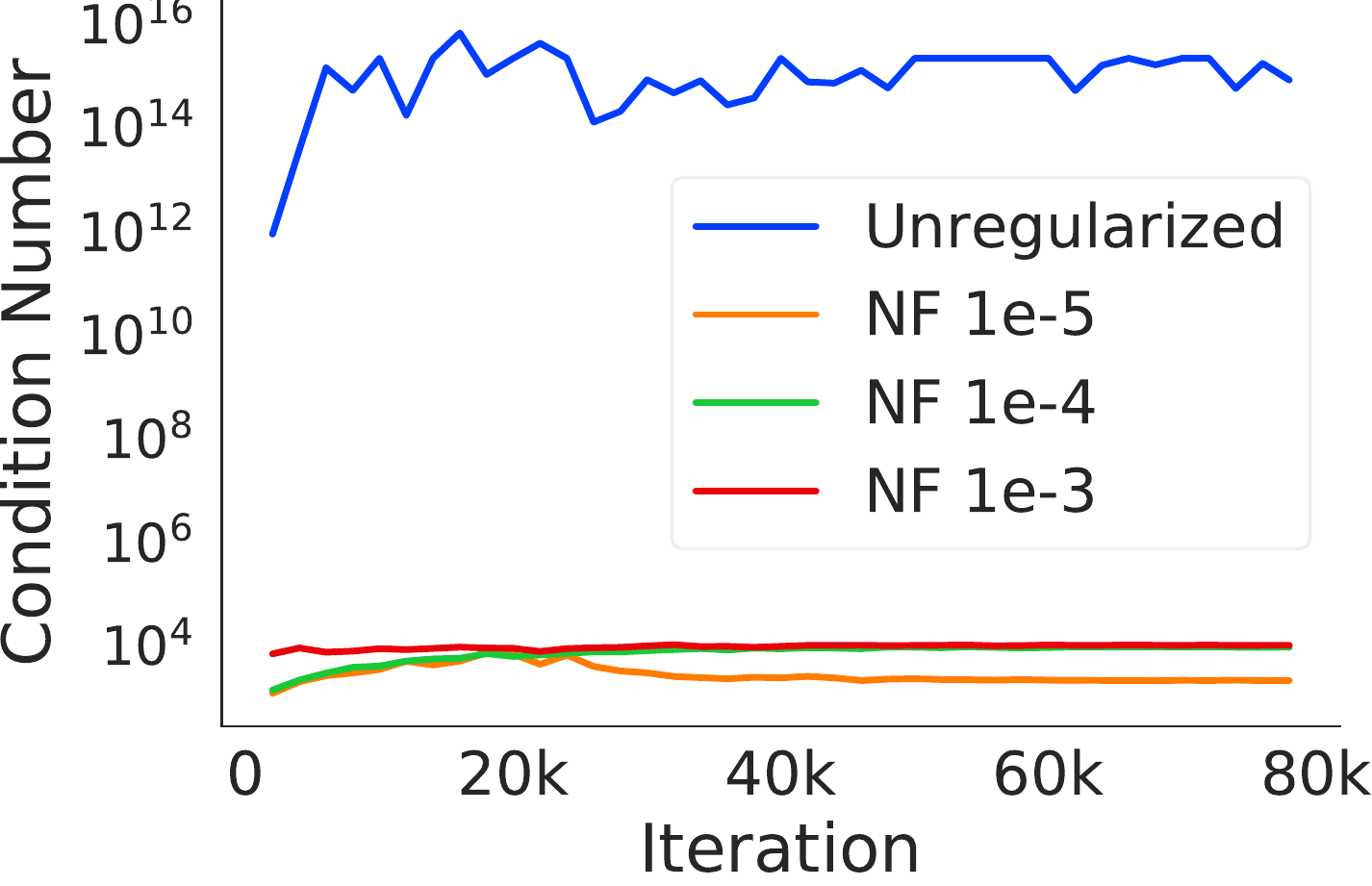}
\end{subfigure}
\quad \quad
\begin{subfigure}[t]{0.33\linewidth}
	\centering
    \includegraphics[width=\linewidth]{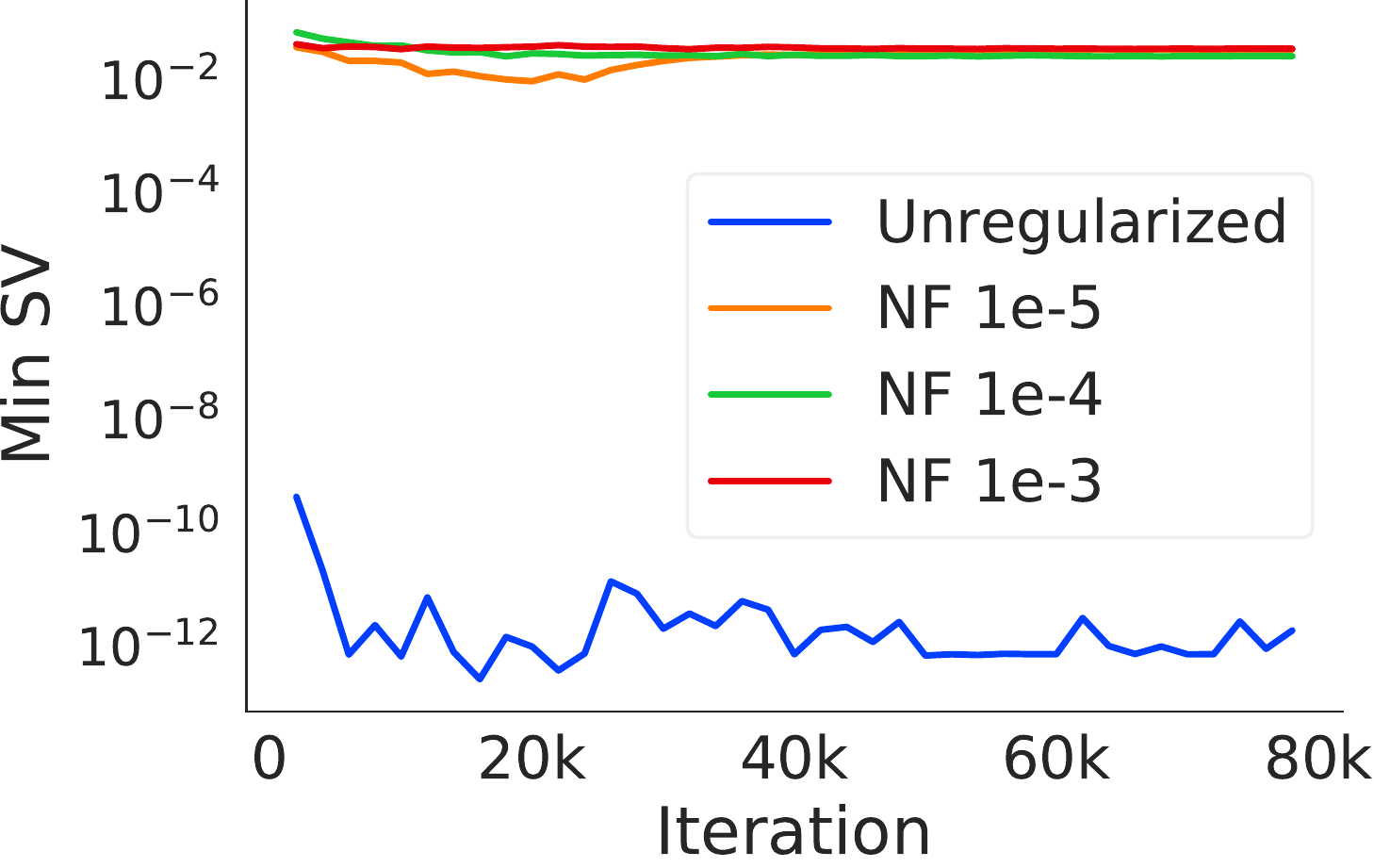}
\end{subfigure}
\\
\begin{subfigure}[t]{0.33\linewidth}
	\centering
    \includegraphics[width=\linewidth]{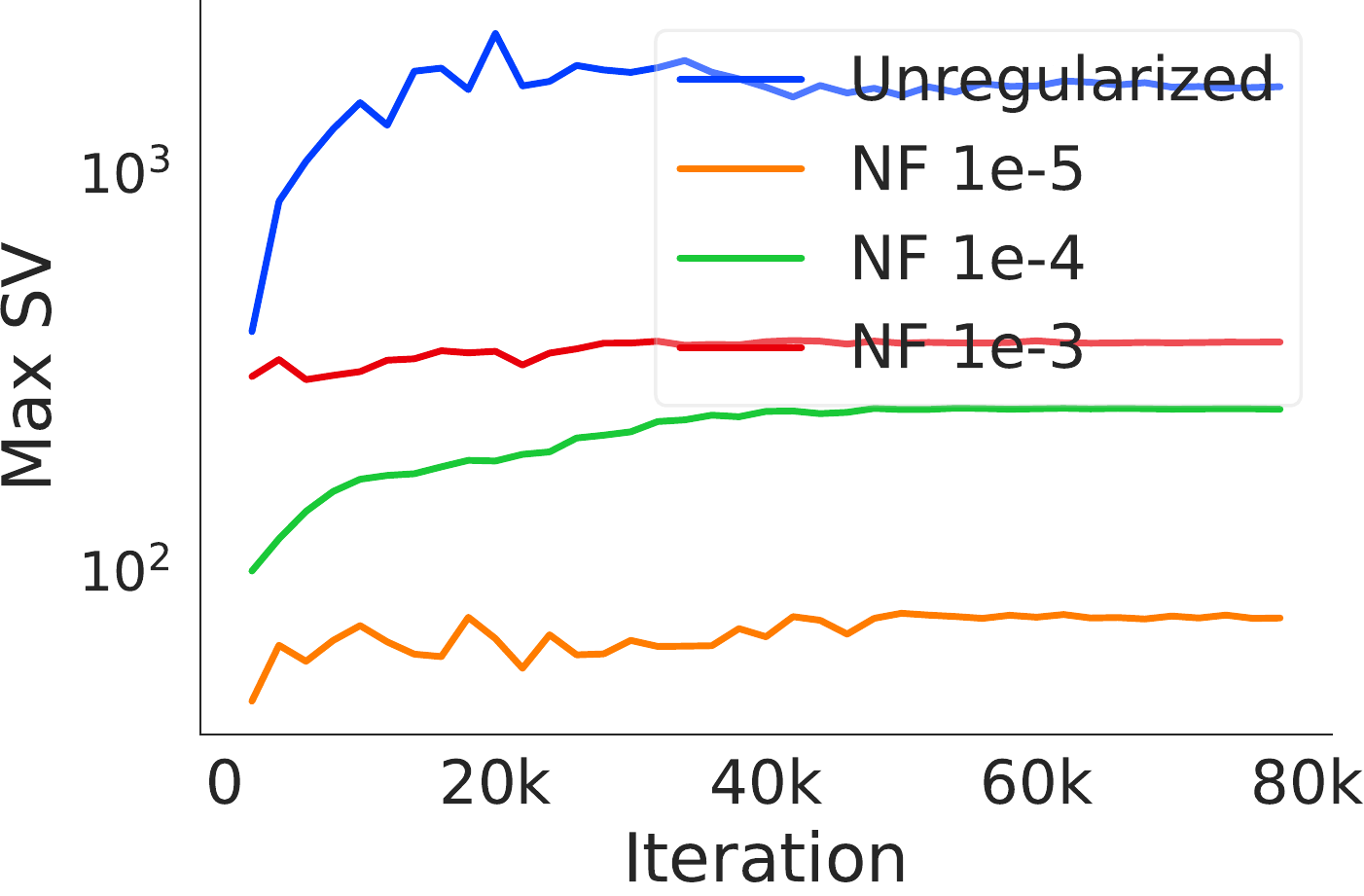}
\end{subfigure}
\quad \quad
\begin{subfigure}[t]{0.33\linewidth}
	\centering
    \includegraphics[width=\linewidth]{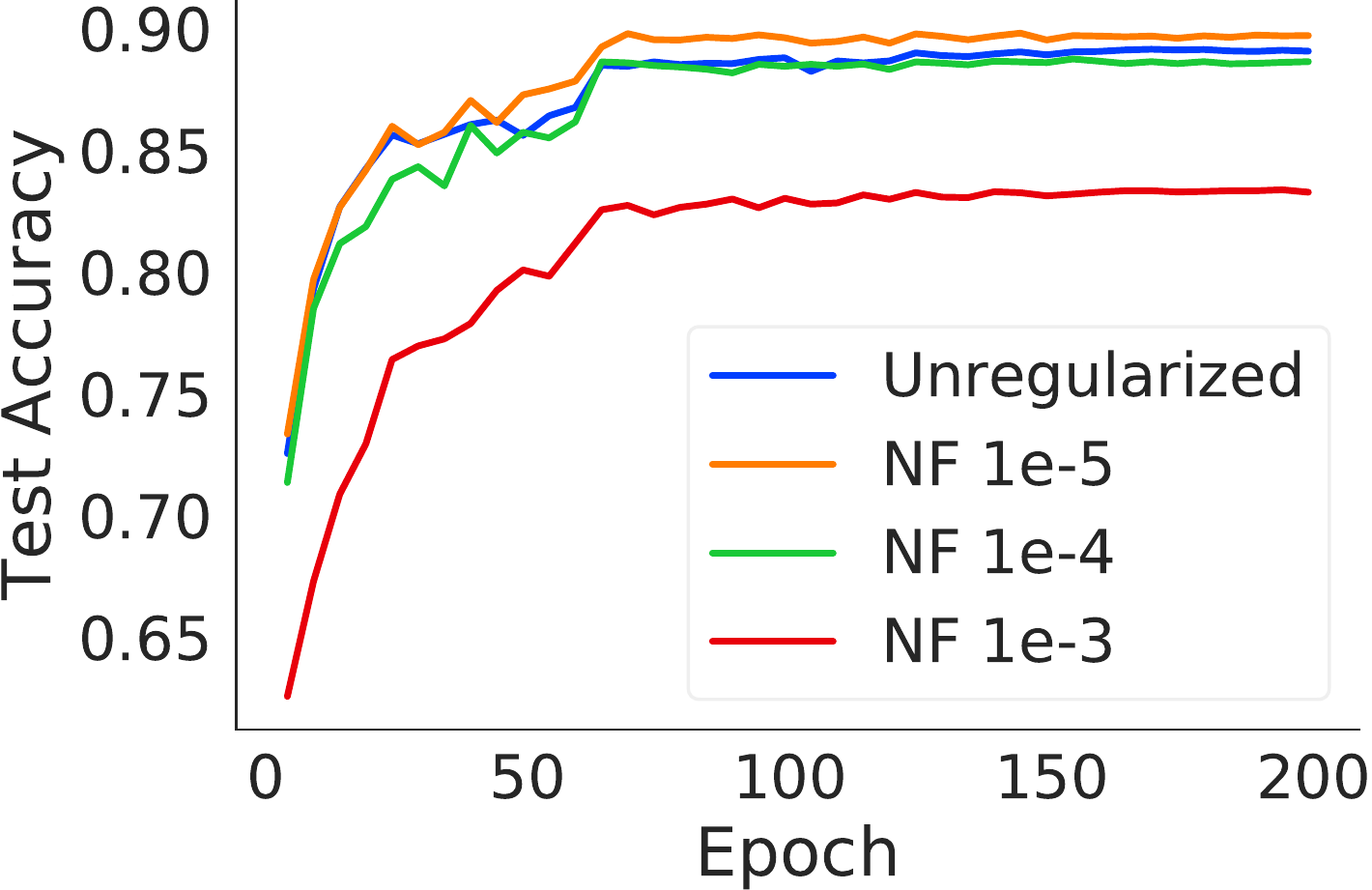}
\end{subfigure}
\caption{Comparison of regularization strengths for the \textbf{normalizing flow objective} during training of a Glow-like architecture (affine coupling, with $1 \times 1$ convolutions and ActNorm) on CIFAR-10. \textbf{Top-Left:} condition numbers for the un-regularized and regularized models; \textbf{Top-Right/Bottom-Left:} min/max singular values of the Jacobian; \textbf{Bottom-Right:} test accuracies.
}
\label{fig:cls-nf-comparison}
\vspace{-0.5cm}
\end{figure*}

\begin{figure*}[htbp]
\centering
\begin{subfigure}[t]{0.33\linewidth}
	\centering
    \includegraphics[width=\linewidth]{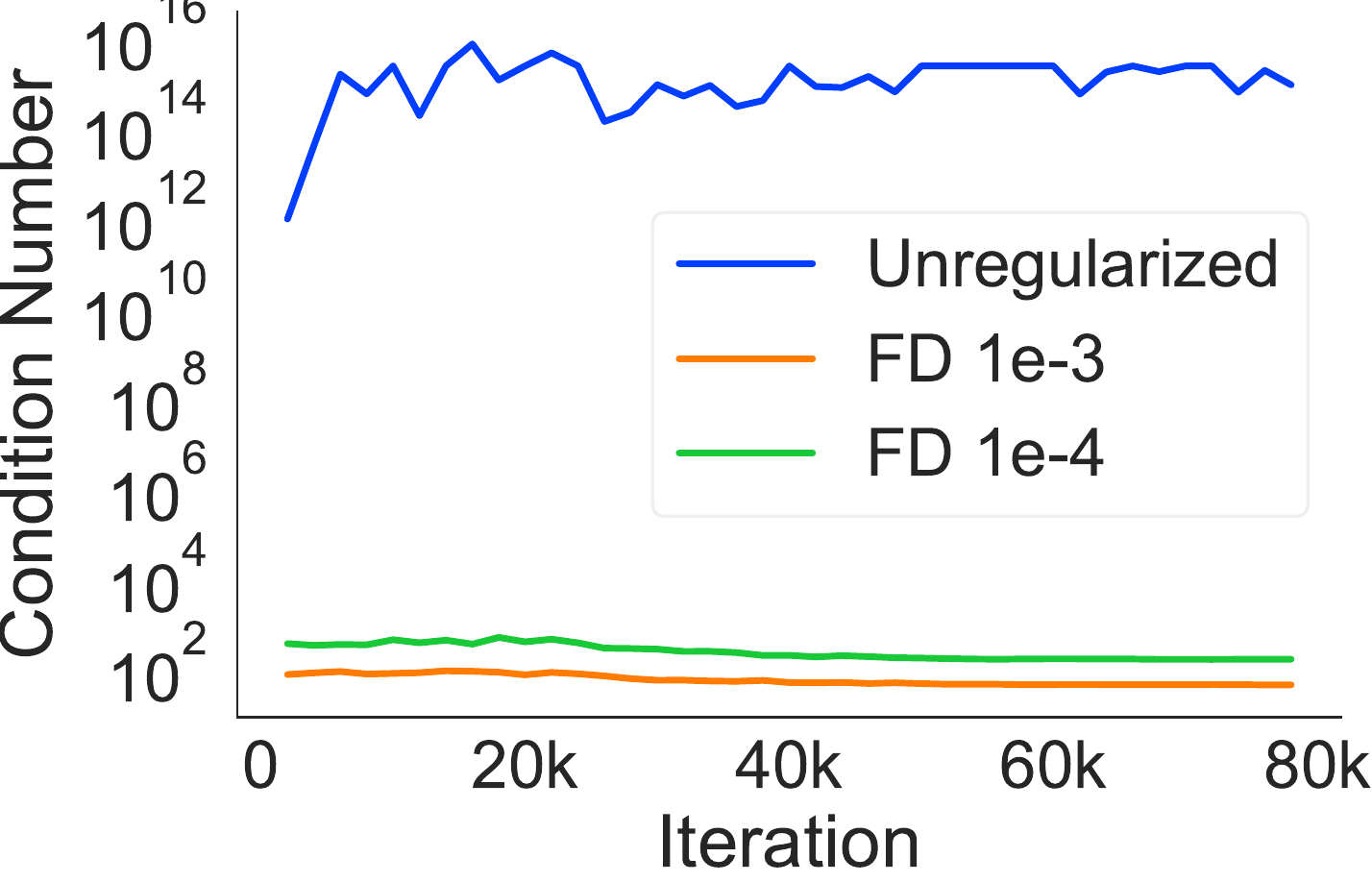}
\end{subfigure}
\quad \quad
\begin{subfigure}[t]{0.33\linewidth}
	\centering
    \includegraphics[width=\linewidth]{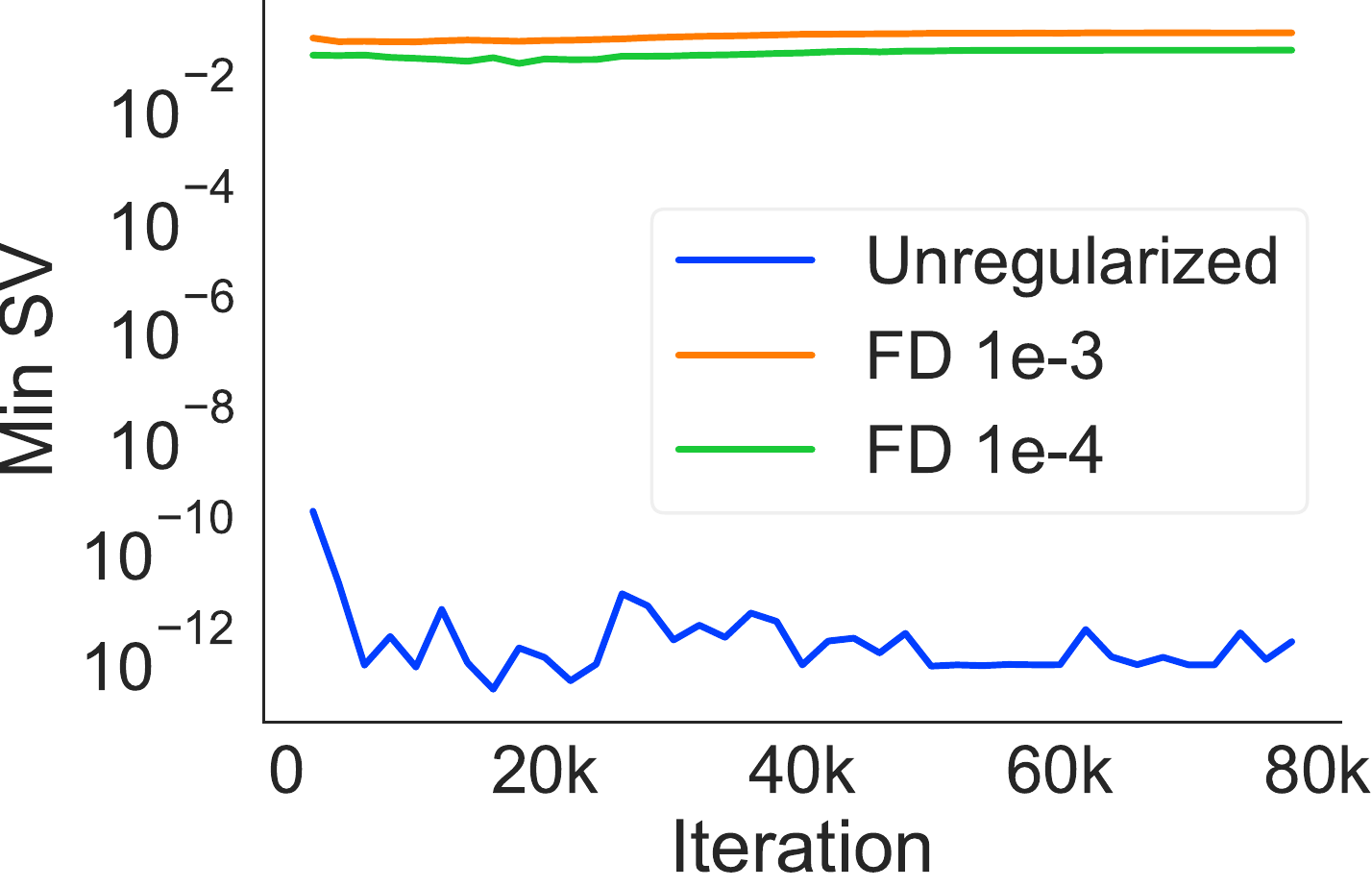}
\end{subfigure}
 \\
\begin{subfigure}[t]{0.33\linewidth}
	\centering
    \includegraphics[width=\linewidth]{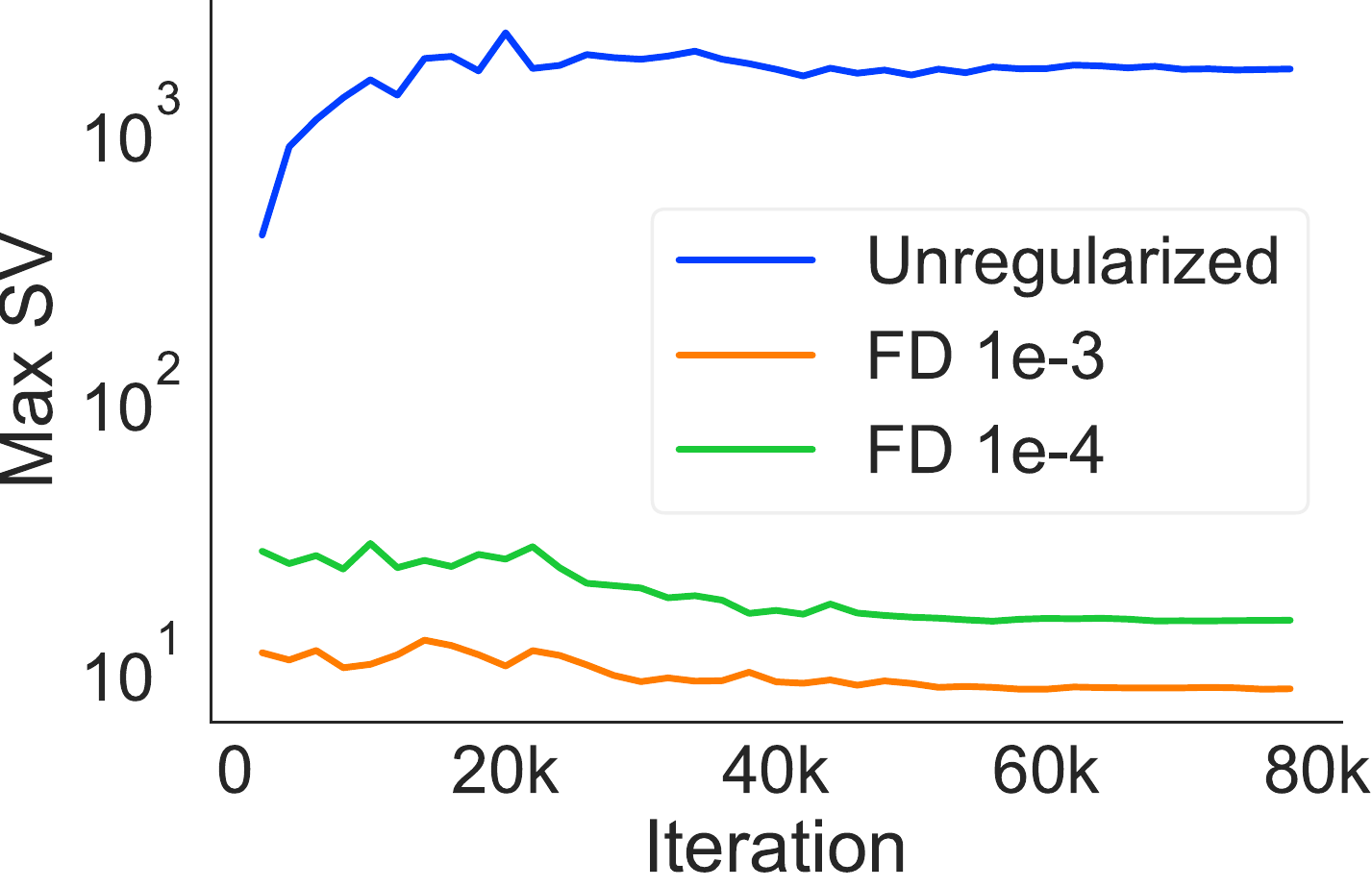}
\end{subfigure}
\quad \quad
\begin{subfigure}[t]{0.33\linewidth}
	\centering
    \includegraphics[width=\linewidth]{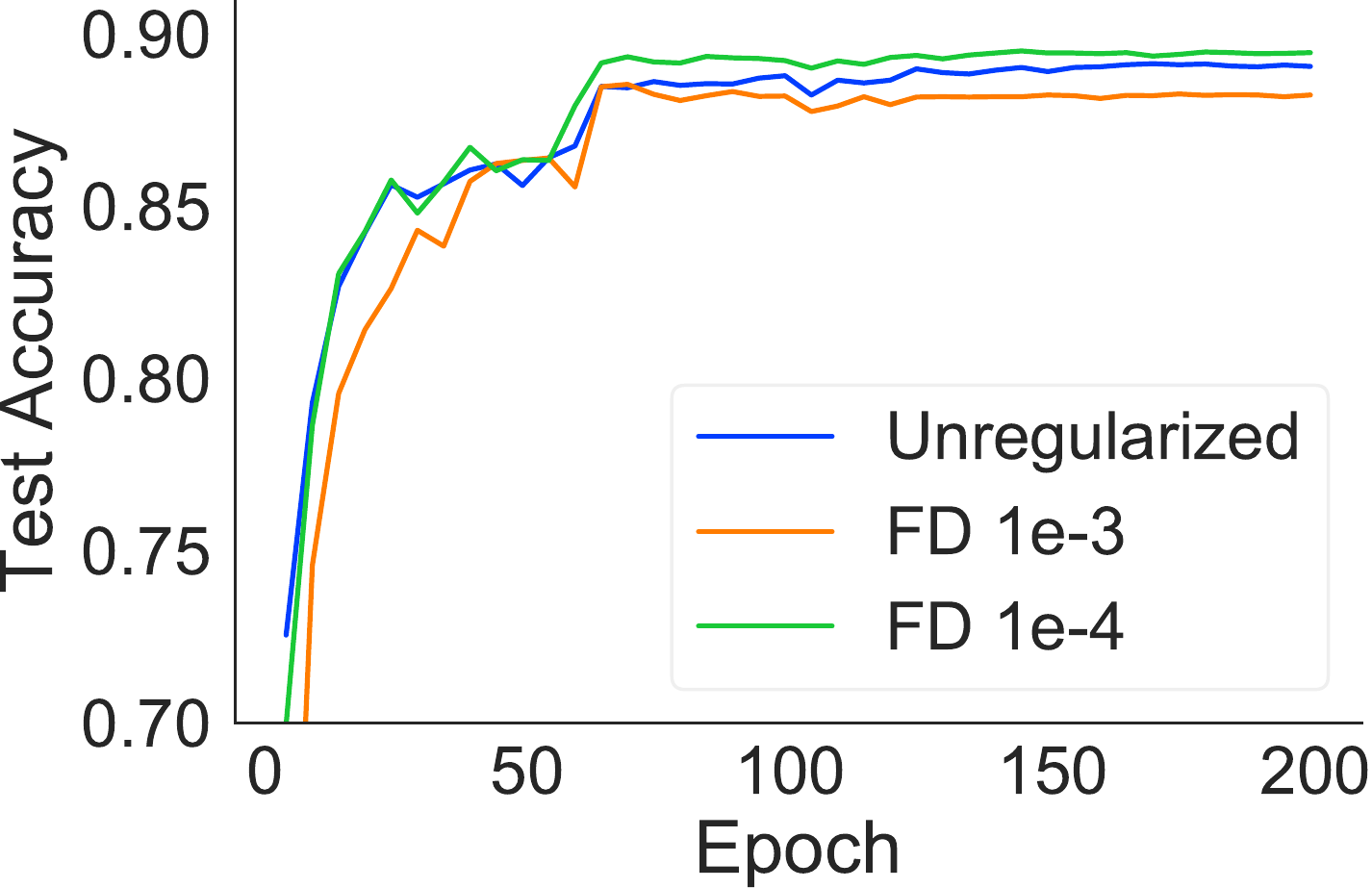}
\end{subfigure}
\caption{Comparison of regularization strengths for the \textbf{bi-directional finite differences (FD) regularizer} during training of a Glow-like architecture with affine coupling and $1 \times 1$ convolutions on CIFAR-10. \textbf{Top-Left:} condition numbers for the un-regularized and regularized models; \textbf{Top-Right/Bottom-Left:} min/max singular values of the Jacobian; \textbf{Bottom-Right:} test accuracies.
}
\label{fig:cls-jr-comparison}
\end{figure*}

\paragraph{Regularization Hyperparameters.}
We performed grid searches to find suitable coefficients for the regularization schemes we propose.
In particular, we searched over coefficients $\{$ 1e-3, 1e-4, 1e-5 $\}$ for the normalizing flow regularizer, and $\{$ 1e-3, 1e-4, 5e-5 $\}$ for bi-directional finite-differences (FD) regularization.
The effects of different regularization strengths are summarized in Table~\ref{table:regularization-comparison}.
We also plot the condition numbers, maximum/minimum singular values, and test accuracies while training with different coefficients for the normalizing flow objective (Figure~\ref{fig:cls-nf-comparison}) and bi-directional FD (Figure~\ref{fig:cls-jr-comparison}).
For both regularization methods, larger coefficients (e.g., stronger regularization) yield more stable models with smaller condition numbers, but too large a coefficient can harm performance (i.e., by hindering model flexibility).
For example, using a large coefficient (1e-3) for the normalizing flow objective substantially degrades test accuracy to 83.28\%, compared to 89.07\% for the un-regularized model.

\pagebreak

\paragraph{Extended Results.} We provide an extended version of Table~\ref{table:classification} (from the main paper) in Table \ref{table:classificationExtended}, where we also include additive and affine models with shuffle permutations.
Figures~\ref{fig:cls-additive-conv} and \ref{fig:cls-c-an1} visualize the stability during training using each regularizer, for additive and affine models, respectively.

\begin{table*}[htb]
\centering
\small
\begin{tabular}{cc|cccccc}
\toprule
  \textbf{Model} & \textbf{Regularizer} & \textbf{Inv?} & \textbf{Test Acc} & \textbf{Recons. Err.} & \textbf{Cond. Num.} & \textbf{Min SV} & \textbf{Max SV}  \\
  \midrule
    \multirow{3}{*}{Additive Shuffle}
     & None  & \cmark & 88.35 & 1.1e-4 & 2.1e+3 & 2.9e-2 & 6.0e+1 \\
     & FD    & \cmark & 88.49 & 5.4e-5 & 7.0e+2 & 5.3e-2 & 3.7e+1 \\
     & NF    & \cmark & 88.49 & 2.2e-5 & 1.1e+3 & 3.0e-2 & 3.3e+1  \\
  \hline
\multirow{3}{*}{Additive Conv}
     & None  & \cmark & 89.73 & 4.3e-2 & 7.2e+4 & 6.1e-2 & 4.4e+3 \\
     & FD    & \cmark & 89.71 & 1.1e-3 & 3.0e+2 & 8.7e-2 & 2.6e+1 \\
     & NF    & \cmark & 89.52 & 9.9e-4 & 1.7e+3 & 3.9e-2 & 6.6e+1 \\
 \hline
 \multirow{3}{*}{Affine Shuffle}
    & None   & \xmark & 87.64 & \texttt{NaN}    & 9.9e+13 & 1.7e-12 & 1.6e+2 \\
    & FD     & \cmark & 88.14 & 3.1e-5 & 1.2e+2  & 1.1e-1  & 1.3e+1  \\
    & NF     & \cmark & 88.31 & 2.8e-5 & 8.5e+2  & 3.7e-2  & 3.2e+1 \\
 \hline
 \multirow{3}{*}{Affine Conv}
     & None & \xmark & 89.07 & \texttt{Inf}    & 8.6e14 & 1.9e-12 & 1.7e+3 \\
     & FD & \cmark & 89.47 & 9.6e-4 & 1.6e+2 & 9.6e-2 & 1.5e+1  \\
     & NF & \cmark & 89.71 & 1.3e-3 & 2.2e+3 & 3.5e-2 & 7.7e+1 \\
\bottomrule
\end{tabular}

\caption{\textbf{Extended results: Effect of regularization when training several additive and affine INN architectures for CIFAR-10 classification.}
}
\label{table:classificationExtended}
\end{table*}

\begin{figure*}[htbp]
\centering
\begin{subfigure}[t]{0.29\linewidth}
	\centering
	\includegraphics[width=\linewidth]{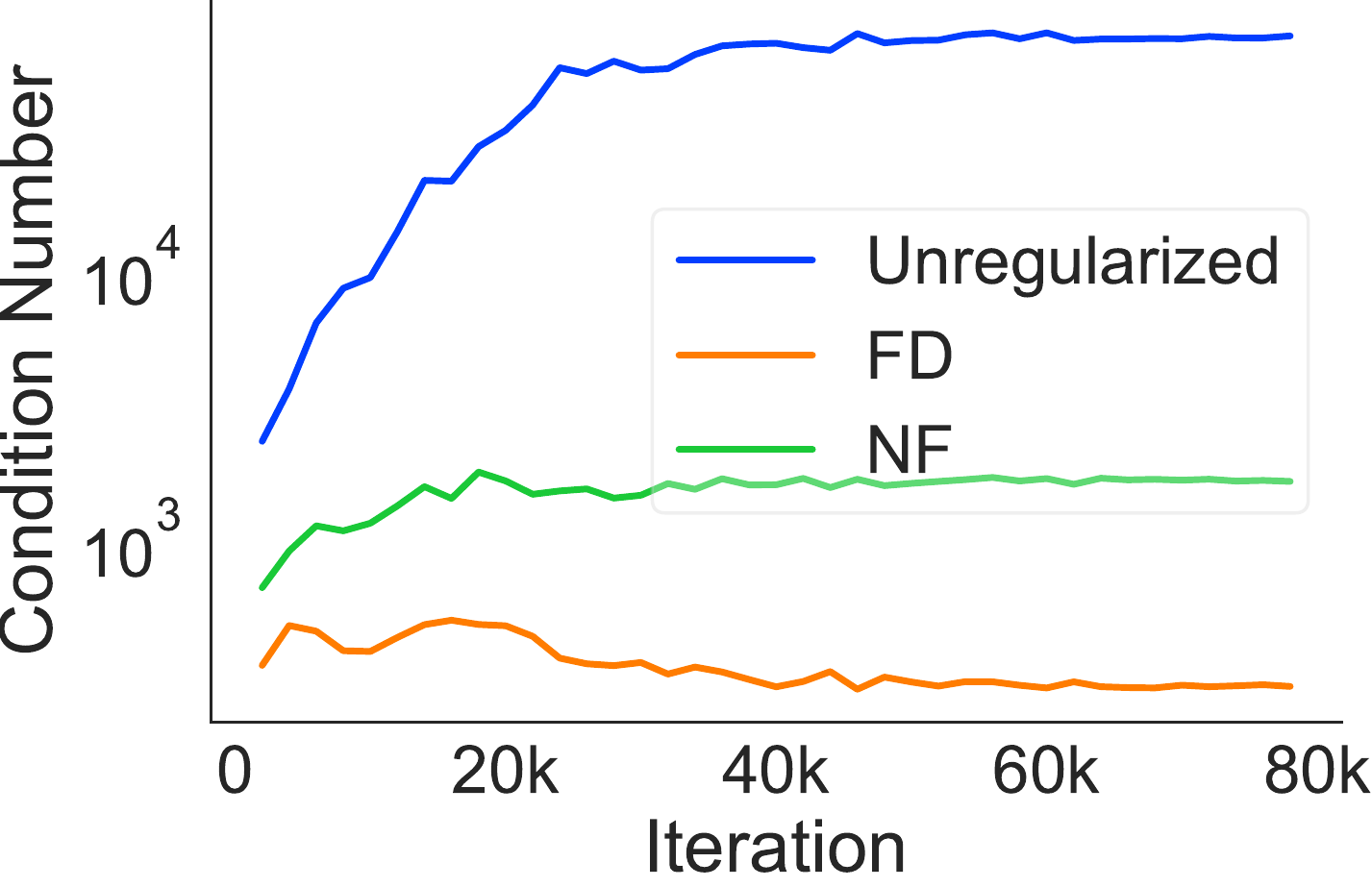}
\end{subfigure}
\hfill
\begin{subfigure}[t]{0.29\linewidth}
	\centering
	\includegraphics[width=\linewidth]{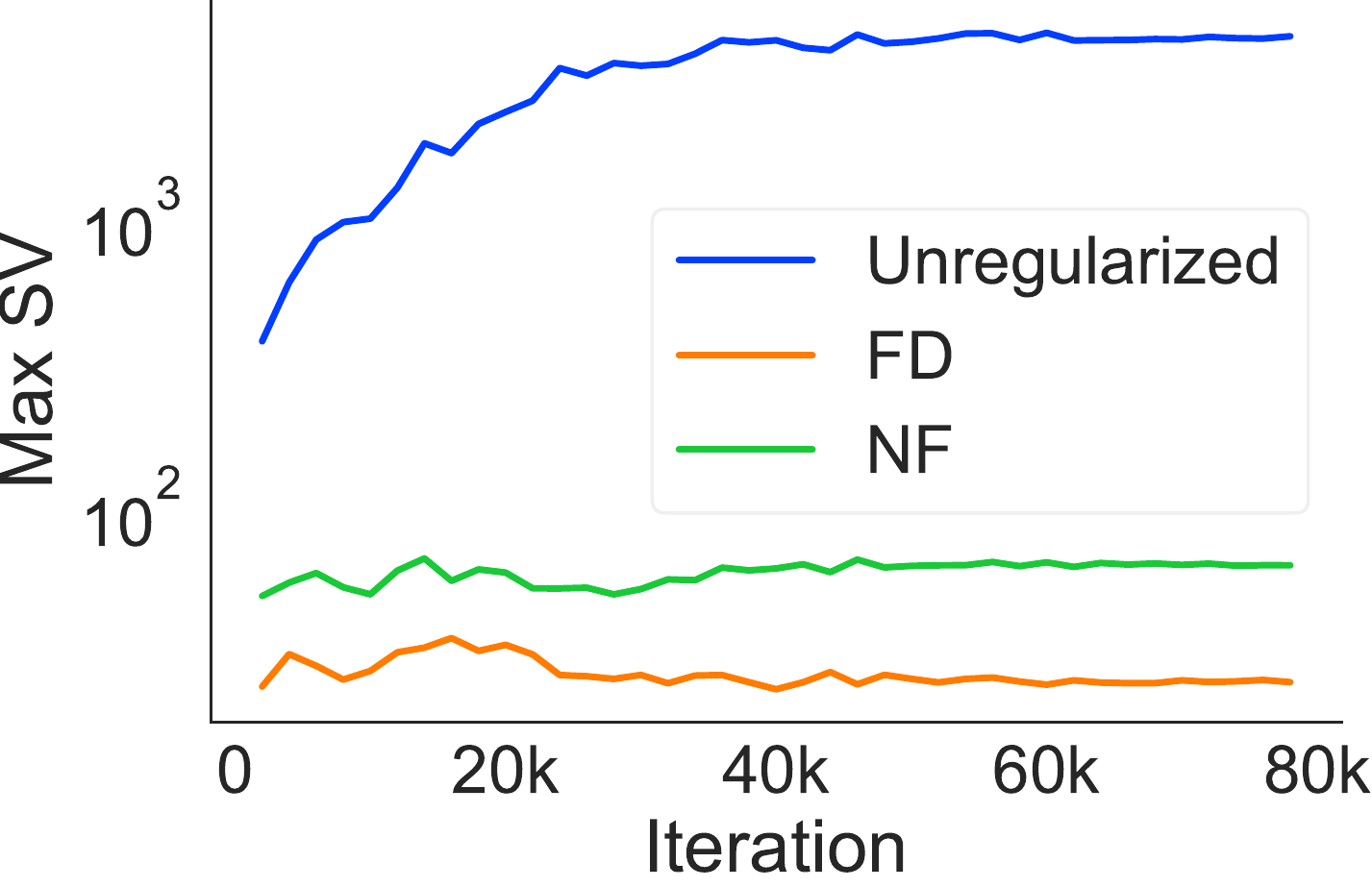}
\end{subfigure}
\hfill
\begin{subfigure}[t]{0.29\linewidth}
	\centering
	\includegraphics[width=\linewidth]{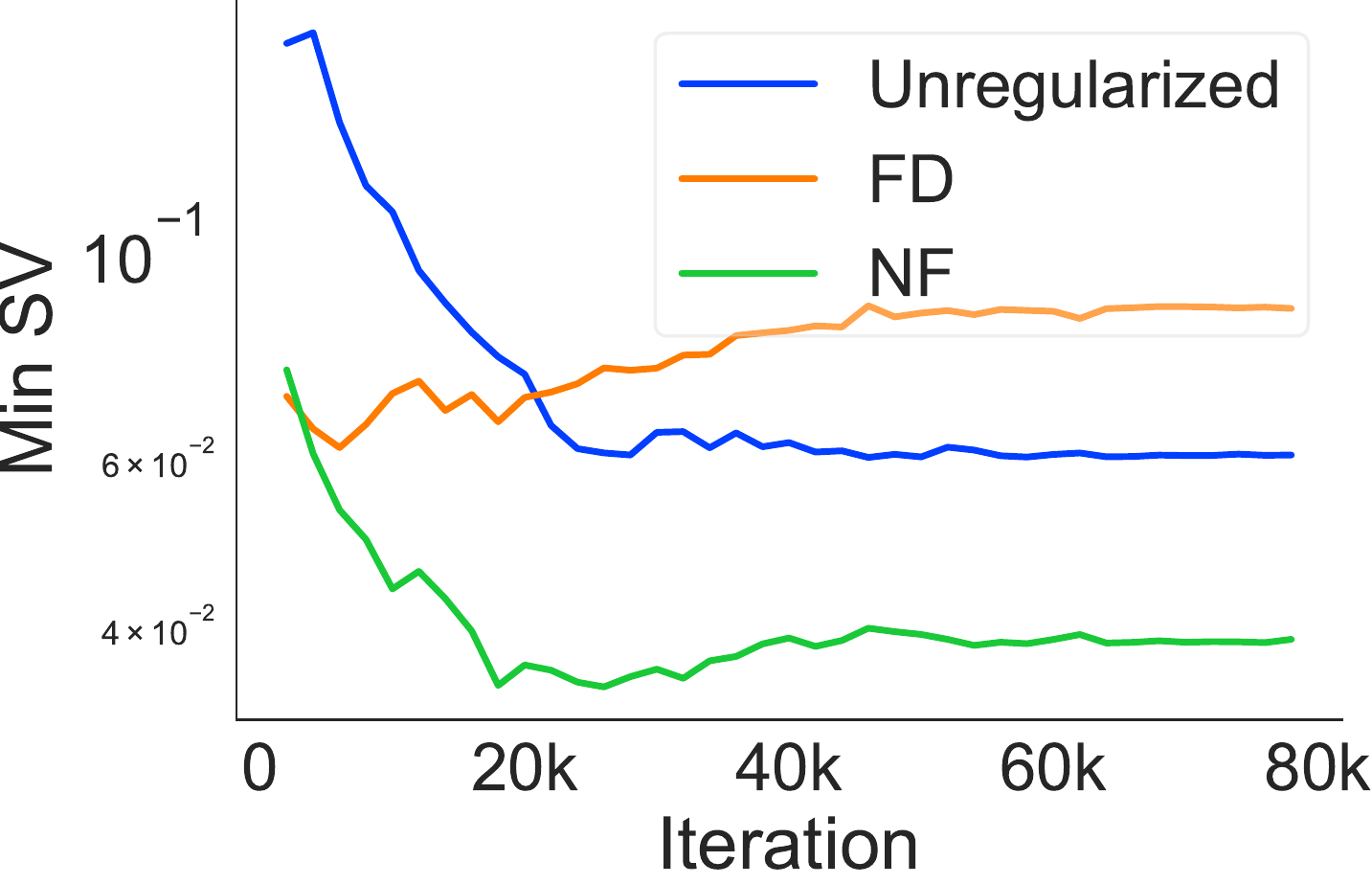}
\end{subfigure}
\caption{Stability during training of a Glow-like architecture with \textbf{additive coupling} and \textbf{$1 \times 1$ convolutions} on CIFAR-10. \textbf{Left:} condition numbers for the un-regularized and regularized models. \textbf{Middle/Right:} min/max singular values of the Jacobian. Note the large effect of both regularizes on the min singular value, indicating a more stable inverse mapping.
}
\label{fig:cls-additive-conv}
\end{figure*}

\begin{figure*}[htbp]
\centering
\begin{subfigure}[t]{0.29\linewidth}
	\centering
    \includegraphics[width=\linewidth]{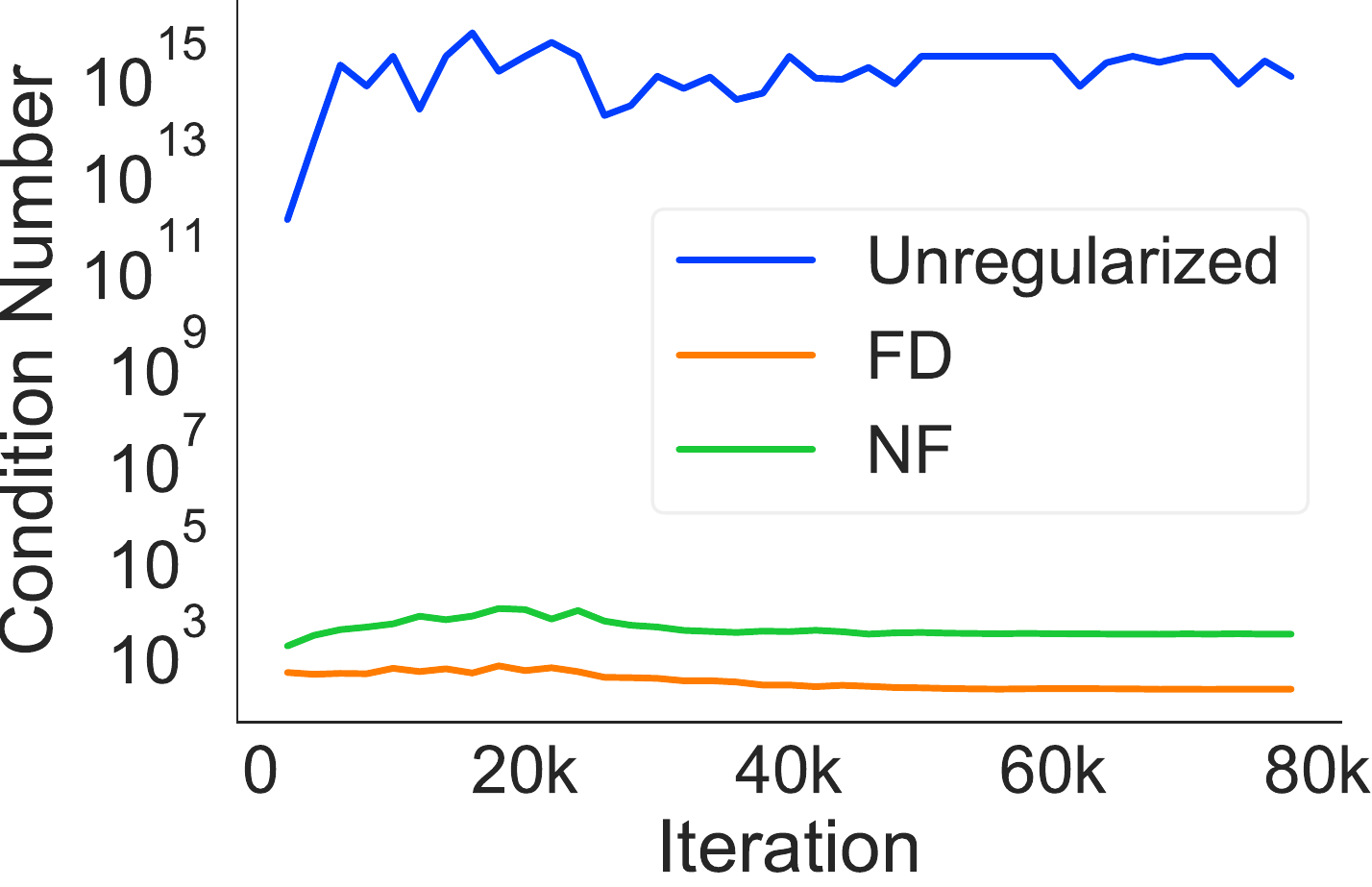}
\end{subfigure}
\hfill
\begin{subfigure}[t]{0.29\linewidth}
	\centering
	\includegraphics[width=\linewidth]{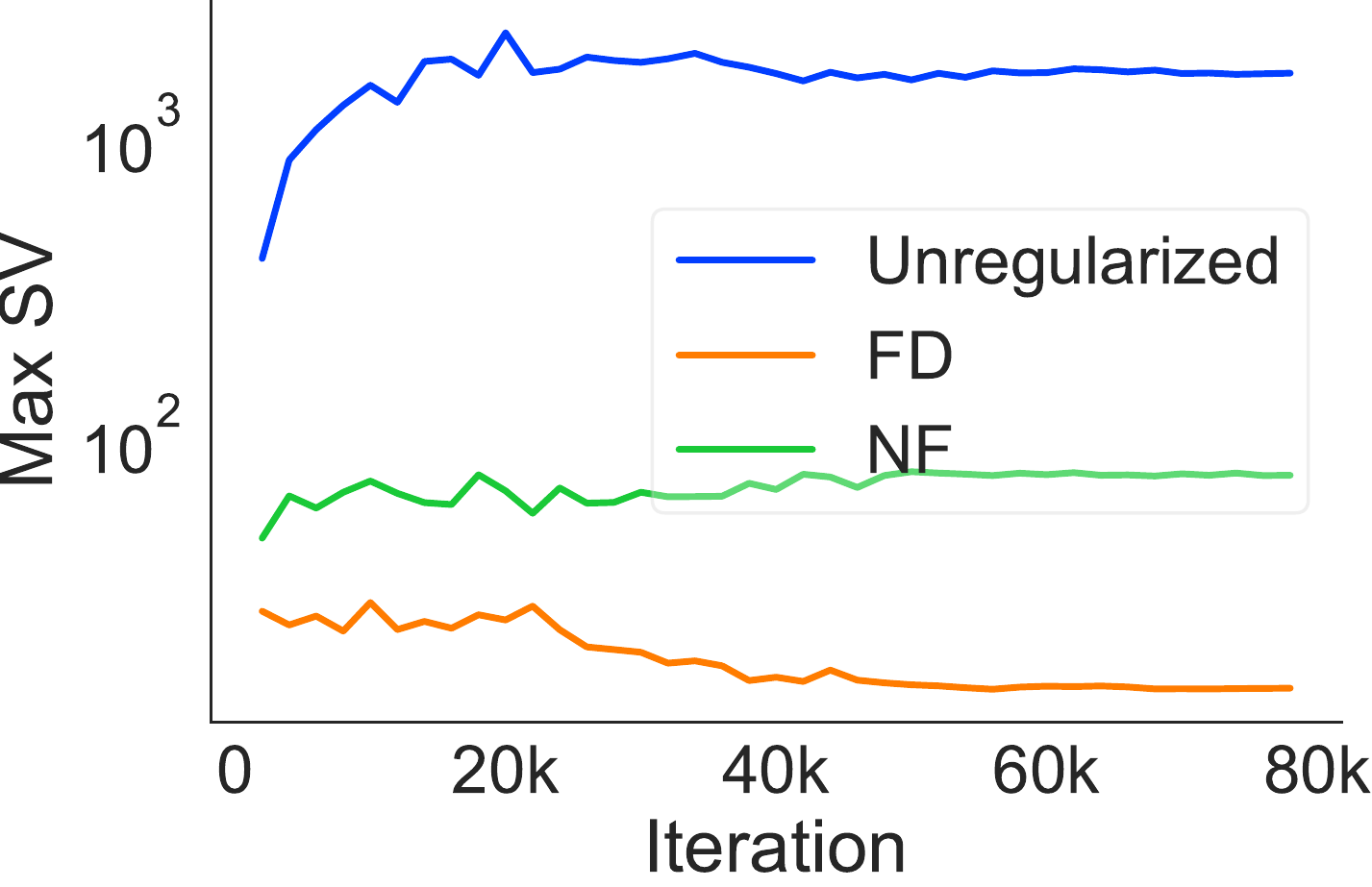}
\end{subfigure}
\hfill
\begin{subfigure}[t]{0.29\linewidth}
	\centering
	\includegraphics[width=\linewidth]{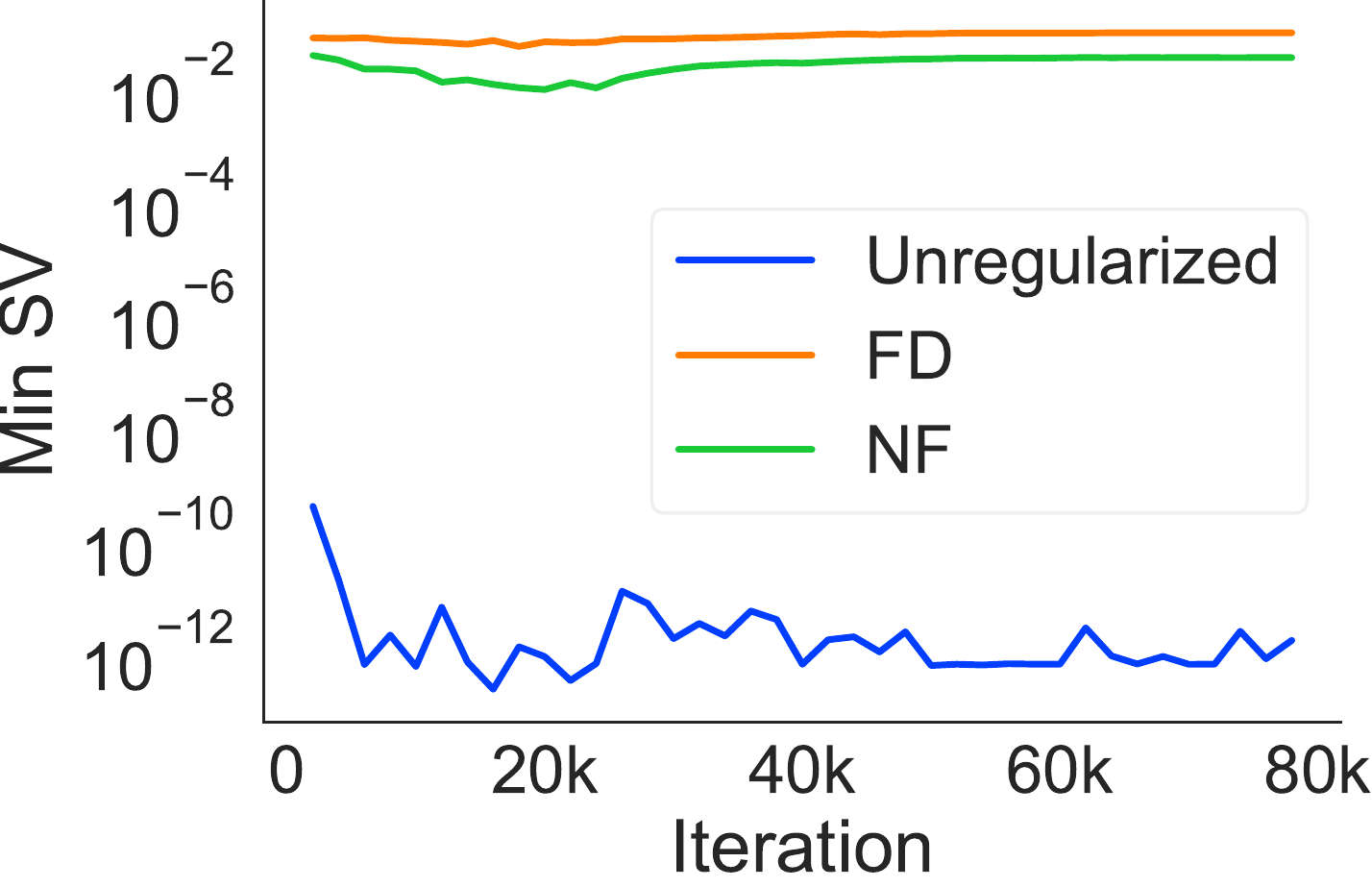}
\end{subfigure}
\caption{Stability during training of a Glow-like architecture with \textbf{affine coupling} and \textbf{$1 \times 1$ convolutions} on CIFAR-10. \textbf{Left:} condition numbers for the un-regularized and regularized models. \textbf{Middle/Right:} min/max singular values of the Jacobian.
}
\label{fig:cls-c-an1}
\end{figure*} 

\pagebreak

\paragraph{Finite Differences Regularization.}
As mentioned in Section~\ref{sec:controlStability}, for the bi-directional finite differences regularizer we used samples $v \sim \mathcal{N}(0,I)$. For the step size in Eq.~\ref{eq:finiteDiff}, we used $\epsilon = 0.1$ to avoid numerical errors due to \textit{catastrophic cancellation}. 

\paragraph{Computational Efficiency of Regularizers.}
Most invertible neural networks are designed to make it easy to compute the log determinant of the Jacobian, needed for the change-of-variables formula used to train flow-based generative models.
For coupling-based INNs, the log determinant is particularly cheap to compute, and can be done in the same forward pass used to map $x \mapsto z$.
Thus, adding a regularizer consisting of the weighted normalizing flow loss does not incur any computational overhead compared to standard training.

Bi-directional finite differences regularization is more expensive: memory-efficient gradient computation involves a forward pass, an inverse pass, and a backward pass.
The forward regularizer adds an additional overhead to these computations because it requires the previous computations not only for clean $x$, but also for noisy $x+\epsilon v$.
The inverse regularizer also passes two variables $z=F(x)$ and $\hat{z} = F(x) + \epsilon v^*$ through its computations, which are: an inverse pass to compute the reconstruction, a forward pass to re-compute activations of the inverse mapping and the backward pass through the inverse and forward.
However, in practice the cost can be substantially reduced by applying regularization only once per every $K$ iterations of training; we found that $K=5$ and $K=10$ performed similarly to $K=1$ in terms of their stabilizing effect, while requiring only a fraction more computation (e.g., 1.26$\times$ the computation when $K=10$) than standard training.
See the wall-clock time comparison in Table~\ref{tab:timeFD}.

We also experimented with applying the FD regularizer only to the inverse mapping, and while this can also stabilize the inverse and achieve similar accuracies, we found bi-directional FD to be more reliable across architectures.
Applying regularization once every 10 iterations, bi-FD is only 1.06$\times$ slower than inverse-FD.
As the FD regularizer is architecture agnostic and conceptually simple, we envision a wide-spread use for memory-efficient training.

\begin{table}
 \centering
 \begin{tabular}{ccc}
 \toprule
 \textbf{Method} & \textbf{Time per Epoch (s)} & \textbf{Overhead} \\ \midrule
 Unregularized   & 281.0 & 1$\times$        \\
 Reg. Every 1    & 854.1 & 3.04$\times$     \\
 Reg. Every 5    & 406.6 & 1.45$\times$     \\
 Reg. Every 10   & 354.6 & 1.26$\times$     \\
\bottomrule
\end{tabular}
\caption{\textbf{Timing comparison for applying bi-directional finite differences regularization at different frequencies during training (every 1, 5, or 10 iterations).}
Time is measured in seconds and corresponds to a single training epoch.
All rows use memory-saving gradient computation; we used additive models with $1 \times 1$ convolutions, as these are also trainable without regularization, allowing us to time unregularized memory-saving gradients.
}
 \label{tab:timeFD}
\end{table}

\paragraph{Comparing True and Memory-Saving Gradients.}

\begin{figure}[H]
	\centering
	\includegraphics[width=0.6\linewidth]{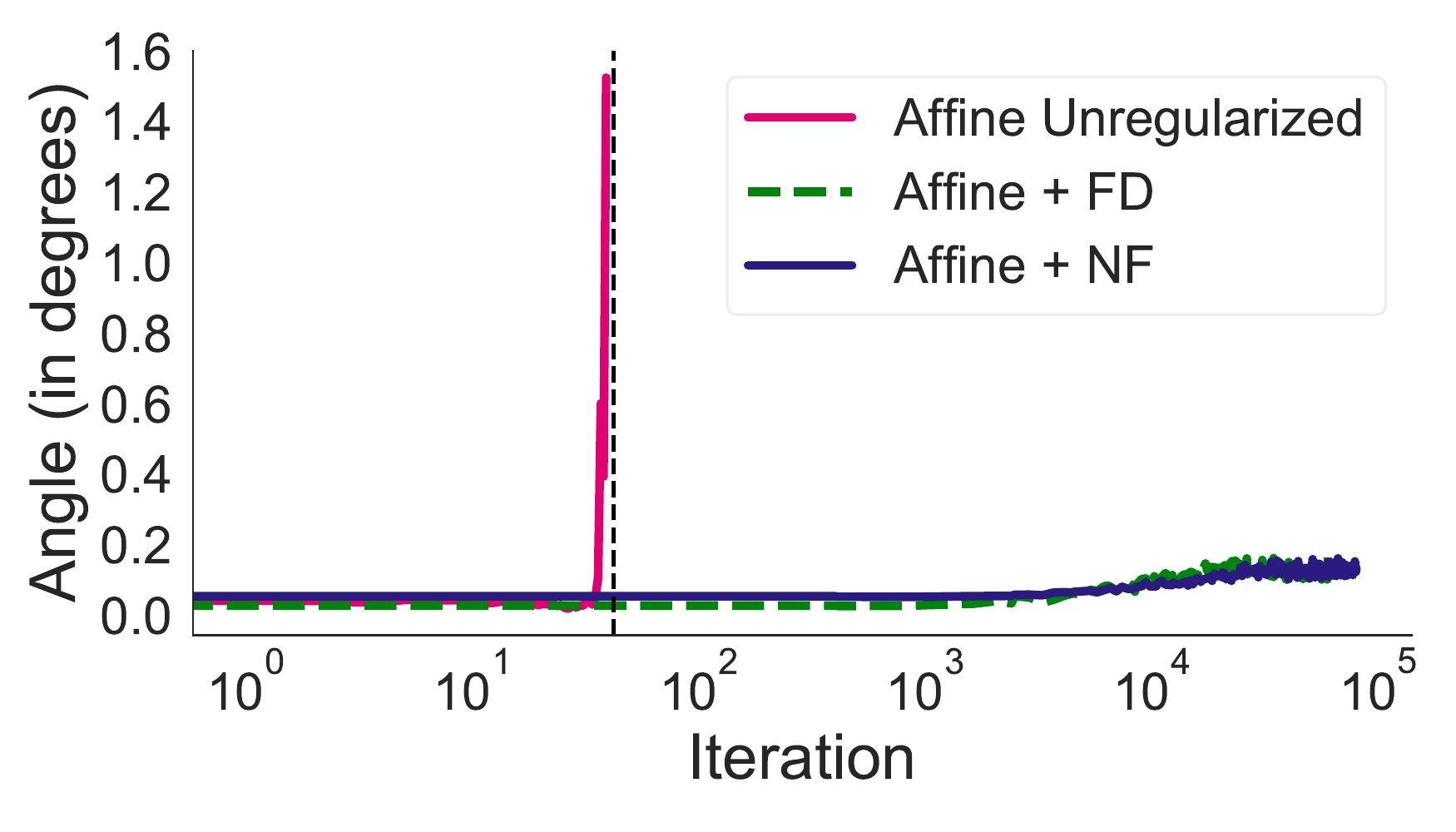}
\caption{\textbf{Comparison of the angle between memory-saving and true gradients during training of a Glow model with affine coupling and $1 \times 1$ convolutions on CIFAR-10.} While the unregularized affine model suffers from an exploding inverse (the angle rapidly increases at the start of training and becomes \texttt{NaN} after the dashed vertical line), adding either finite differences (FD) or normalizing flow (NF) regularization stabilizes the model and keeps the angles small throughout training.}
\label{fig:affine-angles-regularization}
\end{figure} 

For the gradient angle figure in the introduction (Section~\ref{sec:intro}), we used additive and affine Glow models with $1 \times 1$ convolutions.
To create that figure, we first trained both models with exact (non-memory-saving) gradients, and saved model checkpoints during training; then, we loaded each checkpoint and computed the angle between the true and memory-saving gradients for a fixed set of 40 training mini-batches and plotted the mean angle.
We did not train those models with memory-saving gradients, because this would have made it impossible to train the unstable affine model (due to the inaccurate gradient).

Figure~\ref{fig:affine-angles-regularization} shows a similar analysis, comparing the angles between true and memory-saving gradients for an affine model trained with and without regularization.
The regularized models were trained with memory-saving gradients, while the unregularized model was trained with standard backprop.
Similarly to the introduction figure, we saved model checkpoints throughout training, and subsequently loaded the checkpoints to measure the angle between the angles for the 40 training mini-batches, plotting the mean angle.
Applying either finite differences or normalizing flow regularization when training the affine model keeps the angles small, meaning that the memory-saving gradient is accurate, and allowing us to train with this gradient.

\section{OUTLOOK ON BI-DIRECTIONAL TRAINING WITH FLOW-GAN}
\label{app:flowgan}

\paragraph{Experimental Details.}
Table~\ref{tab:hyperp-flow-gan} summarizes the hyperparameters for all models trained in our Flow-GAN experiments.
The INN architecture in this experiment is similar to those described in Appendix~\ref{app:OODsampleExp}.  Table~\ref{tab:hyperp-flow-gan} lists hyperparameters changed for this experiment.

 \begin{table}[h]
 \centering
 \begin{tabular}{cc}
 \hline
  \textbf{Hyperparameter}       &  \textbf{Value}\\
  \hline
 Batch Size                    & 32   \\
 Learning Rate (Generator)     & 5e-5 \\
 Learning Rate (Discriminator) & 5e-5 \\
 Weight Decay (Both)           & 0    \\
 Optimizer                     & Adam($\beta_1=0.5$, $\beta_2=0.99$) \\
 \hline 
 \# of Layers & 3 \\
 \# of Blocks per Layer & 8 \\
 \# of Conv Layers per Block & 3 \\
 \# of Channels per Conv & 128 \\
 
 \hline
 \end{tabular}
 \caption{\textbf{Hyperparameters for ADV models}}
 \label{tab:hyperp-flow-gan}
 \end{table}

The prior is a standard normal.
One extra hyperparameter here is the weight of the MLE loss (which is considered as a regularizer here).  The discriminator had 7 convolution layers, each is regularized by Spectral normalization~\citep{miyato2018spectral} and using the LeakyReLU activation.  
The depth of the convolution layers are (64, 64, 128, 128, 256, 256, 512).  The discriminator is further regularized with gradient penalty~\citep{gulrajani2017improved} and optimized using the binary cross entropy GAN loss. 

\paragraph{Stability at Data.}
Similarly to the results in the classification Section~\ref{sec:classification}, when a flow is trained as a GAN, it suffers from non-invertibility even at training points.
This is easily explained as the INN generator is never explicitly trained at the data points, but only receives gradients from the discriminator.  

Using the MLE objective on the generator alleviates this issue.
It increases the stability at the data-points, and makes training more stable in general, as discussed in section \ref{sec:CoVstability}.
This is exactly the FlowGAN formulation proposed in~\citep{grover2018flow}.
See Table~\ref{table:flowgan-numbers} for a comparison of test-set likelihoods and sample quality.
In contrast to~\citep{grover2018flow}, we find additive Glow models trained with the GAN objective alone to be very unstable.
We used a weighing of 1e-3 for the MLE loss for the reported results here. 
One suspicion is that the introduced ActNorm in the Glow architecture is causing extra instability. 

Originally, one hypothesis was that the astronomically large BPD when trained as a GAN was caused by instability. 
This is likely incorrect.
Even for stable runs (with a smaller, more stable architecture) that had no reconstruction errors and small condition numbers, the model still assigns incredibly large BPDs.

Our FlowGAN models strike a good balance between good test-set likelihood and improving sample quality.
See Figure~\ref{fig:flowgan-cifar-samples} for samples.
The FlowGAN samples show better ``objectness'' and do not have the high-frequency artifacts shown in the MLE samples. 

Unfortunately, given the fact that the GAN-only objective assigns unexplained BPDs, it's still unclear whether interpreting the BPDs for the FlowGAN in the same way as MLE models is recommended.  

\begin{table}[H]
\setlength{\tabcolsep}{3pt}
\centering
    \begin{tabular}{l|c|c}
    \textbf{Objective} 
    & \textbf{Inception Score}
    & \textbf{BPD} \\
    \hline
    MLE & 2.92 & 3.54 \\
    GAN & 5.76 & 8.53 \\
    FlowGAN & 3.90 & 4.21 \\
    \hline
    \textbf{-} 
    & \textbf{FID}
    & \textbf{BPD} \\
    \hline
    MLE & 790 & 3.77 \\
    GAN* & 850 & {\color{red} \texttt{Inf}} \\
    FlowGAN & 420 & 4 \\
    \bottomrule
    
    \end{tabular}
\caption{\textbf{Comparing test-set likelihood and sample quality.} Top half of table copied from \citep{grover2018flow}.  *We find using only the GAN objective to be very unstable. Training can diverge after recording reasonable FID.  One definitive difference between our results and \citep{grover2018flow} is BPD using only the GAN objective. Our runs often get infinite BPDs in a few hundred iterations, whereas \citep{grover2018flow} reports a BPD of 8.53.  After inspecting their Table 2 together with Figure 2, the BPD and Inception Score seem to be reported for different times during training. 
}

\label{table:flowgan-numbers}
\end{table}

\begin{figure}[t]
\setlength{\tabcolsep}{3pt}
\centering
    \begin{tabular}{ccc}
    \textbf{MLE} 
    & \textbf{GAN}
    & \textbf{FlowGAN} \\
    \hline
    \includegraphics[trim={3cm 4cm 6cm 4.5cm},clip,width=0.25\textwidth]{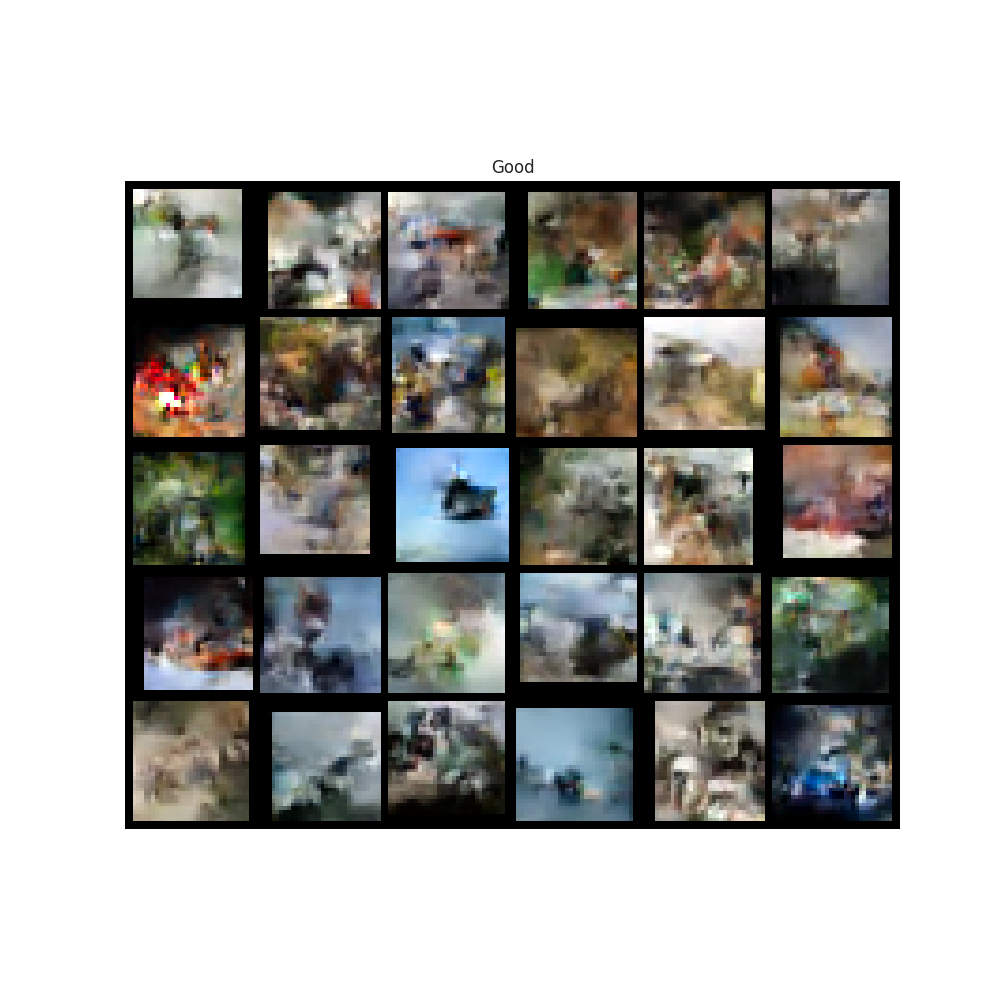} &
    \includegraphics[trim={3cm 4cm 6cm 4.5cm},clip,width=0.25\textwidth]{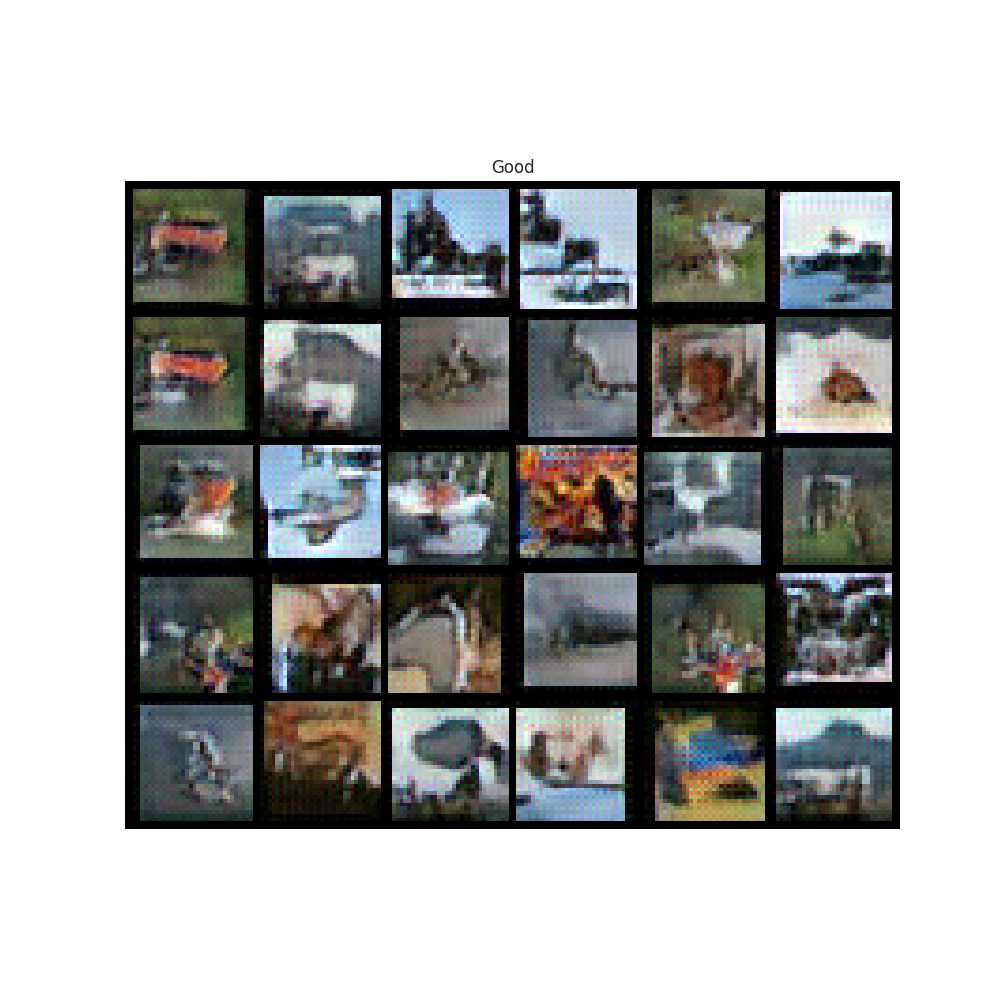} &    \includegraphics[trim={3cm 4cm 6cm 4.5cm},clip,width=0.25\textwidth]{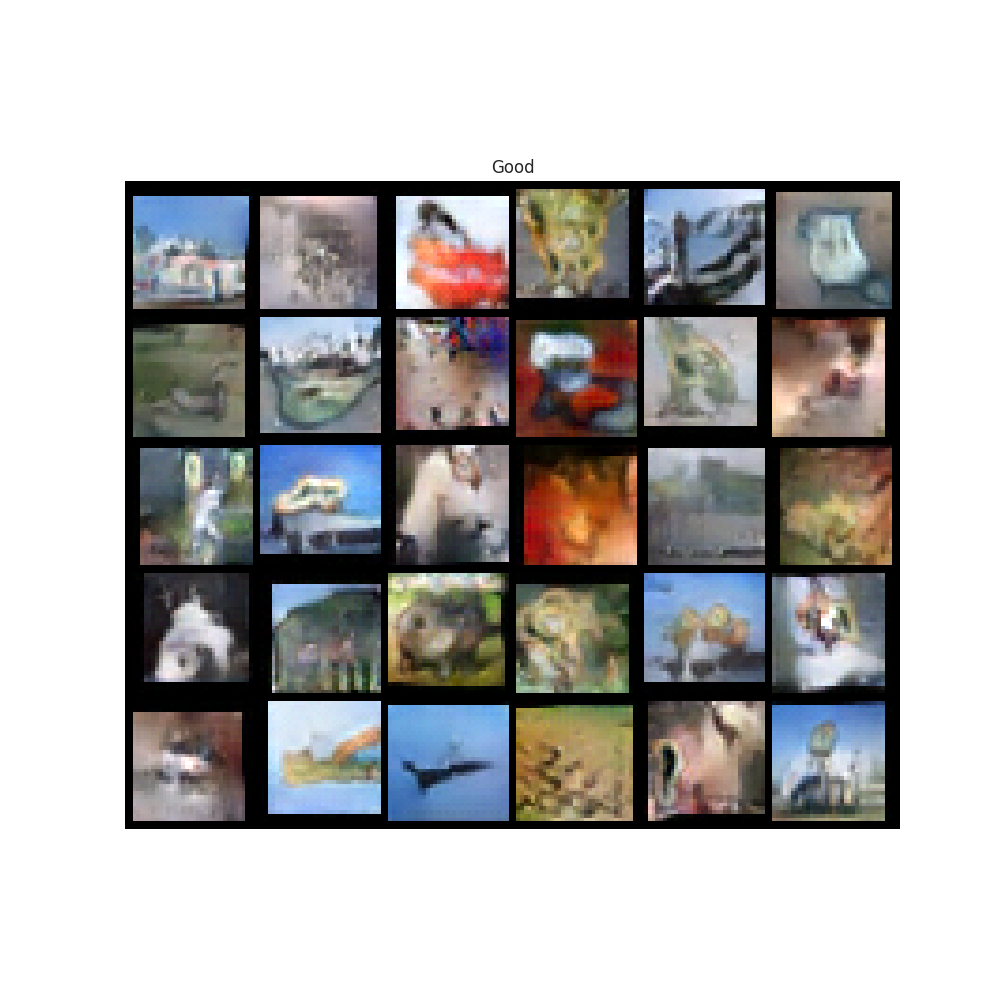}  \\
    \includegraphics[trim={.3cm 1.3cm .5cm .3cm},clip,width=0.25\textwidth]{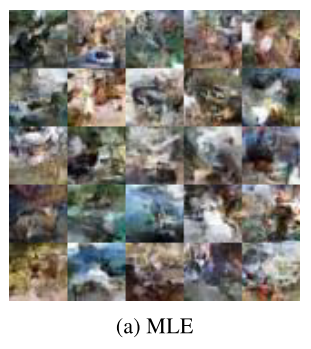} &
    \includegraphics[trim={.3cm 1.3cm .5cm .3cm},clip,width=0.25\textwidth]{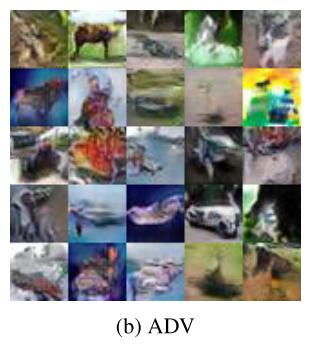} &    \includegraphics[trim={.3cm 1.3cm .5cm .3cm},clip,width=0.25\textwidth]{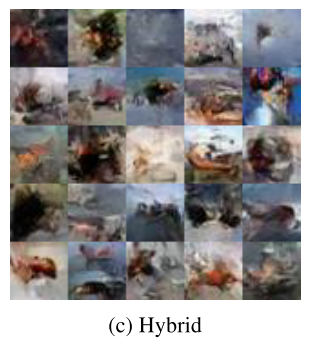}  \\
    \bottomrule
    \end{tabular}

\caption{\textbf{Comparing samples from our FlowGANs to figures copied from \citep{grover2018flow} (bottom row).} 
}
\label{fig:flowgan-cifar-samples}
\end{figure}

\end{document}